\documentclass[sigconf,nonacm]{acmart} % for arXiv

\usepackage{graphicx}
\usepackage{subcaption}
\usepackage{bm}
\usepackage{amsmath}
\usepackage{amsthm}
\usepackage{mathtools}
\usepackage{algorithmicx}
\usepackage{algorithm}
\usepackage[noend]{algpseudocode}
\usepackage[capitalize,noabbrev]{cleveref}
\usepackage{booktabs}
\usepackage{makecell}

\newtheorem{definition}{Definition}
\newtheorem{conjecture}{Conjecture}
\newtheorem{observation}{Observation}
\newtheorem{lemma}{Lemma}
\newtheorem{theorem}{Theorem}

\DeclareMathOperator*{\argmax}{argmax}

\usepackage[most]{tcolorbox}

\newtcolorbox{definition_box}{
    colback=white,
    colframe=black,
    boxrule=0.5pt,
    arc=0pt,
    breakable
}

\newtcolorbox{theorem_box}{
    colback=black!5!white,
    colframe=black!5!white,
    boxrule=0pt,
    arc=0pt,
    breakable
}

\usepackage{enumitem}

\newif\ifshowchanges
\showchangesfalse

\definecolor{revone}{HTML}{E67E22}
\definecolor{revtwo}{HTML}{1F77B4}
\definecolor{revthree}{HTML}{2CA02C}

\ifshowchanges
    \DeclareRobustCommand{\Rone}[1]{{\color{revone}#1}}
    \DeclareRobustCommand{\Rtwo}[1]{{\color{revtwo}#1}}
    \DeclareRobustCommand{\Rthree}[1]{{\color{revthree}#1}}
\else
    \DeclareRobustCommand{\Rone}[1]{#1}
    \DeclareRobustCommand{\Rtwo}[1]{#1}
    \DeclareRobustCommand{\Rthree}[1]{#1}
\fi
\AtBeginDocument{%
  }

\begin{document}

\title{Mathematical Foundations of Poisoning Attacks on\\Linear Regression over Cumulative Distribution Functions}

\author{Atsuki Sato}
\orcid{0009-0001-5366-4842}
\email{a\_sato@hal.t.u-tokyo.ac.jp}
\affiliation{%
    \institution{Graduate School of Information Science and Technology, The University of Tokyo}
    \city{Tokyo}
    \country{Japan}
}

\author{Martin Aum\"uller}
\email{maau@itu.dk}
\affiliation{%
    \institution{Algorithms Group, IT University of Copenhagen}
    \city{Copenhagen}
    \country{Denmark}
}

\author{Yusuke Matsui}
\email{matsui@hal.t.u-tokyo.ac.jp}
\affiliation{%
    \institution{Graduate School of Information Science and Technology, The University of Tokyo}
    \city{Tokyo}
    \country{Japan}
}

\renewcommand{\shortauthors}{Sato et al.}

\begin{abstract}
Learned indexes are a class of index data structures that enable fast search by approximating the cumulative distribution function (CDF) using machine learning models (Kraska et al., SIGMOD'18).
However, recent studies have shown that learned indexes are vulnerable to poisoning attacks, where injecting a small number of poison keys into the training data can significantly degrade model accuracy and reduce index performance (Kornaropoulos et al., SIGMOD'22).
In this work, we provide a rigorous theoretical analysis of poisoning attacks targeting linear regression models over CDFs, one of the most basic regression models and a core component in many learned indexes.
Our main contributions are as follows:
(i) We present a theoretical proof characterizing the optimal single-point poisoning attack and show that the existing method yields the optimal attack.
(ii) We show that in multi-point attacks, the existing greedy approach is not always optimal, and we rigorously derive the key properties that an optimal attack should satisfy.
(iii) We propose a method to compute an upper bound of the multi-point poisoning attack's impact and empirically demonstrate that the loss under the greedy approach is often close to this bound.
Our study deepens the theoretical understanding of attack strategies against linear regression models on CDFs and provides a foundation for the theoretical evaluation of attacks and defenses on learned indexes.
\end{abstract}

\begin{CCSXML}
<ccs2012>
    <concept>
        <concept_id>10002951.10002952.10002971</concept_id>
        <concept_desc>Information systems~Data structures</concept_desc>
        <concept_significance>500</concept_significance>
        </concept>
    <concept>
        <concept_id>10002978.10002979.10002983</concept_id>
        <concept_desc>Security and privacy~Cryptanalysis and other attacks</concept_desc>
        <concept_significance>300</concept_significance>
        </concept>
    <concept>
        <concept_id>10010147.10010257.10010293</concept_id>
        <concept_desc>Computing methodologies~Machine learning approaches</concept_desc>
        <concept_significance>100</concept_significance>
        </concept>
</ccs2012>
\end{CCSXML}

\ccsdesc[500]{Information systems~Data structures}
\ccsdesc[300]{Security and privacy~Cryptanalysis and other attacks}
\ccsdesc[100]{Computing methodologies~Machine learning approaches}

\keywords{Learned Systems, Data Poisoning, Attacks, Indexing}

% >> [Commented out for arXiv]
% \received{XX XXXX 20XX}
% \received[revised]{XX XXXX 20XX}
% \received[accepted]{XX XXXX 20XX}
% << [Commented out for arXiv]

\maketitle

\section{Introduction}
\label{sec:introduction}

\begin{table*}[t]
    \centering
    \caption{\Rtwo{Overview over} time complexity and empirical performance of poisoning attacks and upper-bounding methods \Rtwo{under the original setting~\cite{kornaropoulos2022price} (duplicate keys disallowed; \cref{def:poisoning_linear_regression_on_cdfs}) and the relaxed setting (duplicate keys allowed; \cref{def:relaxed_poisoning_problem}).
    Here, $n$ is the number of legitimate keys, $\lambda$ is the poisoning budget, and $T$ is the number of iterations in our upper-bounding methods.
    \textit{Empirical Performance (Ratio)} is the achieved MSE (attacks) or the bound (upper bounds) normalized by the optimal-attack MSE over $3{,}000$ cases in \cref{sec:experiment:upper_bound_vs_greedy_brute_force} (values closer to $1$ indicate a better attack or a tighter upper bound).}}
    \label{tab:seg_e}
    {\renewcommand{\arraystretch}{0.93}
    \begin{tabular}{@{}lcccc@{}}
        \toprule
        & \multicolumn{2}{c}{Original Setting (\cref{def:poisoning_linear_regression_on_cdfs})} & \multicolumn{2}{c}{Relaxed Setting (\cref{def:relaxed_poisoning_problem})} \\ 
        \cmidrule(lr){2-3} \cmidrule(lr){4-5}
        Algorithm & Time Complexity & Empirical Performance (Ratio) & Time Complexity & Empirical Performance (Ratio) \\
        \midrule
        Greedy~\cite{kornaropoulos2022price}
            & $\mathcal{O}(n \lambda)$ & $\geq\!0.933$ 
            & - & - \\ \midrule
        Exact Seg+E (Ours)
            & $\mathcal{O}(n \lambda^3)$ & $1$
            & $\mathcal{O}(n \lambda)$ & $1$ \\
        Heuristic Seg+E (Ours)
            & $\mathcal{O}(n \lambda)$ & $\geq\!0.998$
            & - & -\\
        Optimal (Ours)
            & $\mathcal{O}\!\left((n + \lambda) \binom{2n-2+\lambda}{\lambda}\right)$ & 1
            & $\mathcal{O}\!\left((n + \lambda) \binom{n+\lambda-1}{\lambda}\right)$ & 1 \\
        \midrule
        Upper Bound (Ours)
            & \makecell{$\mathcal{O}(T(n + \lambda))$ or \\$\mathcal{O}((n + \lambda) \log (n + \lambda))$}  & $\leq\!1.092$
            & \makecell{$\mathcal{O}(T(n + \lambda))$ or \\$\mathcal{O}((n + \lambda) \log (n + \lambda))$}  & $\leq\!1.087$ \\
        \bottomrule
    \end{tabular}
    }
\end{table*}

In recent years, learned indexes have attracted significant attention as a promising alternative to traditional index structures~\cite{kraska2018case_li,sosd-vldb,ferragina2020pgm,vaidya2020partitioned,hsu2019learning,yang2020leaper}.
One representative example of a learned index is the learned B-tree~\cite{kraska2018case_li,ferragina2020pgm,sun2023learned}, which approximates the cumulative distribution function (CDF) using regression models. It achieves superior memory efficiency and faster query performance compared to conventional methods.
During query processing, the learned B-tree first uses a regression model to predict the position of the query key, and then corrects the prediction error using local search algorithms such as exponential search.
Smaller prediction errors lead to narrower search ranges and faster query execution, whereas larger errors increase the cost of local search.

While learned indexes achieve outstanding performance on typical datasets, recent studies have revealed that they are vulnerable to poisoning attacks~\cite{kornaropoulos2022price,dong2025poisoning}.
Specifically, by inserting just a small number of malicious data keys (poisons) into the legitimate training data used to build the index, an attacker can deliberately degrade the model's prediction accuracy.
The increase in prediction error leads to higher costs for subsequent local searches, thereby degrading the overall performance of the learned index.

However, existing attack methods proposed in the literature are largely heuristic~\citep{yang2023algorithmic,schuster2022learned}.
\Rtwo{Even in the most basic poisoning setting---where a linear regression model is trained to minimize the mean squared error (MSE) and the attacker aims to maximize it---the optimality of the existing attack method~\cite{kornaropoulos2022price} remains unclear.}
This leaves several fundamental questions open:
\begin{itemize}[leftmargin=1.5em]
    \item Which structures characterize an optimal attack?
    \item Is the existing algorithm from~\cite{kornaropoulos2022price} an optimal attack?
    \item If not, can we derive a provable and tight upper bound on the maximum impact caused by any attack?
\end{itemize}
We resolve these questions and establish a theoretical foundation for attack strategies \Rtwo{in the same MSE-based linear-regression poisoning setting studied in prior work~\cite{kornaropoulos2022price}.}
Our contributions are as follows:
\begin{itemize}[leftmargin=1.5em]
\item In single-point attack scenarios (i.e., where only one poison point is injected), we provide the first formal proof that placing the poison adjacent to a legitimate key is optimal, as conjectured in~\cite{kornaropoulos2022price}. Thereby, we show that the existing single-point attack method yields an optimal attack.
\item \Rtwo{In multi-point attack scenarios, we show that the iterative greedy method of~\cite{kornaropoulos2022price} does \emph{not} in general exactly match the optimal solution;
we provide instances where it is strictly suboptimal, thereby refuting the implicit assumption in prior work~\cite{kornaropoulos2022price} that the greedy method is always optimal.}
\item In multi-point attack scenarios, we derive the key properties of an optimal attack; we prove that every poison included in an optimal attack is adjacent to a legitimate key, either directly or indirectly through other poison points. This makes it possible to compute the exact optimal solution in realistic time for small-scale settings.
\item In multi-point attack scenarios, we propose a method that establishes a rigorous upper bound on the achievable attack impact, i.e., a provable ceiling that no attack can exceed. This upper bound quantifies how close the greedy algorithm is to optimal, and also offers worst-case guarantees for the linear regression model under adversarial settings.
\item \Rtwo{In multi-point attack scenarios, we experimentally show that the greedy solution is often close to our upper bound:
over $3{,}000$ cases in \cref{sec:experiment:upper_bound_vs_greedy}, the bound is at most $1.25\times$ the greedy MSE and $1.03\times$ on average.
This suggests that the bound is tight and that greedy attacks are near-optimal.}
\item In the multi-point attack setting, we identified a simple structural property that frequently appears in optimal solutions, which we call \textit{Segment + Endpoint (Seg+E)}, and we provide efficient algorithms to find the Seg+E solution. We experimentally demonstrate that Seg+E consistently yields a larger loss than the greedy approach.
\end{itemize}
\Rtwo{All six contributions are original to our paper (i.e., they do not appear in prior work, including~\cite{kornaropoulos2022price}), and require nontrivial new technical or experimental insights.}
These results provide the first theoretical framework for understanding attack strategies against linear regression models trained on CDFs, offering a foundation for analyzing attacks and defenses on learned indexes.

\Rtwo{
\textbf{Summary of Algorithms.}
\cref{tab:seg_e} summarizes the algorithms studied in this work.
Here, $n$ is the number of legitimate keys, $\lambda$ is the poisoning budget, and $T$ is the number of iterations used by our upper-bound search procedures.
The column  \textit{Empirical Performance (Ratio)} reports the achieved MSE or upper bound, divided by the optimal-attack MSE, over $3{,}000$ cases in \cref{sec:experiment:upper_bound_vs_greedy_brute_force} (values closer to $1$ indicate better attack or tighter upper bound).
In addition to the original setting~\cite{kornaropoulos2022price}, which does not allow duplicate keys (\cref{def:poisoning_linear_regression_on_cdfs}), we also define a relaxed setting that allows duplicate keys (\cref{def:relaxed_poisoning_problem}).
We design, for both settings, methods to compute the optimum, upper bounds, and Seg+E attacks.
Empirically, we show that Greedy~\citep{kornaropoulos2022price} is occasionally suboptimal, whereas our heuristic Seg+E is closer to optimal, and Exact Seg+E matches the optimum in all $3{,}000$ instances under both settings.
Moreover, our upper-bound computation is fast and yields tight bounds (at most $1.092\times$ the optimal-attack MSE).}

This paper is organized as follows.
\cref{sec:related_work} provides an overview of related work, and 
\cref{sec:problem_statement_and_existing_attacks} formally introduces the problem setting and known attacks.
\cref{sec:single_point_poisoning_attack} and \cref{sec:multi_point_poisoning_attack} present our theoretical results for single- and multi-point attacks, respectively.
\cref{sec:experimental_evaluation} reports our experimental validation.
\Rone{\cref{sec:discussion} provides further discussion}, \cref{sec:limitations_and_future_work} outlines limitations and directions for future research, and \cref{sec:conclusion} concludes the paper.

\section{Related Work}
\label{sec:related_work}

In this section, we first provide a comprehensive overview of learned indexes, %in \cref{sec:related_work_learned_index}, 
followed by a review of attacks on them. % in \cref{sec:related_work_attacks_on_learned_index}.

\subsection{Learned Indexes}
\label{sec:related_work_learned_index}

% \begin{figure}[t]
%     \centering
%     \includegraphics[width=\columnwidth]{fig/rmi.pdf}
%     \caption{Illustration of the Recursive Model Index (RMI) with a two-stage architecture (Adapted from~\cite{kornaropoulos2022price}, licensed under CC BY 4.0).
%     Linear regression models are employed as the leaf models, each approximating the CDF of a contiguous subset of sorted keys.}
%     \label{fig:rmi}
% \end{figure}

Indexing is a fundamental technology for enabling efficient data access and has been applied across a wide range of scenarios, including databases~\cite{ramakrishnan2002database}, search engines~\cite{schutze2008introduction}, and file systems~\cite{ghemawat2003google}.
Classic examples include B-trees~\cite{bayer1972organization}, hash maps~\cite{knuth1998sorting}, and Bloom filters~\cite{bloom1970spacetime}.
These data structures have been extensively studied for decades as critical components that determine the overall performance of systems where efficient data access is essential.

Recently, learned indexes, which leverage machine learning to improve index performance, have been proposed and have attracted considerable attention~\cite{kraska2018case_li}.
Learned indexes replace or augment traditional data structures with machine learning models, aiming to improve memory efficiency and query speed.
Various learned index architectures have been proposed, extending classical data structures such as Bloom filters~\cite{mitzenmacher2018model,dai2020adaptive,vaidya2020partitioned,sato2023fast}, count-min sketches~\cite{hsu2019learning,zhang2020learned,dolera2023learning}, and LSM-trees~\cite{yang2020leaper,lu2022tridentkv,dai2020wisckey,mo2023learning} using machine learning techniques.

Among these, learned B-trees, which enhance B-trees using machine learning, are the most actively studied type of learned index and are often referred to simply as ``learned indexes''~\cite{galakatos2019fiting,ding2020alex,ferragina2020pgm,wang2020sindex,sun2023learned,amarasinghe2024learned,wu2022nfl}.
The core idea of learned indexes is to approximate the cumulative distribution function (CDF) of the data using machine learning-based regression.
During query processing, the model is used to infer an approximate position of the query key.
This prediction is then refined using local search techniques such as exponential search to identify the exact position.
Beyond the original one-dimensional, static setting~\cite{kraska2018case_li,kipf2020radixspline,marcus2020cdfshop,wu2022nfl,amarasinghe2024learned},
subsequent work has extended learned indexes to multidimensional data~\cite{li2020lisa,nathan2020learning,ding2020tsunami,qi2020effectively,wang2019learned,hidaka2024flexflood,al2024survey}
and string data~\cite{wang2020sindex,yang2024lits,zhou2024slipp,spector2021bounding,kim2024accelerating},
as well as to dynamic settings that support updates~\cite{ding2020alex,li2019scalable,lu2021apex,tang2020xindex,li2021finedex,galakatos2019fiting,ferragina2020pgm,wu2021updatable,wang2025new}.
In addition, there has been increasing interest in theoretical analyses of worst-case/expected complexities in construction/query~\cite{ferragina2020why,zeighami2023distribution,croquevielle2024constant,zeighami2024theoretical,liu2024learned}.

A common structural characteristic shared by many learned indexes is their hierarchical, Mixture-of-Experts-style design~\cite{sun2023learned}.
That is, a root model routes queries to localized leaf models, each approximating the CDF of a contiguous subset of sorted keys.
Linear regression models are often used as leaf models~\cite{ding2020alex,li2021finedex,wang2020sindex,sun2023learned} because they provide sufficient accuracy for small subdatasets, avoiding the computational overhead of complex models.
% The RMI~\cite{kraska2018case_li}, illustrated in \cref{fig:rmi}, is a representative example: the root model (often a small neural network) selects a leaf, and the selected leaf approximates the CDF of its assigned sorted keys using linear regression.

\subsection{Attacks on a Learned Index}
\label{sec:related_work_attacks_on_learned_index}

While learned indexes often offer excellent performance on typical datasets, recent research has revealed that they also exhibit vulnerabilities to various attacks.
Attacks on learned Bloom filters include poisoning attacks~\cite{dong2025poisoning}, side-channel attacks~\cite{farwah2024exploiting}, and mutation attacks~\cite{reviriego2021learned}.
Additional studies demonstrate attacks on learning-based sketches~\cite{jing2022deceiving}, learned cardinality estimators~\cite{zhang2024pace}, and learned index advisors~\cite{zhou2024trap,zheng2024robustness}.
The root of these vulnerabilities lies in two key factors: (i) the increased information flow introduced by fitting models to past data distributions, and (ii) the degradation of worst-case performance caused by optimizing machine learning models for average-case scenarios~\cite{schuster2022learned}.

Among these, attacks on learned indexes (in the narrow sense) have been particularly actively studied in recent years~\citep{kornaropoulos2022price,yang2023algorithmic,schuster2022learned}.
Kornaropoulos et al.~\cite{kornaropoulos2022price} proposed poisoning attacks against linear regression models trained on CDFs, and further demonstrated attacks on RMI architectures using such poisoning methods.
However, the theoretical foundations of these attacks remained unresolved, and addressing this gap is the focus of our study.

Finally, we briefly review other types of attacks on learned B-trees.
For ALEX~\cite{ding2020alex}, one of the most prominent dynamic learned \Rone{indexes}, algorithmic complexity attacks (ACAs) have been proposed in~\cite{yang2023algorithmic,schuster2022learned}.
These attacks force the system into exponential time worst-case behavior.
Additionally, even the learned indexes with guaranteed worst-case complexity (e.g., the PGM-index~\cite{ferragina2020pgm}) are shown to be vulnerable to timing-channel-based attacks~\cite{schuster2022learned}.
By measuring and analyzing query latency, an attacker may infer information about the underlying data distribution that should be inaccessible. Such attacks are not the focus of this work.

\section{Problem Statement and Existing Attacks}
\label{sec:problem_statement_and_existing_attacks}

\begin{figure*}[t]
    \centering
    \begin{subfigure}[b]{0.32\textwidth}
        \centering
        \includegraphics[width=\textwidth]{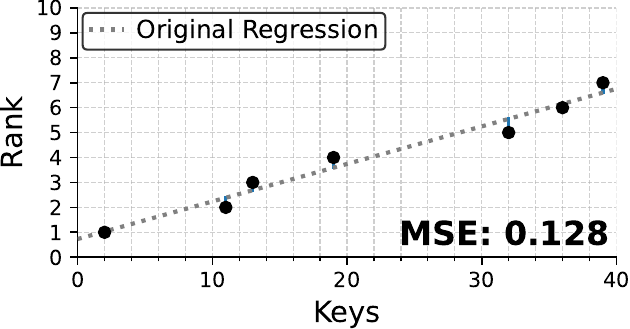}
        \caption{Regression Before Poisoning.}
        \label{fig:single_point_poisoning_attack_a}
    \end{subfigure}
    \begin{subfigure}[b]{0.32\textwidth}
        \centering
        \includegraphics[width=\textwidth]{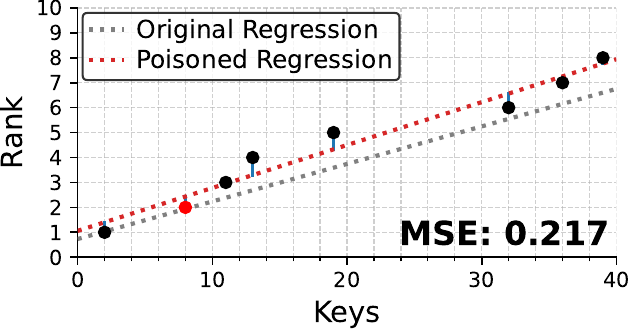}
        \caption{Regression After Poisoning ($\mathcal{P}=\{8\}$).}
        \label{fig:single_point_poisoning_attack_b}
    \end{subfigure}
    \begin{subfigure}[b]{0.32\textwidth}
        \centering
        \includegraphics[width=\textwidth]{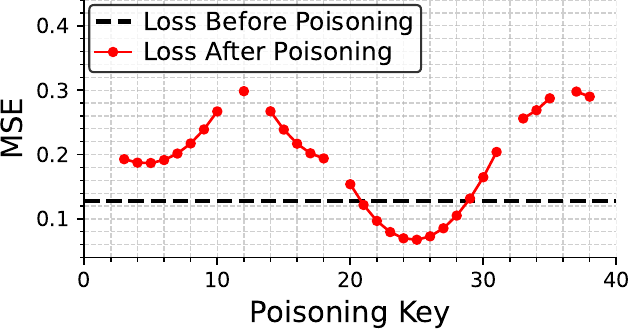}
        \caption{Loss Function.}
        \label{fig:single_point_poisoning_attack_c}
    \end{subfigure}
    \caption{Single-point poisoning attack on $\mathcal{K} = \{2, 11, 13, 19, 32, 36, 39\}$.
    By inserting a poison point, the rank of every key greater than the poison increases by one.
    The linear regression model is fitted on the poisoned key set, leading to a change in MSE.}
    \label{fig:single_point_poisoning_attack}
\end{figure*}

In this section, we first define the problem setting of poisoning attacks against linear regression models trained on CDFs (\cref{sec:problem_statement}), followed by an overview of existing attack methods (\cref{sec:existing_attacks}).

\subsection{Problem Statement}
\label{sec:problem_statement}

Throughout this paper, we let $\mathcal{K}$ denote the set of \emph{legitimate keys}, i.e., the keys originally present in the training data before any attack, and $n \coloneq |\mathcal{K}|$.
We use $\mathcal{X}$ to denote a general key set, which may include both legitimate and poison keys, and $m \coloneq |\mathcal{X}|$.

We begin by defining linear regression on CDFs \Rone{similarly} to~\cite{kornaropoulos2022price}.
\begin{definition_box}
\begin{definition}[\textbf{Linear Regression on CDFs~\cite[Def. 1]{kornaropoulos2022price}}]
\label{def:linear_regression_on_cdfs}
Let $\mathcal{X} = \{x_1, x_2, \dots, x_m\}$ be a multiset of natural numbers representing the keys stored in the index, where $x_1 \leq x_2 \leq \dots \leq x_m$ and at least two distinct values appear in $\mathcal{X}$ (i.e., $x_1 \neq x_m$).
Define the rank associated with $x_i$ as $r_i \coloneq i$.
The MSE of fitting a linear regression model with parameters $w, b \in \mathbb{R}$ is defined as
\begin{equation}
    \mathcal{L}(\mathcal{X};w,b) \coloneq \frac{1}{m} \sum_{i=1}^m \left(w x_i + b - r_i\right)^2.
\end{equation}
The problem of linear regression on CDFs is to find the model parameters that minimize the MSE, that is, to determine
\begin{equation}
    E(\mathcal{X}) \coloneq \min_{w,b} \mathcal{L}(\mathcal{X};w,b).
\end{equation}
\end{definition}
\end{definition_box}
\Rtwo{For convenience, define $\mathrm{Cov}_\mathrm{XR}, \mathrm{Var}_\mathrm{X}, \mathrm{Var}_\mathrm{R}, \bar{x}$, and $\bar{r}$ as follows:}
\Rtwo{\begin{align}
\label{eq:cov_and_var_and_m}
    \mathrm{Cov}_\mathrm{XR} &= \frac{1}{m} \sum_{i=1}^m (x_i - \bar{x})(r_i - \bar{r}), \quad
    \mathrm{Var}_\mathrm{X} = \frac{1}{m} \sum_{i=1}^m (x_i - \bar{x})^2, \nonumber \\
    \mathrm{Var}_\mathrm{R} &= \frac{1}{m} \sum_{i=1}^m (r_i - \bar{r})^2, \quad
    \bar{x} = \frac{1}{m} \sum_{i=1}^m x_i, \quad
    \bar{r} = \frac{1}{m} \sum_{i=1}^m r_i.
\end{align}}
\Rtwo{Then, this problem has the following closed-form solution~\cite{kornaropoulos2022price}:}
\Rtwo{\begin{equation}
\label{eq:closed_form_solution}
    w^\ast = \frac{\mathrm{Cov}_\mathrm{XR}}{\mathrm{Var}_\mathrm{X}}, 
    \quad b^\ast = \bar{r} - w^\ast \bar{x}, 
    \quad E(\mathcal{X}) = \mathrm{Var}_\mathrm{R} - \frac{{\mathrm{Cov}_\mathrm{XR}}^2}{\mathrm{Var}_\mathrm{X}}.
\end{equation}}
\Rtwo{Following prior work~\cite{kornaropoulos2022price}, we focus on integer keys for simplicity; extensions to other key types are discussed in \cref{sec:discussion}.}

We address the following poisoning attack problem, which is the same as the one studied in~\cite{kornaropoulos2022price}:
\begin{definition_box}
\begin{definition}[\textbf{Poisoning Linear Regression on CDFs~\cite[Def. 2]{kornaropoulos2022price}}]
\label{def:poisoning_linear_regression_on_cdfs}
Let $\mathcal{K} = \{k_1, k_2, \dots, k_n\} \subset \mathbb{N}$ be $n$ distinct legitimate keys (i.e., the keys in the original training data before any attack), sorted so that $k_1 < k_2 < \dots < k_n$.
Let $\lambda \in \mathbb{N}$ be the maximum number of poison points.
The problem of poisoning linear regression on CDFs is to find a set $\mathcal{P}^\ast \subset \mathbb{N}$ satisfying $|\mathcal{P}^\ast| \leq \lambda$ and $\mathcal{P}^\ast \subset \{k_1 + 1, k_1 + 2, \dots, k_n - 1\} \setminus \mathcal{K}$ that maximizes $E(\mathcal{K} \cup \mathcal{P}^\ast)$.
Formally, 
\begin{equation}
\label{eq:poisoning_linear_regression_on_cdfs}
    \mathcal{P}^\ast \coloneq \argmax_{\mathcal{P} ~ \mathrm{s.t.} ~ |\mathcal{P}| \leq \lambda, ~ \mathcal{P} \subset \{k_1 + 1, k_1 + 2, \dots, k_n - 1\} \setminus \mathcal{K}} E(\mathcal{K} \cup \mathcal{P}).
\end{equation}
\end{definition}
\end{definition_box}
\Rone{\textbf{Threat Model.}
We adopt the same threat model as in~\cite{kornaropoulos2022price}.
The attacker has full knowledge of the legitimate key set $\mathcal{K}$ and can inject arbitrarily chosen poison keys into the training data, subject only to the constraint in~\cref{eq:poisoning_linear_regression_on_cdfs}.
Among common threat models, this setting gives the attacker the strongest information and capabilities, and thus represents the most impactful attack scenario.}

As in~\cite{kornaropoulos2022price}, we restrict poisons to lie strictly between the minimum and maximum legitimate keys to avoid trivial outlier detection and removal.
\Rtwo{If there are no unused integer values between $k_1$ and $k_n$ (i.e., $\mathcal{K}=\{k_1,k_1+1,\dots,k_n\}$), then no poisons can be injected.}
When $\lambda = 1$, the problem is referred to as a \textit{single-point poisoning attack}; when $\lambda \geq 2$, it is called a \textit{multi-point poisoning attack}.

\cref{fig:single_point_poisoning_attack} illustrates an example of a single-point attack.
\cref{fig:single_point_poisoning_attack_a} shows the legitimate keys and the linear regression model fitted to their CDF.
\cref{fig:single_point_poisoning_attack_b} shows how the regression line is distorted when a poison point $\mathcal{P} = \{8\}$ is inserted.
A key characteristic of poisoning attacks against linear regression on CDFs is that the insertion of poison points alters the ranks of all keys following the poisons.
For example, in \cref{fig:single_point_poisoning_attack_a,fig:single_point_poisoning_attack_b}, inserting the poison point 8 does not affect the ranks of keys smaller than 8, but increases the rank of each key larger than 8 by one. 
This rank-shifting effect is one of the main factors that makes the mathematical analysis of such attacks challenging.

\subsection{Existing Attacks}
\label{sec:existing_attacks}

\begin{algorithm}[t]
    \caption{\Rone{Single-point Poisoning Attack}\Rtwo{~\cite{kornaropoulos2022price}}}
    \label{algo:single-point poisoning attack}
    \begin{algorithmic}[1]
    \Require Legitimate Keys: $\mathcal{K}$
    \Function{Single-pointPoisoningAttack}{$\mathcal{K}$}
        \State $k_\mathrm{min} \gets \min(\mathcal{K}), \quad k_\mathrm{max} \gets \max(\mathcal{K})$
        \State $\mathcal{P}_\mathrm{Cands} \gets \{k + 1 \mid k \in \mathcal{K} \setminus \{k_\mathrm{max}\}\} \cup \{k - 1 \mid k \in \mathcal{K} \setminus \{k_\mathrm{min}\}\}$
        \State $\mathcal{P}_\mathrm{Cands} \gets \mathcal{P}_\mathrm{Cands} \setminus \mathcal{K}$
        \If {$\mathcal{P}_\mathrm{Cands} = \varnothing$} \textbf{return} \texttt{None} \EndIf
        \State $p \gets \textsc{SelectMaxLossCandidate}(\mathcal{K}, \mathcal{P}_\mathrm{Cands})$
        \If {$E(\mathcal{K} \cup \{p\}) < E(\mathcal{K})$} \textbf{return} \texttt{None} \EndIf
        \State \textbf{return} $p$
    \EndFunction
    \end{algorithmic}
\end{algorithm}

\begin{algorithm}[t]
\caption{\Rone{Greedy Multi-point Poisoning Attack}\Rtwo{~\cite{kornaropoulos2022price}}}
\label{algo:multi-point poisoning attack}
\begin{algorithmic}[1]
\Require Legitimate Keys: $\mathcal{K}$, Maximum Number of Poisons: $\lambda$
\Function{Multi-pointPoisoningAttack}{$\mathcal{K}, \lambda$}
    \State $\mathcal{P} \gets \varnothing$
    \While {$|\mathcal{P}| < \lambda$}
        \State $p \gets \textsc{Single-pointPoisoningAttack}(\mathcal{K} \cup \mathcal{P})$
        \If {$p = \texttt{None}$} \textbf{break} \EndIf
        \State $\mathcal{P} \gets \mathcal{P} \cup \{p\}$
    \EndWhile
    \State \textbf{return} $\mathcal{P}$
\EndFunction
\end{algorithmic}
\end{algorithm}

\textbf{Single-point poisoning attack.}
\Rtwo{We first recap the single-point poisoning attack strategy proposed by~\cite{kornaropoulos2022price}.
Importantly, the following statement is an \emph{empirical observation} reported in prior work and is \emph{not} one of our contributions:}
\begin{theorem_box}
\begin{observation}[\Rtwo{\textbf{Empirical Observation in~\cite{kornaropoulos2022price}}}]
\label{observ:single_point_poisoning_attack}
The optimal single-point attack $\mathcal{P}^\ast$ is either the empty set or a singleton containing an integer adjacent to a legitimate key;
\begin{equation}
\mathcal{P}^\ast \subset \{k + 1 \mid k \in \mathcal{K}\} \cup \{k - 1 \mid k \in \mathcal{K}\}.
\end{equation}
\end{observation}
\end{theorem_box}
\cref{fig:single_point_poisoning_attack} presents an example illustrating this observation.
\cref{fig:single_point_poisoning_attack_c} shows a plot of $E(\mathcal{K} \cup \{p\})$ for all possible poison points $p$.
Note that integers already occupied by legitimate keys are excluded from the plot, as they cannot serve as poison keys.
In this example, $E(\mathcal{K} \cup \{p\})$ forms a convex shape within each interval between two consecutive legitimate keys.
Therefore, the poison point $p$ that maximizes the loss in each interval lies at either endpoint of the interval.
Specifically, the optimal single-point attack in this case is $\mathcal{P} = \{12\}$, which aligns with \cref{observ:single_point_poisoning_attack}.

Based on this observation, \cite{kornaropoulos2022price} proposed the following single-point attack algorithm: exhaustively evaluating $E(\mathcal{K} \cup \{p\})$ for every candidate poison point $p$ (i.e., integers adjacent to legitimate keys).
The algorithm is summarized in \cref{algo:single-point poisoning attack}.
By employing differential calculation, this evaluation can be performed in time $\mathcal{O}(n)$.
However,~\cite{kornaropoulos2022price} did not prove the optimality of this algorithm.

\textbf{Multi-point poisoning attack.}
For multi-point attacks, \cite{kornaropoulos2022price} iteratively apply the above single-point attack $\lambda$ times, see \cref{algo:multi-point poisoning attack}.
The time complexity of this multi-point poisoning algorithm is $\mathcal{O}((n + \lambda)\lambda)$.
Empirically, they report that this greedy approach results in an optimal attack, again missing a formal proof.

\begin{figure}[t]
    \centering
    \includegraphics[width=\columnwidth]{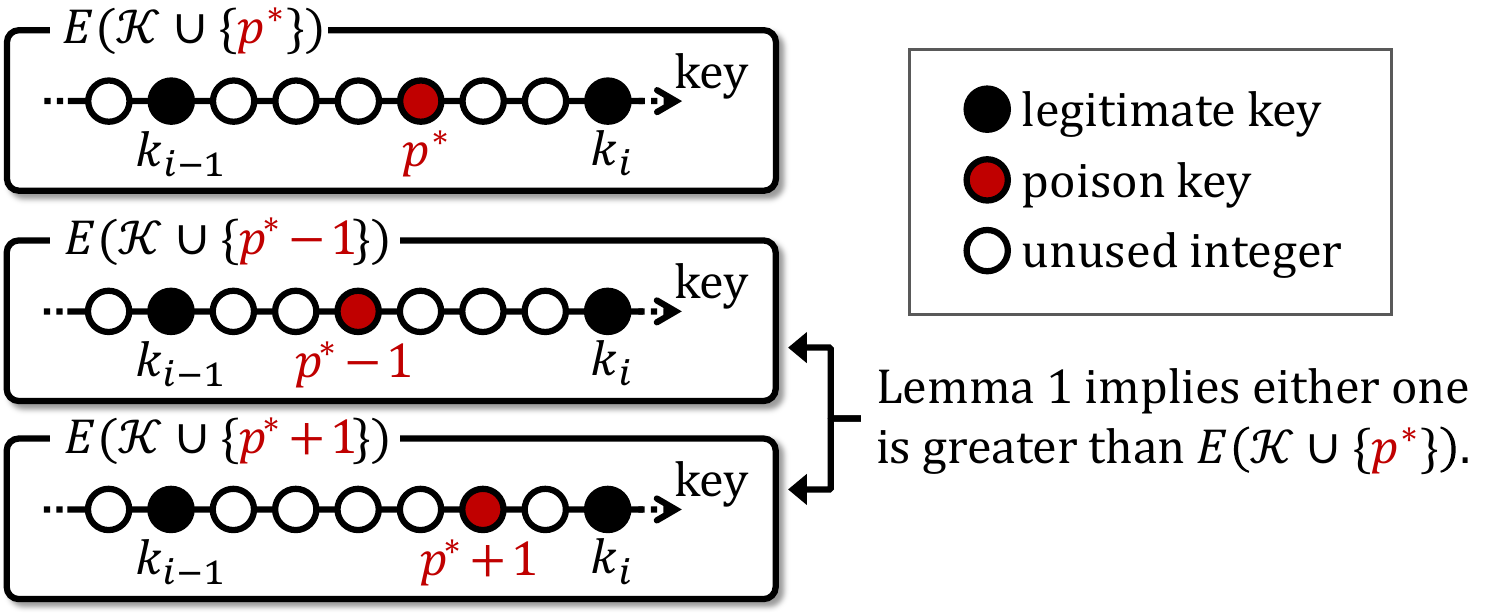}
    \caption{Proof of \cref{thm:single_point_poisoning_attack}. \cref{lem:single_point_poisoning_attack} implies that either $E( \mathcal{K} \cup \{p^\ast - 1\})$ or $E( \mathcal{K} \cup \{p^\ast + 1\})$ is greater than $E( \mathcal{K} \cup \{p^\ast\})$.}
    \label{fig:single_point_poisoning_attack_proof}
\end{figure}

\begin{figure*}[t]
    \centering
    \begin{subfigure}[b]{0.32\textwidth}
        \centering
        \includegraphics[width=\textwidth]{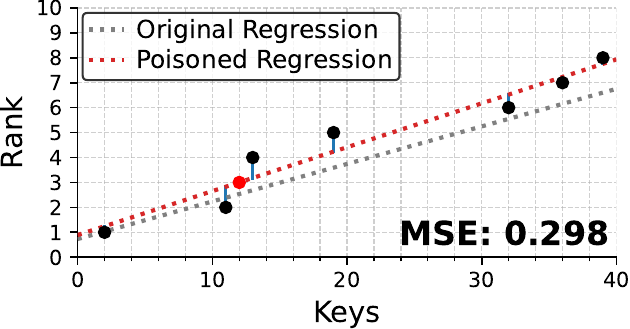}
        \caption{Optimal 1-point Poisoning.}
        \label{fig:greedy_poisoning_p1}
    \end{subfigure}
    \begin{subfigure}[b]{0.32\textwidth}
        \centering
        \includegraphics[width=\textwidth]{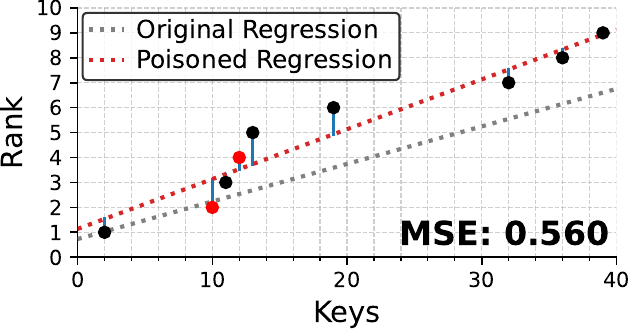}
        \caption{Greedy 2-point Poisoning.}
        \label{fig:greedy_poisoning_p2}
    \end{subfigure}
    \begin{subfigure}[b]{0.32\textwidth}
        \centering
        \includegraphics[width=\textwidth]{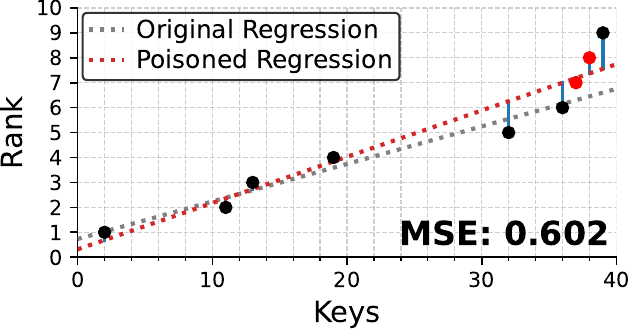}
        \caption{Optimal 2-point Poisoning.}
        \label{fig:optimal_poisoning_p2}
    \end{subfigure}
    \caption{Multi-point poisoning attack.
    The greedy two-point attack~\cite{kornaropoulos2022price} injects the poison point $12$ first (\cref{fig:greedy_poisoning_p1}), followed by the injection of the poison point $10$ (\cref{fig:greedy_poisoning_p2}). 
    In contrast, the optimal attack is $\{37,38\}$, resulting in a higher MSE (\cref{fig:optimal_poisoning_p2}).}
    \label{fig:multi_point_poisoning_attack}
\end{figure*}

\section{Optimality of Single-Point Poisoning Attack}
\label{sec:single_point_poisoning_attack}

\Rtwo{
We theoretically guarantee the optimality of the single-point attack proposed in~\cite{kornaropoulos2022price} by proving the following theorem.
}
\begin{theorem_box}
    \begin{theorem}
    \label{thm:single_point_poisoning_attack}
    \cref{observ:single_point_poisoning_attack} always holds.
    In other words, the optimal single-point attack $\mathcal{P}^\ast$ satisfies:
    \begin{equation}
        \mathcal{P}^\ast \subset \{k + 1 \mid k \in \mathcal{K}\} \cup \{k - 1 \mid k \in \mathcal{K}\}.
    \end{equation}
    \end{theorem}
\end{theorem_box}
To prove \cref{thm:single_point_poisoning_attack}, we first establish the following key lemma.
\begin{theorem_box}
\begin{lemma}
\label{lem:single_point_poisoning_attack}
Let $\mathcal{X}$ be the set of keys stored in the index, with $m = |\mathcal{X}| \geq 3$, and let $x_1 \leq x_2 \leq \cdots \leq x_m$ denote the elements of $\mathcal{X}$ in increasing order.
Then, for each $i \in \{2, 3, \dots, m-1\}$, $\mathrm{sign}\left(\frac{d E(\mathcal{X})}{dx_i}\right)$ is non-decreasing over the interval $x_i \in (x_{i-1}, x_{i+1})$, and there exists at most one value of $x_i$ in this interval for which $\frac{d E(\mathcal{X})}{dx_i} = 0$.
Here, the function $\mathrm{sign}(x)$ is defined as $-1$ if $x < 0$, $0$ if $x = 0$, and $1$ if $x > 0$.
\end{lemma}
\end{theorem_box}
\begin{proof}[\Rthree{Proof sketch of \cref{lem:single_point_poisoning_attack}}]
\Rthree{Computing $\mathrm{sign}\left(d E(\mathcal{X}) / dx_i \right)$ shows that it equals the sign of the covariance between $\bm{x}'$ and $\bm{r}'$, where $\bm{x}'$ and $\bm{r}'$ denote the vectors obtained from $\bm{x}$ and $\bm{r}$, respectively, by removing their $i$-th elements (i.e., $x_i$ and $i$).
Since both $\bm{x}'$ and $\bm{r}'$ are monotonically increasing, the sign of covariance is positive.}
\end{proof}
\Rthree{We defer the full, formal proof to \cref{app:proof_of_lem_single_point_poisoning_attack}.
Using \cref{lem:single_point_poisoning_attack}, we can prove \cref{thm:single_point_poisoning_attack} as follows.}
\begin{proof}[Proof of \cref{thm:single_point_poisoning_attack}]
To facilitate comprehension of the proof structure, an illustrative figure is presented in \cref{fig:single_point_poisoning_attack_proof}.
Assume, for the sake of contradiction, that there exists an optimal single-point attack $\mathcal{P}^\ast = \{p^\ast\}$ such that $p^\ast \notin \{k + 1 \mid k \in \mathcal{K}\} \cup \{k - 1 \mid k \in \mathcal{K}\}$.

Let $i \in \{2, 3, \dots, n\}$ be the unique index satisfying $k_{i-1} < p^\ast < k_i$.
Under our assumption, the integers adjacent to $p^\ast$ are not legitimate keys, meaning $k_{i-1} < p^\ast - 1$ and $p^\ast + 1 < k_i$ (see \cref{fig:single_point_poisoning_attack_proof}).
Since $p^\ast$ is assumed to be an optimal single poison point, we have:
\begin{equation}
\label{eq:single_point_from_assumption}
    E( \mathcal{K} \cup \{p^\ast {-} 1\}) {\leq} E( \mathcal{K} \cup \{p^\ast\}) ~~ \land ~~ E( \mathcal{K} \cup \{p^\ast\}) {\geq} E( \mathcal{K} \cup \{p^\ast {+} 1\}).
\end{equation}

%On the other hand, this leads to a contradiction with \cref{lem:single_point_poisoning_attack}.
According to \cref{lem:single_point_poisoning_attack}, the function $\mathrm{sign}\left(\frac{d E(\mathcal{K} \cup \{p\})}{dp}\right)$ is non-decreasing over the interval $p \in (k_{i-1}, k_i)$, and there exists at most one point $p$ in this interval where $\frac{d E(\mathcal{K} \cup \{p\})}{dp} = 0$.
Therefore, the behavior of $E(\mathcal{K} \cup \{p\})$ over $p \in (k_{i-1}, k_i)$ must fall into one of the following three cases:
(i) strictly increasing,
(ii) strictly decreasing, or
(iii) strictly decreasing over $(k_{i-1}, p')$ and strictly increasing over $(p', k_i)$ for some $p' \in (k_{i-1}, k_i)$. In all cases, either  $E( \mathcal{K} \cup \{p^\ast - 1\}) > E( \mathcal{K} \cup \{p^\ast\})$ or $ E( \mathcal{K} \cup \{p^\ast\}) < E( \mathcal{K} \cup \{p^\ast + 1\})$. This contradicts the assumption that $p^\ast$ is optimal.
% In all cases, the following holds:
% \begin{equation}
% \label{eq:single_point_from_lemma}
%     E( \mathcal{K} \cup \{p^\ast - 1\}) > E( \mathcal{K} \cup \{p^\ast\}) ~~ \lor ~~ E( \mathcal{K} \cup \{p^\ast\}) < E( \mathcal{K} \cup \{p^\ast + 1\}).
% \end{equation}
% Since \cref{eq:single_point_from_assumption} and \cref{eq:single_point_from_lemma} are in contradiction, we conclude that the assumption must be false.
% Thus, we have shown that the optimal single-point attack satisfies $\mathcal{P}^\ast \subset \{k + 1 \mid k \in \mathcal{K}\} \cup \{k - 1 \mid k \in \mathcal{K}\}$.
\end{proof}

By \cref{thm:single_point_poisoning_attack}, we have established that the single-point attack proposed in~\cite{kornaropoulos2022price} is guaranteed to find the optimal solution.
In other words, by exhaustively searching only the integers adjacent to legitimate keys, one can identify the optimal single-point attack.

\section{Multi-Point Poisoning Attack}
\label{sec:multi_point_poisoning_attack}

In this section, we first demonstrate that the iterative greedy algorithm (\cref{algo:multi-point poisoning attack}) proposed in~\cite{kornaropoulos2022price} is not optimal, and we then describe the key properties that can be rigorously asserted about the structure of the optimal solution (\cref{sec:structure_of_optimal_multi_point_attack}).
Next, leveraging these properties, we propose a method to upper-bound the loss incurred by multi-point attacks (\cref{sec:upper_bound_multi_point_poisoning_attack}).
Finally, motivated by the structural insights described above, we introduce the Segment + Endpoint (Seg+E) class of poison sets and provide efficient exact and heuristic algorithms to find Seg+E solutions (\cref{sec:segment_plus_endpoint_attack}).

\subsection{Structure of Optimal Multi-Point Attacks}
\label{sec:structure_of_optimal_multi_point_attack}

\begin{theorem_box}
\begin{observation}
\label{observ:greedy_not_optimal_multipoint}
The iterative greedy approach (\cref{algo:multi-point poisoning attack}) does not guarantee an optimal multi-point poisoning attack. 
\end{observation}
\end{theorem_box}

\cref{fig:multi_point_poisoning_attack} provides an \Rone{illustrative example} where the greedy approach does not lead to the optimal solution.
As shown in the previous section, for the single-point attack case, the optimal solution can be found by exhaustively evaluating all integers adjacent to the legitimate keys (\cref{fig:greedy_poisoning_p1}).
The greedy poisoning method of~\cite{kornaropoulos2022price} performs a two-point attack by adding one additional optimal poison key on top of this (\cref{fig:greedy_poisoning_p2}).
However, the true optimal two-point attack involves adding poison keys as shown in \cref{fig:optimal_poisoning_p2}.
Thus, the greedy poisoning method does not always yield the optimal solution.
This finding is valuable for preventing confusion among future researchers, as~\cite{kornaropoulos2022price} reports the following:
``\textit{Even though we do not provide a proof of optimality for the multiple-point poisoning, we experimentally observed that our approach matched the performance of the brute-force attack in every tested dataset.}''

\begin{figure}[t]
    \centering
    \includegraphics[width=\columnwidth]{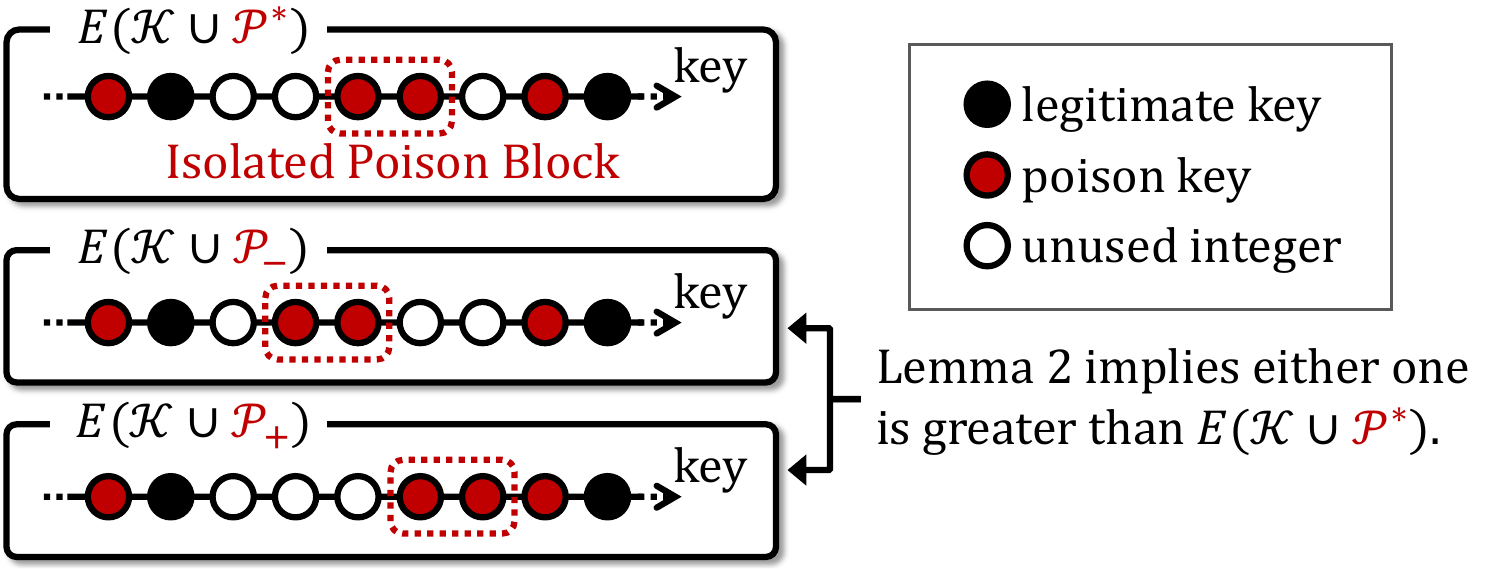}
    \caption{Proof of \cref{thm:structure_of_optimal_multi_point_attack}.
    \cref{lem:multi_point_poisoning_attack} implies that $E(\mathcal{K} \cup \mathcal{P}_{-})$ or $E(\mathcal{K} \cup \mathcal{P}_{+})$ is greater than $E(\mathcal{K} \cup \mathcal{P}^\ast)$.}
    \label{fig:multi_point_theorem_vis}
\end{figure}

Next, we establish the following theorem, which characterizes the structure of the optimal solution for multi-point attacks.
This result can be seen as an extension of \cref{thm:single_point_poisoning_attack}.
\begin{theorem_box}
\begin{theorem}
\label{thm:structure_of_optimal_multi_point_attack}
    The optimal multi-point attack consists only of poison keys that are either adjacent to legitimate keys directly or connected transitively to legitimate keys via chains of neighboring poison keys.
    Formally, the optimal multi-point attack $\mathcal{P}^\ast$ satisfies:
    \begin{equation}
        \forall p \in \mathcal{P}^\ast,~ \exists k \in \mathcal{K}, ~ \{\min(p,k)+1, \dots, \max(p,k)-1\} \subset \mathcal{P}^\ast.
    \end{equation}
\end{theorem}
\end{theorem_box}
\begin{proof}[\Rthree{Proof sketch of \cref{thm:structure_of_optimal_multi_point_attack}}]
\Rthree{To facilitate comprehension of the proof structure, an illustrative figure is presented in \cref{fig:multi_point_theorem_vis}.
First, we prove an extension of \cref{lem:single_point_poisoning_attack}, denoted \cref{lem:multi_point_poisoning_attack} (in \cref{app:proof_of_thm_structure_of_optimal_multi_point_attack}), which implies that if there exists an isolated block of consecutive poison keys (an \emph{Isolated Poison Block}), then moving that block either left or right can strictly increase the loss.
Applying the same argument as in \cref{thm:single_point_poisoning_attack}, but with an Isolated Poison Block instead of a single poison key, yields \cref{thm:structure_of_optimal_multi_point_attack}.}
\end{proof}
\Rthree{We defer the full, formal proof to \cref{app:proof_of_thm_structure_of_optimal_multi_point_attack}.}

An immediate consequence of this theorem is that we can find an optimal solution by examining at most $\binom{2n-2+\lambda}{\lambda}$ candidate multi-point attacks.
This follows because the number of candidate optimal multi-point attacks is bounded by the number of ways to distribute the $\lambda$ units of poison among $2n-1$ groups (corresponding to poison placed to the right of $k_1$, to the left of $k_2$, to the right of $k_2$, $\dots$, to the left of $k_n$, and the unused portion of the budget).
Since the MSE for each poisoning configuration can be computed in $\mathcal{O}(n+\lambda)$ time, the overall running time to find the optimum is $\mathcal{O}\left((n+\lambda)\binom{2n-2+\lambda}{\lambda}\right)$.
Without our structural result, one would have to examine all combinations of integers in the domain that are not occupied by legitimate keys.
The number of such combinations is $\mathcal{O}\left(\binom{k_n-k_1}{\lambda}\right)$, which is computationally infeasible in practice because $k_n-k_1$ is typically at least $10^3$ and can easily exceed $10^9$. 
Our theorem drastically shrinks the search space, making it realistic to compute the optimal multi-point attack for moderately small values of $n$ and $\lambda$.

Next, we highlight two potential misconceptions around the structure of optimal multi-point attacks.
The first potential misconception is that \cref{thm:single_point_poisoning_attack} trivially implies \cref{thm:structure_of_optimal_multi_point_attack}.
For the sake of an example, let the legitimate key set be $\mathcal{K} = \{1, 5\}$.
While \cref{thm:structure_of_optimal_multi_point_attack} allows us to rigorously conclude that $\mathcal{P} = \{3,4\}$ is not an optimal solution, the statement of \cref{thm:single_point_poisoning_attack} alone does not allow us to rule out the possibility that $\mathcal{P} = \{3,4\}$ could be the optimal two-point attack.
Although $\mathcal{P} = \{3\}$ or $\mathcal{P} = \{4\}$ are not optimal in the single-point case, \cref{thm:single_point_poisoning_attack} does not provide any guarantee about whether $\mathcal{P} = \{3,4\}$ could become optimal in the multi-point case.
Only with \cref{thm:structure_of_optimal_multi_point_attack} can we rigorously conclude that $\mathcal{P} = \{3,4\}$ is not the optimal solution.

The second potential misconception is that \cref{thm:structure_of_optimal_multi_point_attack} serves as a proof of the optimality of the greedy approach.
This is also a misconception, which we demonstrated in \cref{fig:multi_point_poisoning_attack}, where the greedy poisoning method is not optimal.
\cref{thm:structure_of_optimal_multi_point_attack} merely shows that the solution obtained by the greedy poisoning method \textbf{can be} an optimal solution in the sense that all poison keys are directly or indirectly adjacent to some legitimate key.
It does not claim that the greedy poisoning method \textbf{is guaranteed to be} optimal.

\subsection{Upper Bound on Impact of Optimal Multi-Point Poisoning Attacks}
\label{sec:upper_bound_multi_point_poisoning_attack}

Here, we propose methods to upper-bound the impact of multi-point attacks.
Our methods run in $\mathcal{O}(n+\lambda)$ or $\mathcal{O}((n+\lambda)\log(n+\lambda))$ time, making them highly efficient in practice.
\Rone{This upper bound is important for both attackers and defenders: it quantifies the maximum possible improvement over existing attacks and provides a worst-case performance guarantee for linear regression (see \cref{sec:discussion} for further discussion).}
We begin by relaxing the poisoning problem, and then describe an efficient method for solving it.

\noindent \textbf{Deriving the Upper-Bounding Problem.}
To upper-bound the loss of multi-point attacks, we relax \cref{def:poisoning_linear_regression_on_cdfs} by allowing duplicate poisons and permitting poisons to lie in $\mathcal{K}$.
Formally, we arrive at the following problem:
\begin{definition_box}
\begin{definition}[\textbf{Relaxed Poisoning Problem}]
\label{def:relaxed_poisoning_problem}
Let $\mathcal{K} = \{k_1, k_2, \dots, k_n\} \subset \mathbb{N}$ be $n$ distinct legitimate keys arranged in increasing order, i.e., $k_1 < k_2 < \dots < k_n$.
Let $\lambda \in \mathbb{N}$ be the maximum number of poison points.
The Relaxed Poisoning Problem is to determine the multiset $\mathcal{P}^\ast$ of integers satisfying $|\mathcal{P}^\ast| \leq \lambda$ and $\mathcal{P}^\ast\subset\{k_1,k_1+1,\dots,k_n\}$ that maximizes $E(\mathcal{K} \uplus \mathcal{P}^\ast)$, where $\uplus$ denotes multiset addition.
Formally,
\begin{equation}
    \mathcal{P}^\ast \coloneq \argmax_{\mathcal{P} \text{ s.t. } |\mathcal{P}| \leq \lambda, \mathcal{P}\subset\{k_1,k_1+1,\dots,k_n\}} E(\mathcal{K} \uplus \mathcal{P}),
\end{equation}
\end{definition}
\end{definition_box}
The solution to \cref{def:relaxed_poisoning_problem} provides an upper bound on the loss of the original poisoning problem (\cref{def:poisoning_linear_regression_on_cdfs}), because \cref{def:relaxed_poisoning_problem} relaxes the constraints on the poison set $\mathcal{P}$.
\Rtwo{Note that unlike \cref{def:poisoning_linear_regression_on_cdfs}, this relaxed formulation remains feasible even when $\mathcal{K}=\{k_1,k_1+1,\dots,k_n\}$, because it allows $\mathcal{P}$ to be a multiset and to include values from $\mathcal{K}$.}

By applying an argument similar to that in \cref{thm:structure_of_optimal_multi_point_attack}, we can establish the following theorem regarding the structure of the optimal solution $\mathcal{P}^\ast$ to the Relaxed Poisoning Problem.
\begin{theorem_box}
\begin{theorem}
\label{thm:relaxed_poisoning_problem}
The optimal solution $\mathcal{P}^\ast$ to the Relaxed Poisoning Problem is a multiset over $\mathcal{K}$, i.e., $\mathrm{Supp}(\mathcal{P}^\ast) \subset \mathcal{K}$, where $\mathrm{Supp}$ denotes the support of the multiset.
\end{theorem}
\end{theorem_box}
\begin{proof}[\Rthree{Proof of \cref{thm:relaxed_poisoning_problem}}]
\Rthree{The proof is analogous to that of \cref{thm:structure_of_optimal_multi_point_attack}:
if there exists poison points outside of $\mathcal{K}$, we can increase the loss by shifting them to the keys within $\mathcal{K}$.}
\end{proof}
Thus, letting $\mathcal{Q}_{\mathcal{K}}(\bm{d})$ be the multiset containing $d_i$ copies of $k_i$ for each $i$ (with $\bm{d} \in \mathbb{Z}_{\geq 0}^n$), the relaxed problem reduces to finding
\begin{equation}
    \bm{d}^\ast \coloneq \argmax_{\bm{d} \text{ s.t. } \sum_{i=1}^{n} d_i \leq \lambda} E(\mathcal{K} \uplus \mathcal{Q}_\mathcal{K}(\bm{d})).
\end{equation}

The following theorem states that in the Relaxed Poisoning Problem, it is always optimal to fully exhaust the given budget $\lambda$:
\begin{theorem_box}
\begin{theorem}
\label{thm:d_saturation}
The optimal solution $\bm{d}^\ast$ to the Relaxed Poisoning Problem satisfies $\sum_{i=1}^{n} d_i^\ast = \lambda$.
% Therefore,
% \begin{equation}
%     \max_{\bm{d} \text{ s.t. } \sum_{i=1}^{n} d_i \leq \lambda} E(\mathcal{K} \uplus \mathcal{Q}_\mathcal{K}(\bm{d})) = 
%     \max_{\bm{d} \text{ s.t. } \sum_{i=1}^{n} d_i = \lambda} E(\mathcal{K} \uplus \mathcal{Q}_\mathcal{K}(\bm{d})).
% \end{equation}
\end{theorem}
\end{theorem_box}
\begin{proof}[\Rthree{Proof Sketch of \cref{thm:d_saturation}}]
    \Rthree{We first prove \cref{lem:d_saturation_lemma} (in \cref{app:proof_of_upper_bound}), which states that in the relaxed setting, there exists some $i\in\{1,2,\dots,n\}$ such that adding a poison at $x_i$ increases the loss.
    This implies that using the full poisoning budget is optimal.}
\end{proof}
\Rthree{The full proof is provided in \cref{app:proof_of_upper_bound}.}
% This result implies that in the Relaxed Poisoning Problem, it is always optimal to fully exhaust the given budget $\lambda$.
% By narrowing the search space in this way, we can reduce the computational cost of the algorithm described later.

Next, we derive the following inequality to further upper-bound the solution of the Relaxed Poisoning Problem:
\begin{align}
&~ \max_{\bm{d} \text{ s.t. } \sum_{i=1}^{n} d_i = \lambda} E(\mathcal{K} \uplus \mathcal{Q}_\mathcal{K}(\bm{d})) \\
=&~ \max_{\bm{d} \text{ s.t. } \sum_{i=1}^{n} d_i = \lambda} \min_{w,b} \mathcal{L}(\mathcal{K} \uplus \mathcal{Q}_\mathcal{K}(\bm{d}); w, b) \\
\label{eq:upper_bound_relaxed_poisoning_problem}
\leq&~ \min_{w} \max_{\bm{d} \text{ s.t. } \sum_{i=1}^{n} d_i = \lambda} \min_{b} \mathcal{L}(\mathcal{K} \uplus \mathcal{Q}_\mathcal{K}(\bm{d}); w, b).
\end{align}
The equality follows from the definition of $E(\mathcal{K} \uplus \mathcal{Q}_\mathcal{K}(\bm{d}))$ in \cref{def:linear_regression_on_cdfs}.
The inequality follows from the standard max-min inequality: $\max_{x} \min_{y} f(x, y) \leq \min_{y} \max_{x} f(x, y)$, for any function $f: \mathcal{X} \times \mathcal{Y} \to \mathbb{R}$. 
By computing the value of \cref{eq:upper_bound_relaxed_poisoning_problem}, we obtain an upper bound on the loss achievable by multi-point attacks.

\noindent \textbf{Computing the Upper Bound.}
To determine the value of \cref{eq:upper_bound_relaxed_poisoning_problem}, we make use of the following two theorems.
\begin{theorem_box}
\begin{theorem}
\label{thm:convexity_of_min_b_l_fixed_d}
For fixed $\bm{d}$,  $\min_{b} \mathcal{L}(\mathcal{K} \uplus \mathcal{Q}_\mathcal{K}(\bm{d}); w, b)$ is a convex quadratic function of $w$.
% \begin{equation}
%     \min_{b} \mathcal{L}(\mathcal{K} \uplus \mathcal{Q}_\mathcal{K}(\bm{d}); w, b)
% \end{equation}
% is a convex quadratic function of $w$.
\end{theorem}
\end{theorem_box}
\begin{proof}[\Rthree{Proof Sketch of \cref{thm:convexity_of_min_b_l_fixed_d}}]
    \Rthree{For fixed $w$, the optimal $b^\ast$ can be solved in closed form, and substituting it shows the loss is a quadratic function in $w$ with positive quadratic coefficient.}
\end{proof}
\begin{theorem_box}
\begin{theorem}
\label{thm:optimal_c_for_fixed_w}
    For fixed $w > 0$, the choice of $\bm{d}$ (subject to $\sum_{i=1}^{n} d_i = \lambda$) that maximizes $\min_{b} \mathcal{L}(\mathcal{K} \uplus \mathcal{Q}_\mathcal{K}(\bm{d}); w, b)$ is characterized as follows:
    \begin{itemize}[leftmargin=1.5em]
        \item If $w (k_n - k_1) < n + \lambda$, all poison points are concentrated on a single key; that is, $\exists i \in \{1, 2, \dots, n\}$ such that $d_i = \lambda$.
        \item Otherwise, all poison points are concentrated at the endpoints $k_1$ and $k_n$; that is, $d_1 + d_n = \lambda$.
    \end{itemize}
\end{theorem}
\end{theorem_box}
\begin{proof}[\Rthree{Proof Sketch of \cref{thm:optimal_c_for_fixed_w}}]
    \Rthree{When $w(k_n - k_1) < n+\lambda$, for any $1\leq i < j \leq n$ with $d_i,d_j \geq 1$, moving a poison from $i$ to $j$, or from $j$ to $i$, always increases the loss.
    When $w(k_n - k_1) \geq n+\lambda$, for any $2\leq i \leq n-1$ with $d_i \geq 1$, moving a poison from $i$ to $1$, or from $i$ to $n$, always increases the loss.}
\end{proof}
\Rthree{The full proofs are provided in \cref{app:proof_of_upper_bound}.}
By \cref{thm:optimal_c_for_fixed_w}, the number of possible $\bm{d}$ configurations that maximize $\min_{b} \mathcal{L}(\mathcal{K} \uplus \mathcal{Q}_\mathcal{K}(\bm{d}); w, b)$ is at most $\mathcal{O}(n + \lambda)$.
For each such $\bm{d}$, \cref{thm:convexity_of_min_b_l_fixed_d} implies that $\min_{b} \mathcal{L}(\mathcal{K} \uplus \mathcal{Q}_\mathcal{K}(\bm{d}); w, b)$ is a convex quadratic function of~$w$.
By properly implementing differential computation, we can compute the coefficients of these quadratic functions in time $\mathcal{O}(n + \lambda)$, as detailed in \cref{sec:efficient_calculation_for_upper_bound}.
Thus, evaluating \cref{eq:upper_bound_relaxed_poisoning_problem} reduces to the following problem: given convex quadratic functions $f_i(w): \mathbb{R} \to \mathbb{R}$ for $i \in \{1,2,\dots,\mathcal{O}(n + \lambda)\}$, determine the value of $\min_w \max_i f_i(w)$.
There exist several methods to efficiently solve this problem.
Three representative approaches are as follows:
\begin{enumerate}[leftmargin=1.5em]
    \item \textbf{Golden section search over $w$}: Minimize the convex function $\max_i f_i(w)$ via golden section search~\citep{press2007numerical}. With $T$ iterations, the time complexity is $\mathcal{O}(T(n+\lambda))$.
    \item \textbf{Binary search over $y$ (function values)}: Given $y\in\mathbb{R}$, we can test in $\mathcal{O}(n+\lambda)$ time whether $y \ge \min_w \max_i f_i(w)$. Thus, binary search~\citep{cormen2009introduction} with $T$ iterations can solve the problem in $\mathcal{O}(T(n + \lambda))$ time.
    \item \textbf{Exact structural analysis}: Since $\max_i f_i(w)$ is piecewise quadratic with $\mathcal{O}(n+\lambda)$ pieces by Davenport--Schinzel sequence theory~\citep{sharir1988davenport}, we can construct its exact representation in $\mathcal{O}((n+\lambda)\log(n+\lambda))$ time and then compute the exact minimum directly (unlike the previous two iterative methods).
\end{enumerate}
Detailed algorithms are given in \cref{sec:efficient_calculation_for_upper_bound}.
By employing one of these methods, we can evaluate \cref{eq:upper_bound_relaxed_poisoning_problem} with a time complexity of either $\mathcal{O}(T(n + \lambda))$ or $\mathcal{O}((n + \lambda)\log(n + \lambda))$, thereby obtaining an upper bound on the impact of multi-point poisoning attacks.

\subsection{Segment + Endpoint Attack}
\label{sec:segment_plus_endpoint_attack}

\begin{figure}[t]
    \centering
    \includegraphics[width=\columnwidth]{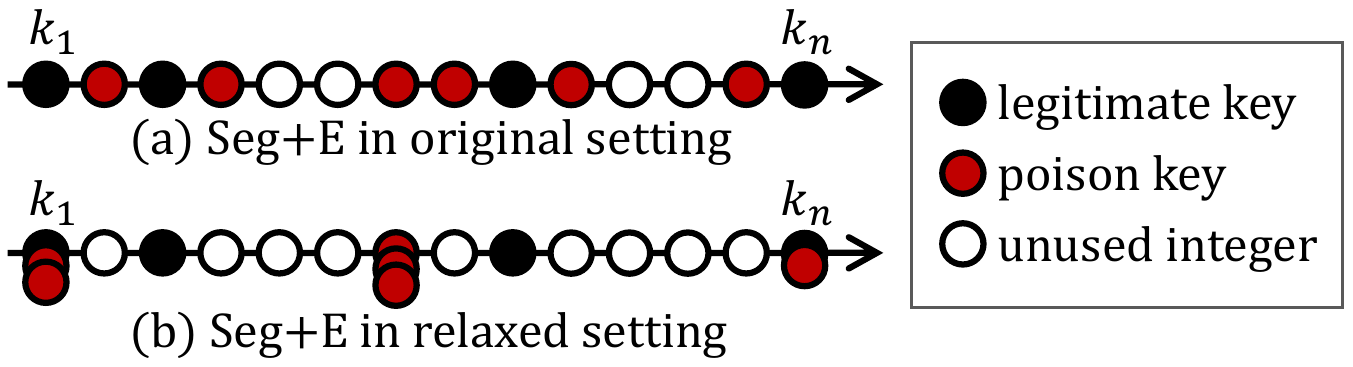}
    \caption{Examples of the Segment + Endpoint (Seg+E) attack in original and relaxed settings.}
    \label{fig:sege}
\end{figure}

Motivated by \cref{thm:optimal_c_for_fixed_w}---which implies that under the relaxed setting with a fixed $w$, the optimal poison mass concentrates either at the two extremes or at a single interior point---we propose a simple structured attack strategy called \emph{Segment + Endpoint (Seg+E)}.
A Seg+E attack uses as poisons at most three consecutive blocks: two on the extreme blocks adjacent to $k_1$ and $k_n$, and a single consecutive segment not adjacent to these extremes. 
We define Seg+E in the relaxed setting (\cref{def:relaxed_poisoning_problem}) analogously: a Seg+E multiset in the relaxed setting consists of at most three integers, namely those corresponding to $k_1$, $k_n$, and a single interior integer.
\cref{fig:sege} illustrates Seg+E attacks in the original and relaxed settings.
The formal mathematical definition is given in \cref{app:segment_plus_endpoint}.

We propose three algorithms for finding the Seg+E solution.
First, we present an $\mathcal{O}(n\lambda^3)$-time exact algorithm for the original setting, followed by an $\mathcal{O}(n\lambda)$-time exact algorithm for the relaxed setting.
Finally, by using the relaxed-setting solution as guidance, we design an $\mathcal{O}(n\lambda)$-time algorithm that finds a solution in the original setting that is empirically very close to the optimal one.

\textbf{$\mathcal{O}(n\lambda^3)$-time algorithm for the original setting.}
In the original setting, we can find an exact Seg+E solution in $\mathcal{O}(n\lambda^3)$ time.
The number of ways to allocate the $\lambda$ poison budget to the left endpoint, the middle segment, and the right endpoint is $\mathcal{O}(\lambda^3)$.
For each such allocation, an argument analogous to the proof of \cref{thm:structure_of_optimal_multi_point_attack} shows that the middle segment admits only $\mathcal{O}(n)$ candidate locations that need to be considered;
see \cref{thm:seg_e_in_original_setting} in \cref{sec:algo_exact_seg_plus_e_original} for a formal proof.
Therefore the total number of Seg+E candidates to evaluate is $\mathcal{O}(n\lambda^3)$, and by computing the MSE for each candidate efficiently via differential updates we obtain the exact Seg+E solution in $\mathcal{O}(n\lambda^3)$ time.

\textbf{$\mathcal{O}(n\lambda)$-time algorithm for the relaxed setting.}
In the relaxed setting, we can find an exact Seg+E solution in $\mathcal{O}(n\lambda)$ time.
Using an argument similar to the proof of \cref{thm:d_saturation}, we first show that any exact Seg+E solution in the relaxed problem exhausts the budget $\lambda$.
As in the original setting, the location of the middle segment admits only $\mathcal{O}(n)$ candidate positions.
Crucially, when the number of poisons placed at the left endpoint and the middle-segment location are fixed, the optimal number of poisons assigned to the middle segment can be determined in $\mathcal{O}(1)$ time, by extending the loss to a continuous function in the middle-segment multiplicity and using its derivative to identify the optimum.
Hence, iterating over all choices of left-endpoint count and middle-segment location (there are $\mathcal{O}(n\lambda)$ such combinations) and evaluating the corresponding MSEs yields an $\mathcal{O}(n\lambda)$-time algorithm for finding the exact Seg+E solution in the relaxed setting.

\textbf{$\mathcal{O}(n\lambda)$-time heuristic algorithm for the original setting.}
We also design an $\mathcal{O}(n\lambda)$-time heuristic algorithm that yields a Seg+E solution that is not guaranteed to be exact but is empirically very close to the exact one.
The heuristic algorithm leverages the Seg+E solution in the relaxed setting as a guide:
we fixed each of the number of poisons at the left endpoint, middle-segment, and right endpoint as the same as the relaxed-setting solution.
Then, we try the $\mathcal{O}(n)$ candidate middle-segment positions and select the best placement.
Because the Seg+E solution in the relaxed setting can be found in $\mathcal{O}(n\lambda)$ time and we can evaluate the MSE for each candidate position in $\mathcal{O}(1)$ time, the overall time complexity of the heuristic algorithm is $\mathcal{O}(n\lambda)$.

\section{Experimental Evaluation}
\label{sec:experimental_evaluation}

In this section, we experimentally evaluate our upper-bounding methods and Seg+E algorithms.
We first assess how tightly the proposed bounds match the attack impact achieved by the greedy method (\cref{sec:experiment:upper_bound_vs_greedy}), then compare the loss achieved by Seg+E and the greedy attack (\cref{sec:experiment:greedy_vs_sege}).
We also measure running time (\cref{sec:experiment:runningtime}), conduct ablation studies on key parameters (\cref{sec:experiment:ablation_study_of_n_and_r}), and present a detailed analysis in the small-scale setting, where the optimum is computable (\cref{sec:experiment:upper_bound_vs_greedy_brute_force}).
\Rone{Finally, we evaluate the impact of poisoning on \emph{lookup time}, i.e., the time to retrieve values from a sorted array using the linear regression model (\cref{sec:experiment:impact_of_poisoning_on_lookup_time}).}

\textbf{Key take-aways:} The main findings are as follows:
\begin{enumerate}[leftmargin=1.5em]
\item For both real-world and synthetic datasets, there is only a small gap between the upper bound and the iterative greedy approach. This indicates that the bound is tight and that the greedy approach achieves near-optimal solutions.
\item Exact Seg+E is never worse than the greedy attack, and the Seg+E solution by our heuristic algorithm always attains a loss very close to that of the exact Seg+E attack.
\item The upper bound can be determined in time faster than running the greedy algorithm, allowing one to effectively judge the quality of the computed solution.
\item The tightness of the bound remains stable across both key density and dataset size, with only moderate degradation observed under extremely dense distributions.
\item The gap decomposition analysis identifies the main source of looseness in the upper bounds: the max-min step dominates when key density is low, while the duplicate-allowing relaxation dominates when key density is extremely high.
\item \Rone{Poisoning substantially increases lookup time: at a 20\% poisoning ratio, it slows down lookups by up to $1.6\times$.}
\end{enumerate}

\subsection{Experimental Setup}

We implemented all upper-bounding, Seg+E, and greedy poisoning algorithms, in C++.
All experiments were conducted on a Linux machine equipped with an Intel\textregistered~ Core\texttrademark~ i9-11900H CPU @ 2.50GHz and 64GB of memory.
We compiled the code using GCC version 9.4.0 with the \texttt{-O3} optimization flag enabled and ran all experiments in a single-threaded setting.
\Rtwo{To verify that our claims hold across typical key-distribution patterns, we evaluate on several synthetic and real-world datasets used in a standard learned-index benchmark.}

\noindent \textbf{Synthetic Datasets.}
Each synthetic legitimate key set is generated based on a specified distribution (Uniform, Normal, or Exponential), random seed, range $R \in \mathbb{N}$, and value of $n \in \mathbb{N}$, using the following procedure:
We sampled $n$ real values, scaled them to $[0, R]$, rounded to integers, and removed duplicates.
\Rone{When $R$ is small, a larger fraction of the integer domain is occupied by legitimate keys, leaving fewer integer values available for the attacker to use.}

\noindent \textbf{Real-World Datasets.}
Each real-world legitimate key set is constructed by extracting $n$ consecutive keys from a real sorted dataset, where the starting position is selected randomly.
We use Amzn, Face, and Osmc datasets from SOSD~\cite{sosd-vldb}, a benchmark for learned indexes.
This setup reflects the practical use case where linear regression models are often deployed as leaf models handling consecutive keys in learned indexes~\cite{ding2020alex,li2021finedex,wang2020sindex,sun2023learned}.

\noindent \textbf{Upper Bound Algorithms.}
We evaluate the three upper bound computation algorithms described in \cref{sec:upper_bound_multi_point_poisoning_attack}.
For the first (golden section search) and second (binary search) approaches, we fix the number of iterations $T$ to 50.
Since the difference among their results is always very small (less than $10^{-9}$), we report results from the first algorithm unless otherwise noted (e.g., for running time evaluation in \cref{sec:experiment:runningtime}).

\subsection{Upper Bound vs Greedy Poisoning}
\label{sec:experiment:upper_bound_vs_greedy}

\begin{figure*}[t]
    \centering
    \includegraphics[width=\textwidth]{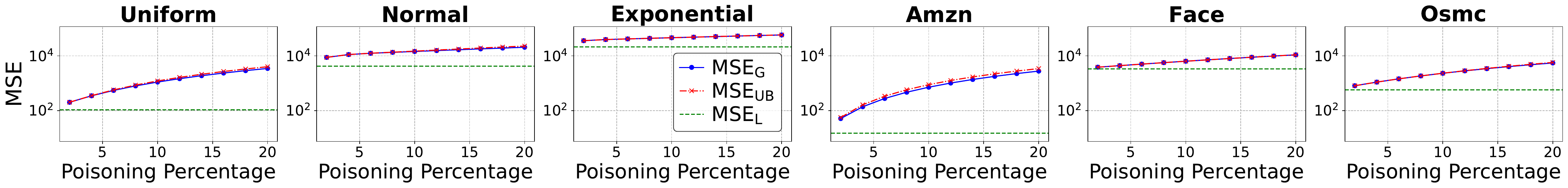}
    \caption{$\mathrm{MSE}_\mathrm{UB}$ (our upper bound), $\mathrm{MSE}_\mathrm{G}$ (greedy attack), and $\mathrm{MSE}_\mathrm{L}$ (legitimate) for seed = 0.
    We set $n=1{,}000$ and $R=100{,}000$.
    We observe that $\mathrm{MSE}_\mathrm{UB} \geq \mathrm{MSE}_\mathrm{G}$ in all cases and the gap is small, which indicates that our upper bound is tight and greedy attack is close to the optimal attack.}
    \label{fig:upper_bound_vs_legitimate_seed0}
\end{figure*}

\begin{figure*}[t]
    \centering
    \includegraphics[width=\textwidth]{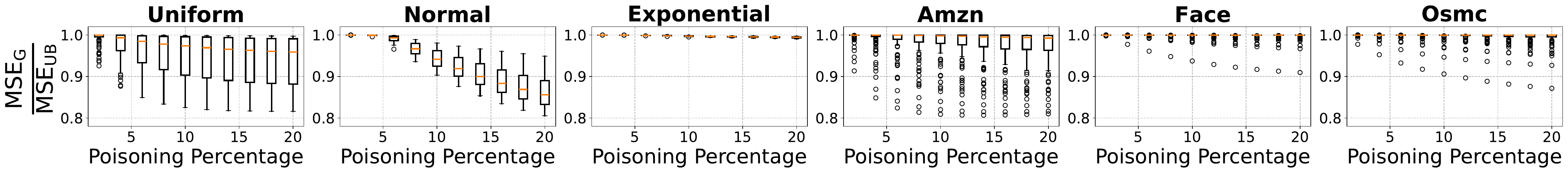}
    \caption{Boxplots of $\mathrm{MSE}_\mathrm{G}/\mathrm{MSE}_\mathrm{UB}$ over 100 seeds for each dataset and poisoning percentage.
    We set $n=1{,}000$ and $R=100{,}000$.
    The ratio never exceeds 1 and is always greater than $0.8$.
    In particular, for Exponential, Amzn, Face, and Osmc, the median ratio is $\geq 0.99$, which indicates that our upper bound is tight and greedy attack is close to the optimal attack.}
    \label{fig:upper_bound_vs_greedy}
\end{figure*}

\begin{figure*}[t]
    \centering
    \includegraphics[width=\textwidth]{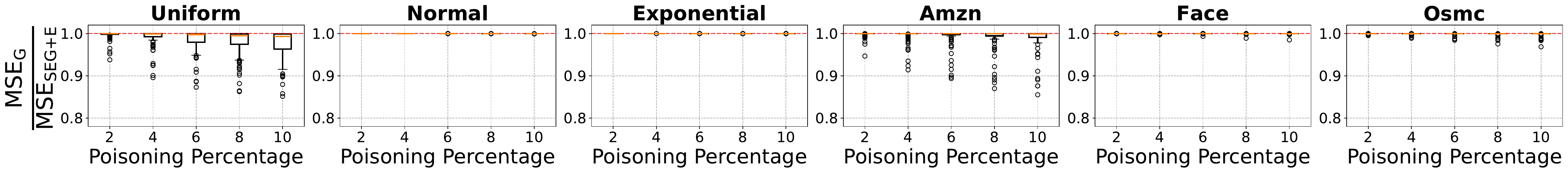}
    \caption{Boxplots of $\mathrm{MSE}_\mathrm{G}/\mathrm{MSE}_\mathrm{Seg+E}$ over 100 seeds for each dataset and poisoning percentage.
    We set $n=1{,}000$ and $R=100{,}000$.
    The ratio never exceeds 1, that is, the Seg+E attack is never worse than the greedy attack.}
    \label{fig:greedy_vs_sege}
\end{figure*}

\begin{figure*}[t]
    \centering
    \includegraphics[width=\textwidth]{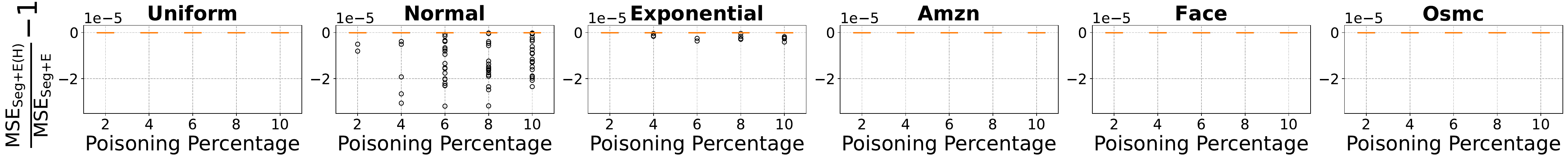}
    \caption{Boxplots of $\mathrm{MSE}_\mathrm{Seg+E(H)}/\mathrm{MSE}_\mathrm{Seg+E} - 1$ over 100 seeds for each dataset and poisoning percentage (for better visibility, we plot the value obtained by subtracting 1 from the ratio).
    We set $n=1{,}000$ and $R=100{,}000$.
    The ratio consistently remains very close to 1, indicating that our heuristic algorithm yields solutions nearly identical to those of the exact Seg+E algorithm.}
    \label{fig:sege_h_vs_sege}
\end{figure*}

Here, we show that the gap between the upper bound computed by our proposed method ($\mathrm{MSE}_\mathrm{UB}$) and the loss from the greedy poisoning method ($\mathrm{MSE}_\mathrm{G}$) is consistently small.
This demonstrates that our upper bound method provides an efficient approach to measuring the impact of multi-point attacks, and that the greedy attack approximates the optimal attack reasonably well.

First, \cref{fig:upper_bound_vs_legitimate_seed0} shows the comparison among $\mathrm{MSE}_\mathrm{G}$, $\mathrm{MSE}_\mathrm{UB}$, and the MSE of legitimate keys ($\mathrm{MSE}_\mathrm{L}$) for seed = 0.
We set $n=1{,}000$ for all datasets, and set $R=100{,}000$ for synthetic data.
The poisoning percentage is defined as $\lambda / n$.

We observe that $\mathrm{MSE}_\mathrm{G}$ increases with larger poisoning percentages.
The MSE of legitimate keys varies significantly depending on the distribution and $n$, and so does the increase in MSE caused by poisoning.
In all cases, $\mathrm{MSE}_\mathrm{UB}$ is greater than $\mathrm{MSE}_\mathrm{G}$ and \Rtwo{\textit{tight} in the sense that $\mathrm{MSE}_\mathrm{G} / \mathrm{MSE}_\mathrm{UB}$ is always at least $0.81$}.

\Rtwo{
We further observe differences across datasets: settings that are locally closer to uniform (e.g., Uniform and Amzn) tend to yield smaller MSE on legitimate keys, yet can exhibit larger increases in MSE after the attack.
For such near-uniform datasets, we also find that the discrepancy between $\mathrm{MSE}_\mathrm{UB}$ and $\mathrm{MSE}_\mathrm{G}$ tends to be larger.
}

Next, \cref{fig:upper_bound_vs_greedy} presents the distribution of the ratio $\mathrm{MSE}_\mathrm{G} / \mathrm{MSE}_\mathrm{UB}$ over 100 seeds.
Values of this ratio close to $1$ indicate a tight upper bound and near-optimal performance of the greedy attack, whereas values close to $0$ indicate a loose bound or weak greedy performance.
Each boxplot represents the distribution across 100 seeds for a given configuration of ($n$, dataset, poisoning percentage).
Note that the randomness comes only from data generation; all algorithms are deterministic.
We set $n=1{,}000$ for all datasets, and set $R=100{,}000$ for synthetic data.

We observe that, in every configuration, $\mathrm{MSE}_\mathrm{G} / \mathrm{MSE}_\mathrm{UB}$ never exceeds 1, and the median over 100 trials is always above 0.85.
In particular, for Exponential, Amzn, Face, and Osmc, the median is consistently above 0.99, showing that our upper bound and the greedy results are very close.
For Uniform and Normal, the gap between upper bound and greedy results tends to grow, especially when the poisoning percentage is large.
That said, across all $6{,}000$ trials, the ratio never falls below 0.8, and the overall average remains as high as 0.97.
This demonstrates that (i) the greedy multi-point poisoning attack leaves at most about 25\% room for improvement (and only around 3\% on average), and (ii) our proposed upper bound is consistently tight across datasets and settings.

\subsection{Greedy Poisoning vs Seg+E}
\label{sec:experiment:greedy_vs_sege}

\begin{figure*}[t]
    \centering
    \includegraphics[width=\textwidth]{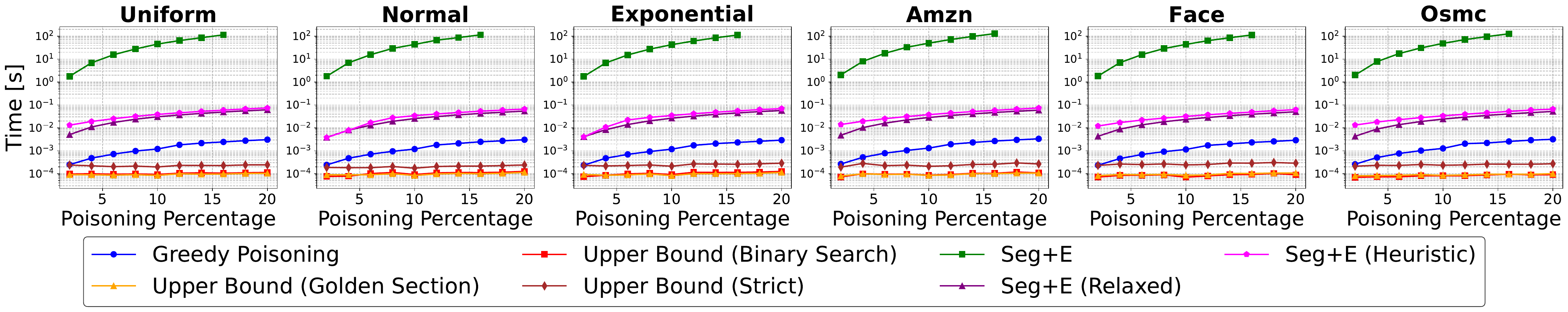}
    \caption{Running time vs poisoning percentage.
    We set $n=1{,}000$ and $R=100{,}000$.
    The running time of the greedy poisoning method grows linearly with the poisoning percentage (consistent with $\mathcal{O}((n + \lambda) \lambda)$ time complexity), whereas our upper-bound methods (Golden, Binary, Strict) are nearly flat (consistent with $\mathcal{O}(n + \lambda)$ or $\mathcal{O}((n + \lambda) \log(n + \lambda))$ time complexity).}
    \label{fig:runningtime_percentage}
\end{figure*}

\begin{figure*}[t]
    \centering
    \includegraphics[width=\textwidth]{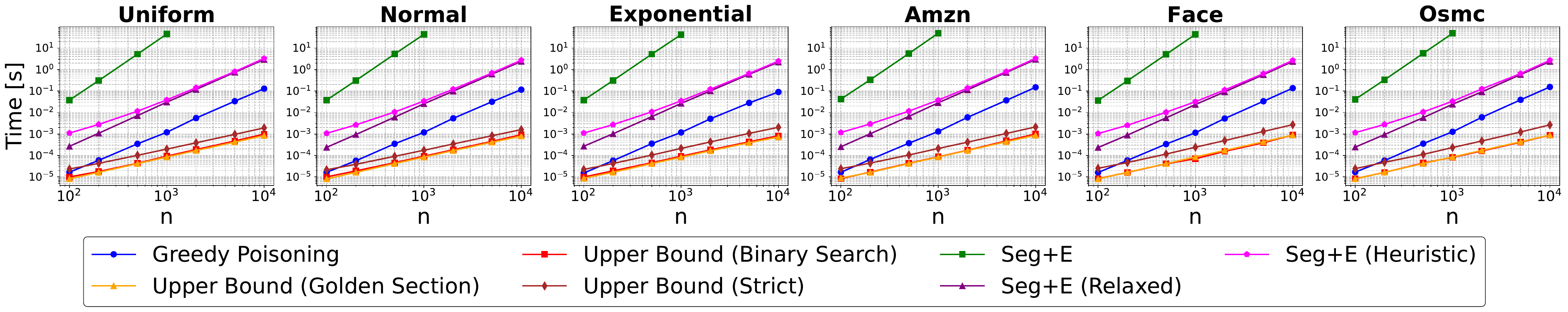}
    \caption{Running time vs $n$.
    We set the poisoning percentage to 10\% and $R=100{,}000$.
    The running time of the greedy poisoning method grows quadratically with $n$ (consistent with $\mathcal{O}((n + \lambda) \lambda)$ time complexity), whereas our upper-bound methods (Golden, Binary, Strict) grow linearly with $n$ (consistent with $\mathcal{O}(n + \lambda)$ or $\mathcal{O}((n + \lambda) \log(n + \lambda))$ time complexity).}
    \label{fig:runningtime_n}
\end{figure*}

Here, we compare $\mathrm{MSE}_\mathrm{G}$, $\mathrm{MSE}_\mathrm{Seg+E}$ (the MSE with the exact Seg+E solution), and $\mathrm{MSE}_\mathrm{Seg+E(H)}$ (the MSE with the Seg+E solution computed by our $\mathcal{O}(n\lambda)$ heuristic algorithm).

\cref{fig:greedy_vs_sege} shows the distribution of $\mathrm{MSE}_\mathrm{G} / \mathrm{MSE}_{\mathrm{Seg+E}}$ over 100 seeds.
\Rtwo{A ratio below 1 means that Seg+E yields a larger MSE than Greedy.}
Across all $3{,}000$ cases (6 datasets $\times$ 5 poisoning percentages $\times$ 100 seeds), the ratio never exceeds 1, \Rtwo{i.e., the exact Seg+E solution is never worse than the Greedy.}
\Rtwo{For the Uniform and Amzn datasets}, the two methods sometimes differ noticeably; the ratio $\mathrm{MSE}_\mathrm{G} / \mathrm{MSE}_\mathrm{Seg+E}$ can drop to below 0.86, \Rtwo{which means that Seg+E yields up to about 1.16 times larger MSE than Greedy.}

\cref{fig:sege_h_vs_sege} shows the distribution of $\mathrm{MSE}_\mathrm{Seg+E(H)}/\mathrm{MSE}_\mathrm{Seg+E}$ over 100 seeds.
Over the entire set of $3{,}000$ measurements, we observed that this ratio is always greater than $0.99996$.
This indicates that the Seg+E attack produced by the fast heuristic algorithm is very close to the one obtained by the exact algorithm.
Although $\mathrm{MSE}_\mathrm{Seg+E(H)}$ is rarely worse than $\mathrm{MSE}_\mathrm{G}$, the maximum observed value of the ratio $\mathrm{MSE}_\mathrm{G}/\mathrm{MSE}_\mathrm{Seg+E(H)}$ is only $1.00004$.
In other words, the heuristic algorithm for finding Seg+E solutions provides an excellent trade-off between computational efficiency and attack effectiveness.

\subsection{Running time Evaluation}
\label{sec:experiment:runningtime}

\begin{figure*}[t]
    \centering
    \includegraphics[width=\textwidth]{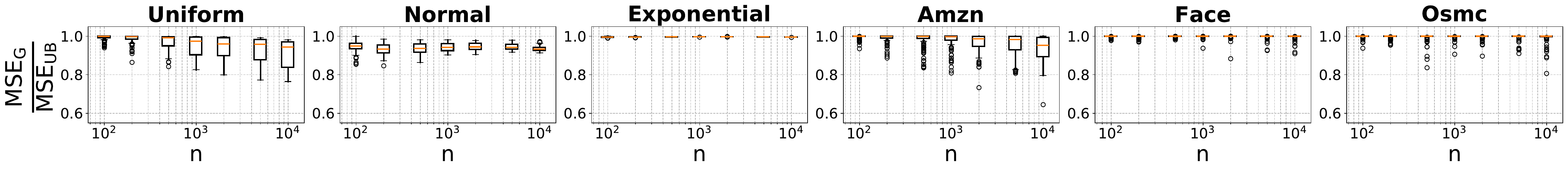}
    \caption{$\mathrm{MSE}_\mathrm{G}/\mathrm{MSE}_\mathrm{UB}$ vs $n$.
    We set the poisoning percentage to 10\% and $R=100{,}000$.
    As $n$ increases, $\mathrm{MSE}_\mathrm{G}/\mathrm{MSE}_\mathrm{UB}$ shows a slight tendency to decrease.}
    \label{fig:ablation_n}
\end{figure*}

\begin{figure}[t]
    \centering
    \includegraphics[width=\columnwidth]{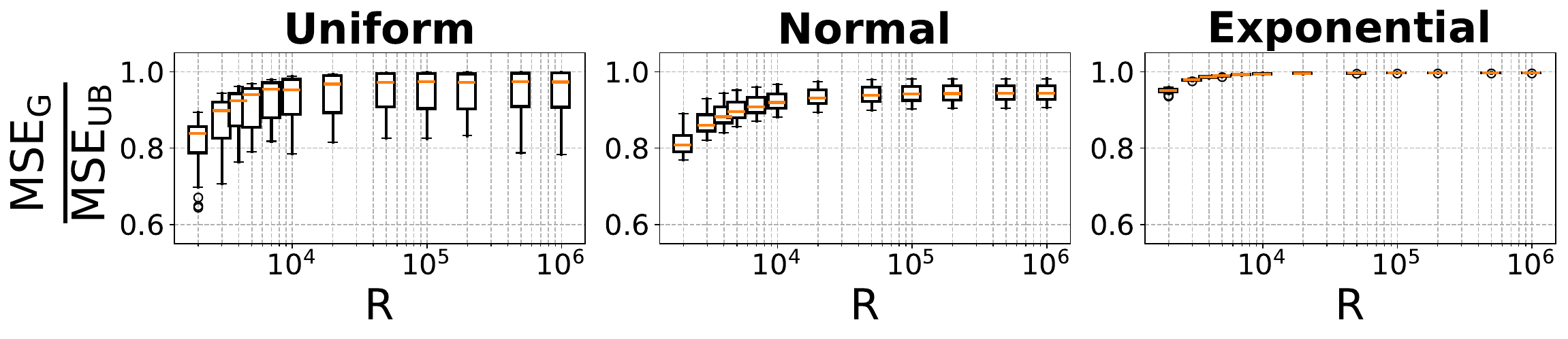}
    \caption{$\mathrm{MSE}_\mathrm{G}/\mathrm{MSE}_\mathrm{UB}$ vs $R$ for synthetic data.
    We set the poisoning percentage to 10\% and $n=1{,}000$ for all datasets.
    When $R$ is quite small, the ratio tends to decrease.}
    \label{fig:ablation_R}
\end{figure}

\begin{figure*}[t]
    \centering
    \includegraphics[width=\textwidth]{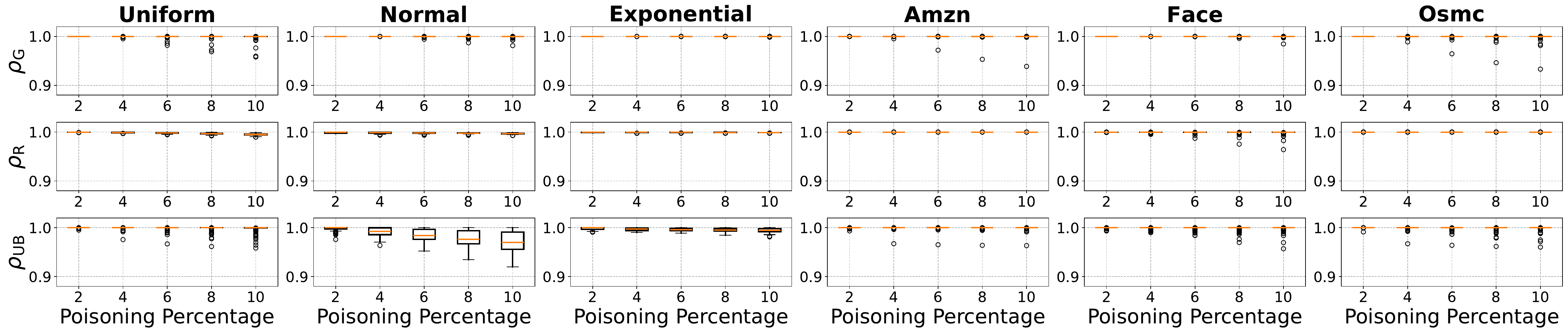}
    \caption{Boxplots of the ratios $\rho_\mathrm{G}$, $\rho_\mathrm{R}$, and $\rho_\mathrm{UB}$, defined in \cref{eq:rho_definitions}, over 100 seeds for each dataset and poisoning percentage.
    We set $n=50$ and $R=1{,}000$.
    All ratios remain close to $1$ (minimum $0.92$), with $\mathrm{MSE}_\mathrm{OPT}/\mathrm{MSE}_\mathrm{ROPT}$ particularly tight.}
    \label{fig:brute_force_vs_greedy}
\end{figure*}

\begin{figure}[t]
    \centering
    \includegraphics[width=\columnwidth]{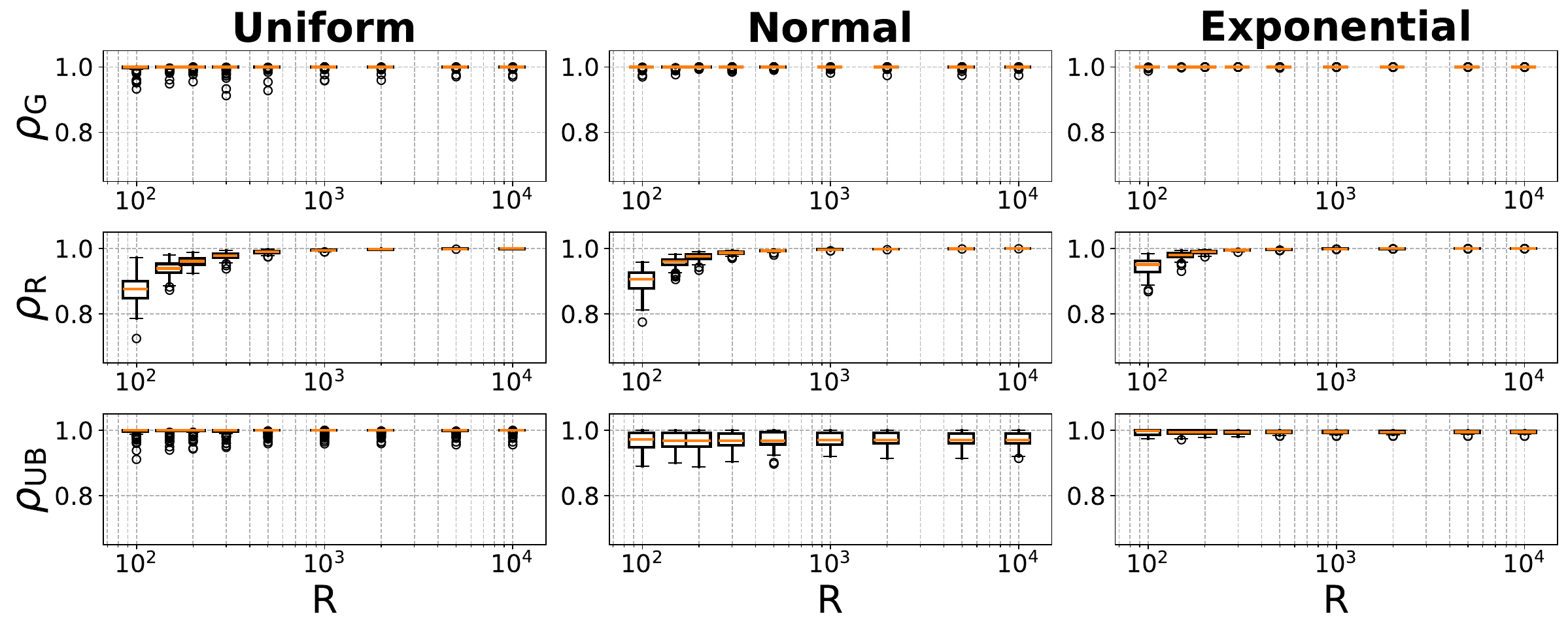}
    \caption{The ratios ($\rho_\mathrm{G}$, $\rho_\mathrm{R}$, and $\rho_\mathrm{UB}$) vs $R$ for synthetic data.
    We set the poisoning percentage to 10\% and $n=50$.
    $\rho_\mathrm{UB}$ and $\rho_\mathrm{G}$ remain stable as $R$ varies, whereas $\rho_\mathrm{R}$ stays close to $1$ for large $R$ but drops when $R$ is small.}
    \label{fig:brute_force_vs_greedy_R}
\end{figure}

\begin{figure}[t]
    \centering
    \includegraphics[width=\columnwidth]{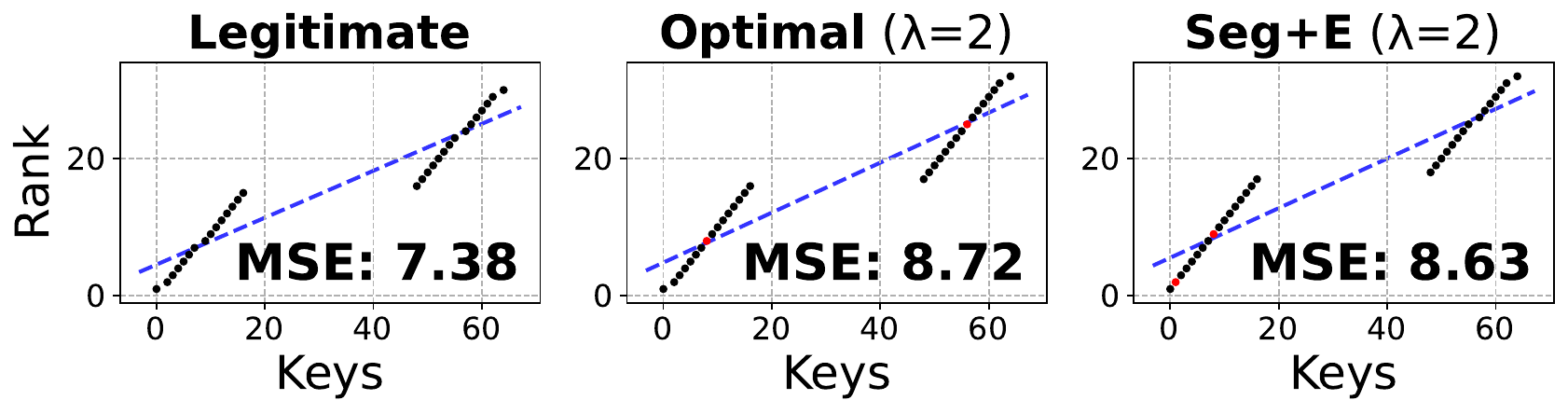}
    \caption{A counter example where the exact Seg+E solution is not globally optimal.
    The legitimate key set is $\mathcal{K}=\{0, 1, \dots, 16, 48, 49, \dots, 64\} \setminus \{1,8,56,63\}$.
    The optimal solution is $\mathcal{P}=\{8,56\}$, but the exact Seg+E solution is $\mathcal{P}=\{1,8\}$.}
    \label{fig:counter_example}
\end{figure}

\begin{figure*}[t]
    \centering
    \includegraphics[width=\textwidth]{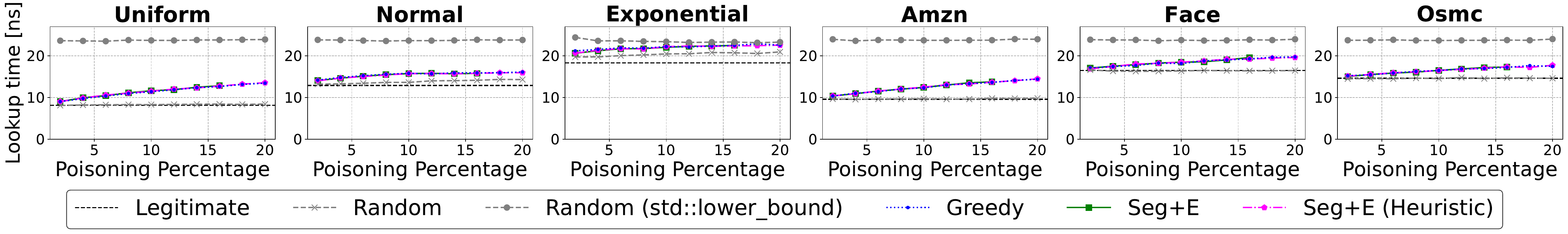}
    \caption{\Rone{Lookup time. The lookup time includes both regression inference and exponential search. We measure it 10 times for each key and report the average. We set $n = 1{,}000$ and $R = 100{,}000$. Poisoning increases the lookup time by up to 1.6 times.}}
    \label{fig:lookuptime_percentage}
\end{figure*}

We measure the running time of our three upper-bounding methods, whose time complexities are $\mathcal{O}(n + \lambda)$, $\mathcal{O}(n + \lambda)$, and $\mathcal{O}((n + \lambda) \log(n + \lambda))$, respectively.
We also report the running time of our Seg+E algorithms (exact: $\mathcal{O}(n\lambda^3)$, heuristic: $\mathcal{O}(n\lambda)$, and exact in the relaxed setting: $\mathcal{O}(n\lambda)$), as well as the running time of the greedy poisoning attack~\citep{kornaropoulos2022price}, which runs in $\mathcal{O}((n + \lambda) \lambda)$ time.

\cref{fig:runningtime_percentage} shows the relationship between poisoning percentage and their running time for each algorithm.
Since $\lambda$ is the product of $n$ and the poisoning percentage, $\lambda$ is proportional to the poisoning percentage.
We set $n=1{,}000$ and $R=100{,}000$.
We see that the running time of the greedy poisoning increases linearly with poisoning percentage, aligning with the complexity $\mathcal{O}((n + \lambda) \lambda)$ (note that $\lambda < n$ in our setting).
In contrast, the proposed upper-bound algorithms (Golden, Binary, Strict) show minimal change in running time.
This is because, when $\lambda < n$, the time complexity of the proposed upper-bound algorithms is $\mathcal{O}(n)$ or $\mathcal{O}(n \log n)$.

As expected, the $\mathcal{O}(n \lambda^3)$-time exact Seg+E algorithm is significantly slower than the other algorithms, taking more than $10{,}000$ times longer than the greedy algorithm.
In contrast, the exact Seg+E algorithm in the relaxed setting and the heuristic Seg+E algorithm, both of which run in $\mathcal{O}(n\lambda)$ time, are much faster.
However, they are still about 20 times slower than the greedy algorithm despite having the same asymptotic time complexity.
This is because the Seg+E algorithms involve heavier constant-time operations, such as solving cubic equations, resulting in a larger constant factor.

\cref{fig:runningtime_n} shows the relationship between $n$ and the running time of each algorithm.
We set the poisoning percentage to 10\% and $R=100{,}000$.
The running time of the greedy poisoning method increases quadratically with $n$, consistent with the time complexity $\mathcal{O}((n + \lambda) \lambda)$.
In contrast, the proposed upper-bound algorithms (Golden, Binary, Strict) scale as $\mathcal{O}(n)$ or $\mathcal{O}(n\log n)$ and are much faster than the greedy algorithm.
In particular, the two approximation methods (Golden and Binary) are extremely fast, with running times no greater than 0.003 seconds.

These results show that our methods compute upper bounds in practical time even for large $n$ and $\lambda$.
\Rone{This efficiency enables downstream uses, such as lightweight attack-quality assessment and as a surrogate objective for higher-level attack-generation procedures; we discuss these use cases in \cref{sec:discussion}.}

% This property allows us to use the upper-bound algorithm to efficiently assess the quality of a computed solution: one can run the upper-bound method concurrently with the greedy method; if the gap between them is small, we can stop, whereas if the gap is large, we can switch to a more sophisticated (and possibly more expensive) attack algorithm to approach the optimal solution.
% Moreover, the upper bound can be used as an approximation of the MSE under an optimal attack. For instance, in the attack on RMI proposed by \citet{kornaropoulos2022price}, the algorithm performs hill climbing by evaluating the post-attack MSE of each linear regression model. In such cases, using the upper bound as a proxy for the MSE under the optimal attack can significantly reduce the computational cost.
% In this way, the fact that our upper-bound algorithm runs faster than the greedy algorithm (and that the bound is tight) makes it practical in various scenarios.

\subsection{Ablation Study of $n$ and $R$}
\label{sec:experiment:ablation_study_of_n_and_r}

Here, we investigate how $\mathrm{MSE}_\mathrm{G} / \mathrm{MSE}_\mathrm{UB}$ changes with respect to $n$ and $R$.
\cref{fig:ablation_n} shows the relationship between $n$ and the ratio $\mathrm{MSE}_\mathrm{G} / \mathrm{MSE}_\mathrm{UB}$ for synthetic data.
We set the poisoning percentage to 10\% and $R=100{,}000$ for synthetic data.
As $n$ increases, we observe a slight decreasing trend in the ratio.
That is, the ratio remains consistently tight: the minimum value is $0.64$ (or $0.73$ if excluding a single outlier seed from Amzn), while the average value is always above $0.91$.
This indicates that, regardless of $n$, our proposed upper bound tightly bounds the attack impact.

\cref{fig:ablation_R} shows the relationship between $R$ and $\mathrm{MSE}_\mathrm{G} / \mathrm{MSE}_\mathrm{UB}$.
We set the poisoning percentage to 10\% and $n=1{,}000$ for all datasets.
We observe that the smaller the value of $R$, the smaller the ratio tends to be.
This is expected because our upper bound is obtained by relaxing the constraint to allow duplicate keys, and the effect of this relaxation becomes more pronounced when $R$ is small.
In particular, the ratio worsens when $R < 10^4$.
However, note that $R=10^4$ already corresponds to a fairly narrow key range (i.e., a dense distribution), where 10\% of the integers contained in the range are legitimate.
In the real datasets we used, the range spanned by $n=1{,}000$ consecutive keys is at least $1.8\times 10^{13}$ (Amzn), $2.8\times 10^4$ (Face), and $2.3\times 10^{10}$ (Osmc).
Thus, the degradation observed for small $R$ is negligible in practice, and our upper bound remains tight under realistic settings.

\subsection{Detailed Analysis in Small-Scale Settings}
\label{sec:experiment:upper_bound_vs_greedy_brute_force}

Here, we present a detailed analysis in the small-scale setting, where our theory makes it possible to obtain the optimal solution within practical time.
As explained in \cref{sec:multi_point_poisoning_attack}, \cref{thm:structure_of_optimal_multi_point_attack} implies that the number of candidates for the optimal solution to the original problem is at most $\binom{2n - 2 + \lambda}{\lambda}$, allowing us to compute $\mathrm{MSE}_\mathrm{OPT}$ in $\mathcal{O}\left((n+\lambda)\binom{2n - 2 + \lambda}{\lambda}\right)$ time.
Similarly, \cref{thm:structure_of_optimal_multi_point_attack} and \cref{thm:d_saturation} imply that there are at most $\binom{n+\lambda-1}{\lambda}$ candidates in the relaxed setting, allowing us to compute $\mathrm{MSE}_\mathrm{ROPT}$ in $\mathcal{O}\left((n+\lambda)\binom{n+\lambda-1}{\lambda}\right)$ time.
Hence, we can compute the exact optimum for moderately small $n$ and $\lambda$; in our experiments, we set $n=50$ and $\lambda \leq 5$.

First, we examine the gap between the greedy solution and the upper bound.
Since our algorithm produces upper bounds by further upper-bounding the optimal solution of the relaxed problem, letting $\mathrm{MSE}_\mathrm{ROPT}$ as the MSE under the optimal solution of the relaxed problem, the inequalities $\mathrm{MSE}_\mathrm{G} \leq \mathrm{MSE}_\mathrm{OPT} \leq \mathrm{MSE}_\mathrm{ROPT} \leq \mathrm{MSE}_\mathrm{UB}$ hold.
Accordingly, we evaluate the following three ratios:
\begin{equation}
\label{eq:rho_definitions}
    \rho_\mathrm{G} \coloneq \frac{\mathrm{MSE}_\mathrm{G}}{\mathrm{MSE}_\mathrm{OPT}}, \quad
    \rho_\mathrm{R} \coloneq \frac{\mathrm{MSE}_\mathrm{OPT}}{\mathrm{MSE}_\mathrm{ROPT}}, \quad
    \rho_\mathrm{UB} \coloneq \frac{\mathrm{MSE}_\mathrm{ROPT}}{\mathrm{MSE}_\mathrm{UB}}.
\end{equation}
All of these ratios are less than or equal to 1.
Values of $\rho_\mathrm{G}$ near $1$ indicate that the greedy attack is close to optimal,
values of $\rho_\mathrm{R}$ near $1$ indicate that the gap introduced by the relaxation is small,
and values of $\rho_\mathrm{UB}$ near $1$ indicate that the max-min upper-bounding step (\cref{eq:upper_bound_relaxed_poisoning_problem}) is tight.
Note that the product of these ratios equals the ratio examined in the previous sections, $\mathrm{MSE}_\mathrm{G}/\mathrm{MSE}_\mathrm{UB}$.
By decomposing it in this manner, we can analyze in detail which component contributes to the observed deviation from $1$.

\cref{fig:brute_force_vs_greedy} reports the distributions of $\rho_\mathrm{G}$, $\rho_\mathrm{R}$, and $\rho_\mathrm{UB}$ over 100 seeds;
we zoom the $y$-axis to a narrow neighborhood of 1 for readability.
Across datasets and settings, all three ratios are consistently close to 1 (all exceed $0.92$) and exhibit a mild decreasing trend as $\lambda$ increases.
We observe that $\rho_\mathrm{G}$ tends to be smaller (i.e., the gap between the greedy and the optimal solutions is larger) on Uniform, Amzn, and Osmc.
We also observe that $\rho_\mathrm{R}$ is generally closer to 1 than $\rho_\mathrm{UB}$, suggesting that the main source of looseness comes from the max-min inequality rather than from permitting duplicate poisons.
We also note that $\rho_\mathrm{UB}$ is most frequently smaller on the Normal dataset, whereas $\rho_\mathrm{R}$ tends to be smaller on the Face dataset.

\cref{fig:brute_force_vs_greedy_R} studies how the ratios vary with $R$ on synthetic data (with $n=50$ and a 10\% poisoning percentage).
We find that both $\rho_\mathrm{UB}$ and $\rho_\mathrm{G}$ are largely insensitive to $R$.
In contrast, $\rho_\mathrm{R}$ is close to 1 for moderately large $R$ (e.g., $R\ge 1{,}000$) but can be somewhat smaller when $R$ is small.
Taken together with the ablation in \cref{sec:experiment:ablation_study_of_n_and_r}, these findings reinforce that the dominant contributor to the overall gap at small $R$ is the duplicate-allowing relaxation, whereas for larger $R$ the relaxation is tight and the remaining looseness mainly stems from the max-min upper-bounding step.

Next, we compare Seg+E with the optimal solution.
We set $n=50$ and $R=1{,}000$.
The summary statistics of $\mathrm{MSE}_\mathrm{Seg+E}/\mathrm{MSE}_\mathrm{OPT}$ for both the original and relaxed settings are reported in \cref{tab:seg_e} as \textit{Empirical Performance (Ratio)}.
We observed that, under both the original and relaxed settings, the exact Seg+E solution always matched the global optimal solution in all $3{,}000$ cases (6 datasets, 5 poisoning percentages, 100 seeds).
However, in the original setting, we have already found corner cases where the Seg+E solution is not globally optimal.
\cref{fig:counter_example} (along with additional corner cases in \cref{app:observations_and_conjectures}) shows one such example.
Here, the optimal solution is $\mathcal{P}=\{8,56\}$, consisting of two isolated points, which is not Seg+E and has larger MSE than the exact Seg+E solution.
Thus, while there exist rare corner cases where Seg+E is suboptimal, it almost always coincides with the global optimum in practice.

\subsection{\Rone{Impact of Poisoning on Lookup Time}}
\label{sec:experiment:impact_of_poisoning_on_lookup_time}

\Rone{
To quantify the practical slowdown caused by poisoning in linear models of learned indexes, we measure its impact on the \emph{lookup time}.
We define lookup time as the time from receiving a query key to obtaining its position in the sorted array, i.e., the total time for linear-regression inference and an exponential search around the predicted position.
We query all legitimate keys, repeat the measurement 10 times per key, and report the average.
We set $n=1{,}000$ and $R=100{,}000$.
As a reference, we include a random-poisoning baseline, where poisons are sampled uniformly from the key domain, and we also report the time of \texttt{std::lower\_bound}.
}

\Rone{
\Cref{fig:lookuptime_percentage} reports lookup time before and after poisoning.
We find that poisoning increases lookup time by up to $1.6\times$ at 20\% poisoning.
By contrast, randomly chosen poisons barely affect lookup time, indicating that poisons generated by Greedy~\cite{kornaropoulos2022price} or our Seg+E algorithm increase the lookup time effectively.
These results quantify how poisoning can slow down the leaf-model lookup routine by degrading the model accuracy, suggesting that poisoning can measurably increase the end-to-end lookup time in learned indexes.
}

\section{\Rone{Discussion}}
\label{sec:discussion}

\Rone{\textbf{On Defending Against Poisoning Attacks.}
Defending linear regression over the CDF against poisoning attacks is challenging, as also discussed in \cite{kornaropoulos2022price}.}
\Rtwo{In our experiments, adding only 4\% poison points increased the MSE by up to 10$\times$, suggesting that simply tracking the cardinality of the key set is unlikely to be sufficient for reliable poisoning detection.}
\Rone{Furthermore, poisons are not outliers in this setting: they lie within the domain and are placed adjacent to legitimate keys, making them hard to detect or mitigate with robust regression methods such as Huber regression~\cite{huber1992robust} and RANSAC~\cite{fischler1981random}.
Developing defenses tailored to regression over the CDF remains an important direction for future work.}

\Rone{\textbf{On Using the Upper Bound for Defense.}
Our upper bound provides a worst-case guarantee that the poisoning impact never exceeds the bound.
For example, if an index must operate within a tolerance of at most a $10\times$ increase in MSE, a defender can estimate how many additional keys can be inserted while ensuring that the MSE provably remains within this range.}

\Rone{\textbf{On Using the Upper Bound for Attack.}}
\Rone{Computing the upper bound is substantially faster than the greedy attack (\cref{sec:experiment:runningtime}), enabling several downstream uses.
First, an attacker can use the bound to cheaply assess candidate solutions: one can run the upper-bound method in parallel with the greedy attack and stop early when the gap is small; otherwise, switch to a more sophisticated (and potentially more expensive) attack algorithm to approach the optimum.
Second, the bound can serve as a proxy for the post-attack MSE in higher-level attack-generation procedures.
For example, the RMI attack of \citet{kornaropoulos2022price} performs hill climbing by repeatedly evaluating the post-attack MSE of each linear model; replacing these evaluations with our bound can substantially reduce the computational cost.}

\Rthree{\textbf{On Empirical Observations and Conjectures.}
Based on extensive experiments and analyses, we report several empirical observations and formulate key conjectures (see Appendix~\ref{app:observations_and_conjectures} for details).
First, we observe that optimal poisons tend to concentrate near the endpoints or near locations where the regression line intersects the point sequence induced by the legitimate keys (\cref{observ:1});
\Cref{app:observations_and_conjectures} also provides an intuitive mathematical explanation for this phenomenon.
As observed in \cref{sec:experiment:upper_bound_vs_greedy_brute_force}, Seg+E matches the optimum in many cases, while \cref{fig:counter_example} shows intentionally constructed corner cases where Seg+E is not optimal (\cref{observ:2}).
Moreover, in every instance we tested, Seg+E is optimal in the relaxed setting; we leave as a conjecture that this holds in general (\cref{con:2}).
We also propose an auxiliary conjecture (\cref{con:3}) which, if proven, would imply \cref{con:2}.}

\Rtwo{\textbf{On Non-Integer Keys.}
Although we state the problem for integer keys, the same analyses can be extended to non-integer keys.
For ordered non-integer keys (e.g., strings), we can apply an order-preserving bijection to map them to a contiguous integer range.
For multi-dimensional keys, as in prior spatial learned indexes~\cite{wang2019learned,wang2020spatial}, we can map them to one-dimensional integer keys via space-filling curves.
Training and poisoning are performed on the resulting integers, and all of our analyses apply directly.}

\Rone{\textbf{On the Implications for Learned Indexes.}
Directly applying our analysis to learned indexes (e.g., RMI~\cite{kraska2018case_li}) is not feasible because they are hierarchical and nonlinear.
However, typical learned indexes rely on a large number of linear regression models.
In such cases, our analysis can be applied to obtain (near-)optimal attacks or upper bounds for each linear regression model.
Moreover, motivated by the performance of our Seg+E technique, we conjecture that for some models, a global worst-case attack influences only a small part of the model.}

\Rthree{\textbf{On Dynamic Settings.}
Our contributions may also offer insights into dynamic settings, where the key set evolves through updates and an attacker injects poisons sequentially to maximize cumulative loss.
A key challenge is to quantify the gap between online attacks (which choose poisons without knowing future updates) and offline attacks (which plan poisons with full knowledge of the update sequence), and to understand how defenses should adapt.
Our upper bound enables lightweight evaluation of the worst-case impact at each update step.
Moreover, the structure of our Seg+E approach may inspire simple heuristics for online poison selection, and it may help defenders design triggers for poisoning detection and retraining.}

\section{Limitations and Future Work}
\label{sec:limitations_and_future_work}

One limitation of our theoretical guarantees is that we were unable to prove the convexity of the loss function in single-point poisoning attacks.
While experimental results consistently suggest that this convexity holds, we could not provide a formal proof.
Although our current findings are sufficient to ensure that the single-point poisoning attack can efficiently lead to the optimal solution, proving this convexity would be a valuable theoretical insight.

Additionally, developing methods to compute tighter bounds remains an important direction for future research.
We observed that the gap between our upper bound and the greedy solution can sometimes be relatively large.
Part of this gap reflects the difference between the greedy method and the true optimum, but some of it stems from the looseness of our bound itself (as discussed in \cref{sec:experiment:upper_bound_vs_greedy_brute_force}).
This looseness stems from two approximations made in the derivation of the bound:
(i) allowing duplicate poison keys, and
(ii) swapping the order of $\min_{w}$ and $\max_{\bm{d}}$ in \cref{eq:upper_bound_relaxed_poisoning_problem}.
Avoiding these approximations and designing methods to obtain tighter bounds remain important topics for future work.

Another important future direction is to extend our analysis beyond linear regression.
We deliberately focused on linear regression, as it is the simplest and most widely used leaf model in learned indexes.
This restriction enabled a clean theoretical foundation and clear insights into adversarial robustness, but it also limits generality.
A natural next step is to broaden the framework to nonlinear regression models (e.g., higher-order polynomials or neural networks) and to multi-stage learned indexes.

% \Rthree{Finally, extending our study to dynamic settings, where legitimate keys can be inserted or deleted over time, is an important direction for future work.
% In such scenarios, an attacker may need to place poison keys online, without knowing future updates, aiming to maximize the cumulative loss over the evolving key set.
% A promising direction is to evaluate, both empirically and theoretically, the gap between the online attack and the offline optimum, i.e., the best attack computed with full knowledge of the entire update sequence.
% Another key question is how defenses should adapt under continuous updates.}

\section{Conclusion}
\label{sec:conclusion}

In this work, we provided the first theoretical advancements on single-point and multi-point poisoning attacks against linear regression models trained on CDFs.
For single-point attacks, we proved that a previously proposed heuristic---whose optimality had remained unclear---is indeed guaranteed to yield the optimal solution.
For multi-point attacks, we derived the key structural properties of the optimal attack, proposed methods to guarantee an upper bound on the achievable loss, and proposed algorithms to find exact and heuristic Seg+E solutions efficiently.
Through experiments, we confirmed that the greedy poisoning method achieves near-optimal attacks, and at the same time, our methods tightly upper-bound the actual attack impact.
Overall, our contributions advance the theoretical understanding of linear regression models on CDFs, providing a foundation for analyzing both the vulnerabilities and robustness of learned indexes under adversarial conditions.

\clearpage
\bibliographystyle{ACM-Reference-Format}
\bibliography{
    bib/attack,
    bib/classic_index,
    bib/learned_data_structures,
    bib/learned_index,
    bib/learned_sort,
    bib/others
}

\clearpage
\appendix

\section{Proof of the Lemma on Single-Point Poisoning Attack}
\label{app:proof_of_lem_single_point_poisoning_attack}

Here, we prove \cref{lem:single_point_poisoning_attack}.
\begin{proof}[Proof of \cref{lem:single_point_poisoning_attack}]
From the closed-form solution of \cref{def:linear_regression_on_cdfs}, $E(\mathcal{X})$ can be expressed as follows:
\begin{equation}
    E(\mathcal{X}) = \mathrm{Var}_\mathrm{R} - \frac{{\mathrm{Cov}_\mathrm{XR}}^2}{\mathrm{Var}_\mathrm{X}}.
\end{equation}
Therefore,
\begin{align}
    \frac{d E(\mathcal{X})}{dx_i} &= - \frac{2 \mathrm{Cov}_\mathrm{XR} \frac{d \mathrm{Cov}_\mathrm{XR}}{dx_i} \mathrm{Var}_\mathrm{X} - \mathrm{Cov}_\mathrm{XR}^2 \frac{d \mathrm{Var}_\mathrm{X}}{dx_i}}{{\mathrm{Var}_\mathrm{X}}^2} \\
    &= \frac{\mathrm{Cov}_\mathrm{XR}}{{\mathrm{Var}_\mathrm{X}}^2} \left( - 2 \frac{d \mathrm{Cov}_\mathrm{XR}}{dx_i} \mathrm{Var}_\mathrm{X} + \mathrm{Cov}_\mathrm{XR} \frac{d \mathrm{Var}_\mathrm{X}}{dx_i} \right).
\end{align}
Now, let $g_i(\mathcal{X})$ be defined as follows:
\begin{equation}
\label{eq:g_i_x}
    g_i(\mathcal{X}) = - 2 \frac{d \mathrm{Cov}_\mathrm{XR}}{dx_i} \mathrm{Var}_\mathrm{X} + \mathrm{Cov}_\mathrm{XR} \frac{d \mathrm{Var}_\mathrm{X}}{dx_i}.
\end{equation}
Since $\mathrm{Cov}_\mathrm{XR} > 0$ and $\mathrm{Var}_\mathrm{X} > 0$,
\begin{equation}
\label{eq:sign_of_d_e_x_i}
    \mathrm{sign}\left(\frac{d E(\mathcal{X})}{dx_i}\right) = \mathrm{sign}(g_i(\mathcal{X})).
\end{equation}
Taking the derivative of \cref{eq:g_i_x} with respect to $x_i$, we get:
\begin{align}
    \frac{d g_i}{dx_i} (\mathcal{X}) &= - 2 \frac{d^2 \mathrm{Cov}_\mathrm{XR}}{{dx_i}^2} \mathrm{Var}_\mathrm{X} - \frac{d \mathrm{Cov}_\mathrm{XR}}{dx_i} \frac{d \mathrm{Var}_\mathrm{X}}{dx_i} + \mathrm{Cov}_\mathrm{XR} \frac{d^2 \mathrm{Var}_\mathrm{X}}{{dx_i}^2} \\
    &= - \frac{d \mathrm{Cov}_\mathrm{XR}}{dx_i} \frac{d \mathrm{Var}_\mathrm{X}}{dx_i} + \mathrm{Cov}_\mathrm{XR} \frac{d^2 \mathrm{Var}_\mathrm{X}}{{dx_i}^2}.
\end{align}
Here, we use the fact that $\mathrm{Cov}_\mathrm{XR}$ does not contain ${x_i}^2$ terms.
Now, since
\begin{align}
    \frac{d \mathrm{Cov}_\mathrm{XR}}{dx_i} &= \frac{d}{dx_i} \left( \frac{1}{n} \sum_{j=1}^{n} x_j j - \bar{x} \bar{r} \right)
    = \frac{1}{n} \left( i - \bar{r} \right), \\
    \frac{d \mathrm{Var}_\mathrm{X}}{dx_i} &= \frac{d}{dx_i} \left(\left( \frac{1}{n} \sum_{j=1}^{n} {x_j}^2 \right) - \bar{x}^2 \right)
    = \frac{2}{n} \left( x_i - \bar{x} \right). \\ 
    \frac{d^2 \mathrm{Var}_\mathrm{X}}{{dx_i}^2} &= \frac{d}{dx_i} \left(\frac{2}{n} \left( x_i - \bar{x} \right)\right)
    = \frac{2}{n} \left( 1 - \frac{1}{n} \right),
\end{align}
we can rewrite $\frac{d g_i}{dx_i} (\mathcal{X})$ as follows:
\begin{align}
    \frac{d g_i}{dx_i} (\mathcal{X}) &= - \frac{1}{n} \left( i - \bar{r} \right) \frac{2}{n} \left( x_i - \bar{x} \right) + \mathrm{Cov}_\mathrm{XR} \frac{2}{n} \left( 1 - \frac{1}{n} \right) \\
    &= \frac{2}{n^2} \left( - (i - \bar{r})(x_i - \bar{x}) + (n - 1) \mathrm{Cov}_\mathrm{XR} \right). \label{eq:dg_i_dx_i}
\end{align}
Now, let $\bm{x}' \in \mathbb{R}^{n-1}$ and $\bm{r}' \in \mathbb{N}^{n-1}$ be defined as follows:
\begin{align}
    \bm{x}' &\coloneq [x_1, x_2, \dots, x_{i-1}, x_{i+1}, \dots, x_n]\\
    \bm{r}' &\coloneq [1, 2, \dots, i-1, i+1, \dots, n].
\end{align}
Since $\bm{x}'$ and $\bm{r}'$ are both strictly monotonically increasing, $\mathrm{Cov}_\mathrm{X'R'}$ is positive, i.e.,
\begin{equation}
    \label{eq:cov_x'r'_is_positive}
    \mathrm{Cov}_\mathrm{X'R'} = \frac{1}{n-1} \sum_{j=1}^{n-1} x'_j r'_j - \left( \frac{1}{n-1} \sum_{j=1}^{n-1} x'_j \right) \left( \frac{1}{n-1} \sum_{j=1}^{n-1} r'_j \right) > 0.
\end{equation}
Now, since
\begin{align}
    \sum_{j=1}^{n-1} x'_j r'_j
&= \sum_{j=1}^{n} x_j j - x_i i
= n \left( \frac{1}{n} \sum_{j=1}^{n} x_j j\right) - x_i i \\
&= n \left( \mathrm{Cov}_\mathrm{XR} + \bar{x} \bar{r} \right) - x_i i. \\
\sum_{j=1}^{n-1} x'_j &= n \bar{x} - x_i,\\
\sum_{j=1}^{n-1} r'_j &= n \bar{r} - i,
\end{align}
we can rewrite $\mathrm{Cov}_\mathrm{X'R'}$ as follows:
\begin{align}
    \mathrm{Cov}_\mathrm{X'R'} &= \frac{1}{n-1} \left( n (\mathrm{Cov}_\mathrm{XR} + \bar{x} \bar{r}) - x_i i \right) - \left( \frac{n \bar{x} - x_i}{n-1} \right) \left( \frac{n \bar{r} - i}{n-1} \right) \\
    &= \frac{n}{(n-1)^2} \Big[(n-1)(\mathrm{Cov}_\mathrm{XR} + \bar{x} \bar{r}) - \frac{n-1}{n} x_i i \notag \\ &\qquad\qquad\qquad - \frac{1}{n} \left( n \bar{x} - x_i \right) \left( n \bar{r} - i \right)\Big] \\
    &= \frac{n}{(n-1)^2} \Big[ - \bar{x} \bar{r} -x_i i + x_i \bar{r} + i \bar{x} + (n-1) \mathrm{Cov}_\mathrm{XR} \Big] \\
    &= \frac{n}{(n-1)^2} \left( - (i - \bar{r})(x_i - \bar{x}) + (n-1) \mathrm{Cov}_\mathrm{XR} \right). \label{eq:con_x'r'}
\end{align}
From \cref{eq:cov_x'r'_is_positive,eq:con_x'r'}, we have
\begin{equation}
    - (i - \bar{r})(x_i - \bar{x}) + (n-1) \mathrm{Cov}_\mathrm{XR} > 0.
\end{equation}
Therefore, from \cref{eq:dg_i_dx_i}, we have
\begin{equation}
    \frac{d g_i}{dx_i} (\mathcal{X}) > 0.
\end{equation}
Therefore, $g_i(\mathcal{X})$ is strictly monotonically increasing with respect to $x_i$.
From \cref{eq:sign_of_d_e_x_i}, $\mathrm{sign}\left(\frac{d E(\mathcal{X})}{dx_i}\right)$ is non-decreasing with respect to $x_i$, and there is at most one $x_i$ such that $\frac{d E(\mathcal{X})}{dx_i}=0$.
\end{proof}

\section{Proof of Theorem and Lemma on Multi-Point Poisoning Attack}
\label{app:proof_of_thm_structure_of_optimal_multi_point_attack}

To prove \cref{thm:structure_of_optimal_multi_point_attack}, we first establish the following lemma.
\begin{theorem_box}
\begin{lemma}
\label{lem:multi_point_poisoning_attack}
    Let $\mathcal{X}$ be the set of keys stored in the index, with $m = |\mathcal{X}| \geq 3$, and let $x_1 \leq x_2 \leq \dots \leq x_m$ be the elements of $\mathcal{X}$ in increasing order.
    For $2 \leq i \leq j \leq m - 1$ and $\epsilon \in (x_{i-1} - x_i, x_{j+1} - x_j)$, define:
    \begin{equation}
        \mathcal{X}_{i,j,\epsilon} = \{x_1, x_2, \dots, x_{i-1}, x_i + \epsilon, x_{i+1} + \epsilon, \dots, x_j + \epsilon, x_{j+1}, \dots, x_m\}.
    \end{equation}
    Then, $\mathrm{sign}\left(\frac{d E(\mathcal{X}_{i,j,\epsilon})}{d\epsilon}\right)$ is non-decreasing over $\epsilon \in (x_{i-1} - x_i, x_{j+1} - x_j)$, and there exists at most one value of $\epsilon$ in this interval for which $\frac{d E(\mathcal{X}_{i,j,\epsilon})}{d\epsilon} = 0$.
\end{lemma}
\end{theorem_box}
\begin{proof}[Proof of \cref{lem:multi_point_poisoning_attack}]
Let $\bar{x}(\epsilon)$ be the average of $\mathcal{X}_{i,j,\epsilon}$, $\mathrm{Var}_\mathrm{X}(\epsilon)$ be the variance of $\mathcal{X}_{i,j,\epsilon}$, and $\mathrm{Cov}_\mathrm{XR}(\epsilon)$ be the covariance between $\mathcal{X}_{i,j,\epsilon}$ and its rank.
As in the case of single-point poisoning attacks, we have
\begin{equation}
    E(\mathcal{X}_{i,j,\epsilon})=\mathrm{Var}_\mathrm{R}-\frac{\mathrm{Cov}_\mathrm{XR}(\epsilon)^2}{\mathrm{Var}_\mathrm{X}(\epsilon)}.
\end{equation}
Now, let $g(\epsilon)$ be defined as follows:
\begin{equation}
    g(\epsilon)
    =-2\frac{d\mathrm{Cov}_\mathrm{XR}}{d\epsilon}(\epsilon)\,\mathrm{Var}_X(\epsilon)
        +\mathrm{Cov}_\mathrm{XR}(\epsilon)\,\frac{d\mathrm{Var}_X}{d\epsilon}(\epsilon).
    \label{eq:def_g}
\end{equation}
Then, we have
\begin{equation}
    \frac{dE}{d\epsilon}(\mathcal{X}_{i,j,\epsilon})
    =\frac{\mathrm{Cov}_\mathrm{XR}(\epsilon)}{\mathrm{Var}_X(\epsilon)^2}\,g(\epsilon).
\end{equation}
Since $\mathrm{Cov}_\mathrm{XR}(\epsilon) > 0$ and $\mathrm{Var}_\mathrm{X}(\epsilon) > 0$,
\begin{equation}
\label{eq:dE_d_epsilon}
    \mathrm{sign}\left(\frac{dE}{d\epsilon}(\mathcal{X}_{i,j,\epsilon})\right)=\mathrm{sign}(g(\epsilon)).
\end{equation}

First, we calculate the first and second derivatives of each element with respect to $\epsilon$.
Since $j-i+1$ keys move at the same speed, we have
\begin{align}
    \frac{d\bar{x}}{d\epsilon}(\epsilon)&=\frac{m}{n},\\
    \frac{d\mathrm{Var}_\mathrm{X}}{d\epsilon}(\epsilon)
        &=\frac{2m}{n}\left(\bar{x}_{i:j}+\epsilon-\bar{x}(\epsilon)\right)=:\frac{2m}{n}\Delta_\mathrm{X},\\
    \frac{d^2\mathrm{Var}_\mathrm{X}}{d\epsilon^2}(\epsilon)
        &=\frac{2m}{n}\left(1-\frac{m}{n}\right), \label{eq:d2Var}\\
    \frac{d\mathrm{Cov}_\mathrm{XR}}{d\epsilon}(\epsilon)
        &=\frac{m}{n}\left(\bar{r}_{i:j}-\bar{r}\right)=:\frac{m}{n}\Delta_\mathrm{R},\\
    \frac{d^2\mathrm{Cov}_\mathrm{XR}}{d\epsilon^2}(\epsilon)&=0. \label{eq:d2Cov}
\end{align}

From \cref{eq:def_g,eq:d2Var,eq:d2Cov}, we have
\begin{align}
    \frac{dg}{d\epsilon}(\epsilon)
        &= -\frac{d\mathrm{Cov}_\mathrm{XR}}{d\epsilon}(\epsilon)\frac{d\mathrm{Var}_\mathrm{X}}{d\epsilon}(\epsilon)
            +\mathrm{Cov}_\mathrm{XR}(\epsilon)\frac{d^2\mathrm{Var}_\mathrm{X}}{d\epsilon^2}(\epsilon) \\
        &= -\left(\tfrac{m}{n}\Delta_\mathrm{R}\right)
                \left(\tfrac{2m}{n}\Delta_\mathrm{X}\right)
            +\mathrm{Cov}_\mathrm{XR}(\epsilon)\left(\tfrac{2m}{n}\left(1-\tfrac{m}{n}\right)\right) \\
        &=\frac{2m}{n^2}\left(
                -m\Delta_\mathrm{R}\Delta_\mathrm{X}
                +\left(n-m\right)\mathrm{Cov}_\mathrm{XR}(\epsilon)
            \right).
    \label{eq:dgdeps}
\end{align}

Now, let $\bm{x}' \in \mathbb{R}^{n-m}$ and $\bm{r}' \in \mathbb{N}^{n-m}$ be the vectors obtained by removing the $i$th to $j$th elements from $\bm{x}$ and $\bm{r}$, respectively, and let $\bm{x}'' \in \mathbb{R}^{m}$ and $\bm{r}'' \in \mathbb{N}^{m}$ be the vectors obtained by extracting the $i$th to $j$th elements from $\bm{x}$ and $\bm{r}$, respectively.
That is,
\begin{align}
    \bm{x}'&=(x_1,\dots,x_{i-1},x_{j+1},\dots,x_n), \quad \bm{x}''=(x_i + \epsilon,\dots,x_j + \epsilon),\\
    \bm{r}'&=(1,\dots,i-1,j+1,\dots,n), \quad \bm{r}''=(i,\dots,j).
\end{align}
Since $\bm{x}'$, $\bm{r}'$, $\bm{x}''$, $\bm{r}''$ are all strictly monotonically increasing, we have
\begin{equation}
\label{eq:cov_positive}
    \mathrm{Cov}_\mathrm{X'R'}>0,\quad \mathrm{Cov}_\mathrm{X''R''}>0.
\end{equation}
Now, we can calculate $\mathrm{Cov}_\mathrm{X'R'}$ as follows:
\begin{align}
\label{eq:cov_expand}
    \mathrm{Cov}_\mathrm{X'R'} &= \frac{n}{(n-m)^2} \left(
        -m\Delta_\mathrm{R}\Delta_\mathrm{X}
        +\left(n-m\right)\mathrm{Cov}_\mathrm{XR}(\epsilon)
    \right) \notag \\
    &\qquad - \frac{m}{n-m} \mathrm{Cov}_\mathrm{X''R''}.
\end{align}

From \cref{eq:cov_positive,eq:cov_expand,eq:dgdeps}, we have
\begin{equation}
    \frac{dg}{d\epsilon}(\epsilon) = \frac{2m}{n^2} \frac{(n-m)^2}{n} \left( \mathrm{Cov}_\mathrm{X'R'} + \frac{m}{n-m} \mathrm{Cov}_\mathrm{X''R''} \right) > 0.
\end{equation}
Therefore, $g(\epsilon)$ is strictly monotonically increasing with respect to $\epsilon$.
From \cref{eq:dE_d_epsilon}, $\mathrm{sign}(\frac{dE(\epsilon)}{d\epsilon})$ is non-decreasing with respect to $\epsilon$, and there is at most one $\epsilon$ such that $\frac{dE(\epsilon)}{d\epsilon}=0$.
\end{proof}

Based on this lemma, we now prove \cref{thm:structure_of_optimal_multi_point_attack}:
\begin{proof}[Proof of \cref{thm:structure_of_optimal_multi_point_attack}]
The case where $\mathcal{P}^\ast = \varnothing$ trivially satisfies the proposition, so we focus on the case where $\mathcal{P}^\ast \neq \varnothing$.
To facilitate comprehension of the proof structure, an illustrative figure is presented in \cref{fig:multi_point_theorem_vis}.
Assume, for the sake of contradiction, that there exists an optimal multi-point attack $\mathcal{P}^\ast \neq \varnothing$ such that:
\begin{equation}
    \exists p \in \mathcal{P}^\ast, ~ \forall k \in \mathcal{K}, ~ \{\min(p,k)+1, \dots, \max(p,k)-1\} \not\subset \mathcal{P}^\ast.
\end{equation}

We define a set of consecutive integers ${l, l + 1, \dots, r}$ as an \textit{Isolated Poison Block} (ISB) if it satisfies:
\begin{equation}
    \{l, l+1, \dots, r\} \subset \mathcal{P}^\ast ~~ \land ~~ l-1 \notin \mathcal{K} \cup \mathcal{P}^\ast ~~ \land ~~ r+1 \notin \mathcal{K} \cup \mathcal{P}^\ast.
\end{equation}
Under the contradiction assumption, such an ISB must exist; otherwise, every poison key would be (directly or transitively) connected to a legitimate key, violating the assumption.

Let $\mathcal{P}_\mathrm{ISB} = \{l, l + 1, \dots, r\}$ denote one such ISB in $\mathcal{P}^\ast$.
For $\epsilon \in [-1, 1]$, define $\mathcal{P}(\epsilon)$ by replacing $\mathcal{P}_\mathrm{ISB}$ in $\mathcal{P}^\ast$ with its $\epsilon$-shifted block:
\begin{equation}
    \mathcal{P}(\epsilon) \coloneq \left(\mathcal{P}^\ast \setminus \{l, l+1, \dots, r\} \right) \cup \{l + \epsilon, l+1 + \epsilon, \dots, r + \epsilon\}.
\end{equation}
In particular, we define $\mathcal{P}_{-}$ and $\mathcal{P}_{+}$ as follows:
\begin{align}
    \mathcal{P}_{-} &\coloneq \mathcal{P}(-1) = \left(\mathcal{P}^\ast \setminus \{r\}\right) \cup \{l-1\},\\
    \mathcal{P}_{+} &\coloneq \mathcal{P}(1) = \left(\mathcal{P}^\ast \setminus \{l\}\right) \cup \{r+1\}.
\end{align}
Note that by the definition of ISB, we have $l - 1 \notin \mathcal{K} \cup \mathcal{P}^\ast$ and $r + 1 \notin \mathcal{K} \cup \mathcal{P}^\ast$.
Since $\mathcal{P}^\ast$ is assumed to be an optimal multi-point attack, we must have:
\begin{equation}
\label{eq:multi_point_from_assumption}
    E(\mathcal{K} \cup \mathcal{P}_{-}) \leq E(\mathcal{K} \cup \mathcal{P}^\ast) ~~ \land ~~ E(\mathcal{K} \cup \mathcal{P}^\ast) \geq E(\mathcal{K} \cup \mathcal{P}_{+}).
\end{equation}

On the other hand, this leads to a contradiction with \cref{lem:multi_point_poisoning_attack}.
According to \cref{lem:multi_point_poisoning_attack}, over the interval $\epsilon \in [-1, 1]$, the function $\mathrm{sign}\left(\frac{d E(\mathcal{K} \cup \mathcal{P}(\epsilon))}{d\epsilon}\right)$ is non-decreasing, and there exists at most one point $\epsilon$ where $\frac{d E(\mathcal{K} \cup \mathcal{P}(\epsilon))}{d\epsilon} = 0$.
Therefore, the behavior of $E(\mathcal{K} \cup \mathcal{P}(\epsilon))$ over $\epsilon \in [-1, 1]$ falls into one of three cases:
(i) strictly increasing,
(ii) strictly decreasing, or
(iii) strictly decreasing over $[-1, \epsilon')$ and strictly increasing over $(\epsilon', 1]$ for some $\epsilon' \in [-1, 1]$.
Noting that $\mathcal{P}^\ast = \mathcal{P}(0)$, $\mathcal{P}_{-} = \mathcal{P}(-1)$, and $\mathcal{P}_{+} = \mathcal{P}(1)$, it follows that in all cases:
\begin{equation}
    \label{eq:multi_point_from_lemma}
    E(\mathcal{K} \cup \mathcal{P}_{-}) > E(\mathcal{K} \cup \mathcal{P}^\ast) ~~ \lor ~~ E(\mathcal{K} \cup \mathcal{P}^\ast) < E(\mathcal{K} \cup \mathcal{P}_{+}).
\end{equation}

Since \cref{eq:multi_point_from_assumption} and \cref{eq:multi_point_from_lemma} are contradictory, we conclude that the assumption is false.
\end{proof}

\section{Proof of the Upper Bound}
\label{app:proof_of_upper_bound}

Here, we prove \cref{thm:d_saturation,thm:convexity_of_min_b_l_fixed_d,thm:optimal_c_for_fixed_w}.
First, we introduce the notation.
Let $c_i$ be the number of $k_i$ in the key (multi-set) set $\mathcal{K} \uplus \mathcal{Q}_\mathcal{K}(\bm{d})$ after poisoning (i.e., $c_i = d_i + 1$).
Let $\lambda' \coloneq \sum_{i=1}^{n} d_i$.
Then, $\sum_{i=1}^{n} c_i = n + \lambda' \leq n + \lambda$ holds.
Let $\bm{k}' \in \mathbb{N}^{n+\lambda'}$ be the vector obtained by sorting the elements of $\mathcal{K} \uplus \mathcal{Q}_\mathcal{K}(\bm{d})$ in ascending order.
Also, for convenience, let $\bm{r}' \in \mathbb{N}^{n+\lambda'}$ be the vector such that $r'_i = i ~ (i\in\{1,2,\dots,n+\lambda'\})$.

First, we prove the following lemma:
\begin{theorem_box}
\begin{lemma}
\label{lem:d_saturation_lemma}
Let $E(\mathcal{K} \uplus \mathcal{Q}_\mathcal{K}(\bm{d}))$ denote the objective function defined in the Relaxed Poisoning Problem (\cref{def:relaxed_poisoning_problem}). For any $\bm{d} \in \mathbb{Z}_{\ge 0}^n$, there exists some $i \in \{1,2,\dots,n\}$ such that
\begin{equation}
    E(\mathcal{K} \uplus \mathcal{Q}_\mathcal{K}(\bm{d} + \bm{e}_i)) > E(\mathcal{K} \uplus \mathcal{Q}_\mathcal{K}(\bm{d})).
\end{equation}
\end{lemma}
\end{theorem_box}
\begin{proof}[Proof of \cref{lem:d_saturation_lemma}]
Let $N \coloneq n + \lambda'$.
Let $V_{\mathrm{K}}$ be the variance of $\bm{k}'$, $V_{\mathrm{R}}$ be the variance of $\bm{r}'$, and $C$ be the covariance between $\bm{k}'$ and $\bm{r}'$.
\begin{align}
    V_{\mathrm{K}} &\coloneq \frac{1}{N} \sum_{i=1}^{N} (k'_i - \bar{k'})^2, \\
    V_{\mathrm{R}} &\coloneq \frac{1}{N} \sum_{i=1}^{N} (r'_i - \bar{r'})^2 = \frac{N^2 - 1}{12}, \\
    C &\coloneq \frac{1}{N} \sum_{i=1}^{N} (k'_i - \bar{k'})(r'_i - \bar{r'}),
\end{align}
where
\begin{equation}
    \bar{k'} \coloneq \frac{1}{N} \sum_{i=1}^{N} k'_i = \frac{1}{N} \sum_{i=1}^{n} c_i k_i, \quad \bar{r'} \coloneq \frac{1}{N} \sum_{i=1}^{N} r'_i = \frac{N + 1}{2}.
\end{equation}
Similarly, let $V_{\mathrm{K}i}$ be the variance of $\bm{k}' + \bm{e}_i$, $V_{\mathrm{R}i}$ be the variance of $\bm{r}' + \bm{e}_i$, and $C_{i}$ be the covariance between $\bm{k}' + \bm{e}_i$ and $\bm{r}' + \bm{e}_i$, where $\bm{e}_i$ is the $i$th standard basis vector.
Note that when $d_i$ is increased by 1, both $\bm{k}'$ and $\bm{r}'$ change.

We prove the following inequality:
\begin{equation}
    \label{eq:delta_v_ki_c_i_positive}
    \sum_{i=1}^{n} \Delta_i \cdot V_{\mathrm{K}i} c_i > 0,
\end{equation}
where
\begin{equation}
    \Delta_i \coloneq E(\mathcal{K} \uplus \mathcal{Q}_\mathcal{K}(\bm{d} + \bm{e}_i)) - E(\mathcal{K} \uplus \mathcal{Q}_\mathcal{K}(\bm{d})).
\end{equation}
Since $c_i > 0$ and $V_{\mathrm{K}i} > 0$, the inequality \cref{eq:delta_v_ki_c_i_positive} implies that there exists some $i \in \{1,2,\dots,n\}$ such that $E(\mathcal{K} \uplus \mathcal{Q}_\mathcal{K}(\bm{d} + \bm{e}_i)) > E(\mathcal{K} \uplus \mathcal{Q}_\mathcal{K}(\bm{d}))$.

Since
\begin{align}
    \Delta_i &= \left(V_{\mathrm{R}i} - \frac{{C_i}^2}{V_{\mathrm{K}i}}\right) - \left(V_{\mathrm{R}} - \frac{C^2}{V_{\mathrm{K}}}\right) \\
&= \frac{2N+1}{12} - \frac{C_i^2}{V_{\mathrm{K}i}} + \frac{C^2}{V_{\mathrm{K}}},
\end{align}
we have
\begin{align}
\label{eq:delta_i_v_ki}
    \Delta_i \cdot V_{\mathrm{K}i} = \frac{2N+1}{12} V_{\mathrm{K}i} - C_i^2 + C^2 \frac{V_{\mathrm{K}i}}{V_{\mathrm{K}}}.
\end{align}

Now, we calculate $V_{\mathrm{K}i}$ and $C_{i}$.
\begin{align}
\label{eq:v_ki}
V_{\mathrm{K}i} &= \frac{N}{N+1} V_{\mathrm{K}} + \frac{N}{(N+1)^2} (k_i - \bar{k'})^2,\\
\label{eq:c_ki}
C_{i} &= \frac{N}{N+1} C + \frac{1}{2(N+1)} S_i,
\end{align}
where
\begin{equation}
\label{eq:s_i}
    S_i \coloneq \sum_{j=1}^{n} c_j |k_i - k_j|.
\end{equation}
From \cref{eq:delta_i_v_ki,eq:v_ki,eq:c_ki}, we have
\begin{align}
    \Delta_i \cdot V_{\mathrm{K}i}
&= \frac{2N+1}{12} V_{\mathrm{K}i} 
- \left( \frac{N}{N+1} C + \frac{1}{2(N+1)} S_i \right)^2 \notag \\
&\qquad + C^2 \left( \frac{N}{N+1} + \frac{N (k_i - \bar{k'})^2}{(N+1)^2 V_{\mathrm{K}}} \right) \\
\label{eq:delta_cov_kr_cov_kr_var_k_expression}
&= \frac{2N+1}{12} V_{\mathrm{K}i} + \frac{NC^2}{(N+1)^2} - \frac{NC}{(N+1)^2} S_i \notag \\
&\qquad - \frac{1}{4(N+1)^2} S_i^2 + \frac{N C^2}{(N+1)^2 V_{\mathrm{K}}} (k_i - \bar{k'})^2.
\end{align}

Now, we calculate (or upper bound) the following four expressions: (i) $\sum_{i=1}^{n} c_i V_{\mathrm{K}i}$, (ii) $\sum_{i=1}^{n} c_i S_i$, (iii) $\sum_{i=1}^{n} c_i S_i^2$, (iv) $\sum_{i=1}^{n} c_i (k_i - \bar{k'})^2$.

\noindent \textbf{(i) $\sum_{i=1}^{n} c_i V_{\mathrm{K}i}$.}
\begin{align}
    \sum_{i=1}^{n} c_i V_{\mathrm{K}i}
&= \sum_{i=1}^{n} c_i \left( \frac{N}{N+1} V_{\mathrm{K}} + \frac{N}{(N+1)^2} (k_i - \bar{k'})^2 \right) \\
&= \frac{N V_{\mathrm{K}}}{N+1} \sum_{i=1}^{n} c_i + \frac{N}{(N+1)^2} \sum_{i=1}^{n} c_i (k_i - \bar{k'})^2 \\
&= \frac{N^2 V_{\mathrm{K}}}{N+1} + \frac{N^2 V_{\mathrm{K}}}{(N+1)^2} \\
&= \frac{N^2 (N + 2)}{(N+1)^2} V_{\mathrm{K}}.
\end{align}

\noindent \textbf{(ii) $\sum_{i=1}^{n} c_i S_i$.}
\begin{align}
    \sum_{i=1}^{n} c_i S_i
&= \sum_{i=1}^{n} \sum_{j=1}^{n} c_i c_j |k_i - k_j| \\
&= 4 N C.
\end{align}

\noindent \textbf{(iii) $\sum_{i=1}^{n} c_i S_i^2$.}
\begin{align}
    \sum_{i=1}^{n} c_i S_i^2
&= \sum_{i=1}^{n} c_i \left( \sum_{j=1}^{n} c_j |k_i - k_j| \right)^2 \\
&= \sum_{i=1}^{n} c_i \left( \sum_{j=1}^{i} c_j (k_i - k_j) + \sum_{j=i+1}^{n} c_j (k_j - k_i) \right)^2 \\
&= \sum_{i=1}^{n} c_i \left( T_i + U_i \right)^2 \\
&= \sum_{i=1}^{n} c_i \left( T_i - U_i \right)^2 + 4 \sum_{i=1}^{n} c_i T_i U_i,
\end{align}
where
\begin{equation}
    T_i \coloneq \sum_{j=1}^{i} c_j (k_i - k_j), \quad U_i \coloneq \sum_{j=i+1}^{n} c_j (k_j - k_i)
\end{equation}

We calculate (or upper bound) $\sum_{i=1}^{n} c_i \left( T_i - U_i \right)^2$ and $\sum_{i=1}^{n} c_i T_i U_i$.

First, we can rewrite $\sum_{i=1}^{n} c_i \left( T_i - U_i \right)^2$ as follows:
\begin{align}
    \sum_{i=1}^{n} c_i \left( T_i - U_i \right)^2
&= \sum_{i=1}^{n} c_i \left( \sum_{j=1}^{i} c_j (k_i - k_j) - \sum_{j=i+1}^{n} c_j (k_j - k_i) \right)^2 \\
&= \sum_{i=1}^{n} c_i \left( \sum_{j=1}^{n} c_j (k_i - k_j) \right)^2 \\
&= \sum_{i=1}^{n} c_i \left( N k_i - N \bar{k'} \right)^2 \\
&= N^2 \sum_{i=1}^{n} c_i (k_i - \bar{k'})^2 \\
&= N^3 V_{\mathrm{K}}.
\end{align}

Next, we can upper bound $\sum_{i=1}^{n} c_i T_i U_i$ as follows:
\begin{align}
    &~ \sum_{i=1}^{n} c_i T_i U_i \\
=&~ \sum_{i=1}^{n} c_i \left( \sum_{j=1}^{i} c_j (k_i - k_j) \right) \left( \sum_{j=i+1}^{n} c_j (k_j - k_i) \right) \\
=&~ \sum_{i=1}^{n} \sum_{j=1}^{i-1} \sum_{l=i+1}^{n} c_i c_j c_l (k_i - k_j) (k_l - k_i) \\
\leq&~ \frac{1}{4} \sum_{i=1}^{n} \sum_{j=1}^{i-1} \sum_{l=i+1}^{n} c_i c_j c_l (k_l - k_j)^2 \\
=&~ \frac{1}{24} \sum_{i=1}^{n} 
\sum_{\substack{1\leq j \leq n,\\j \neq i}}
\sum_{\substack{1\leq l \leq n,\\l \neq i,j}}
c_i c_j c_l (k_{\max(i,j,l)} - k_{\min(i,j,l)})^2 \\
\leq&~ \frac{1}{24} \sum_{i=1}^{n} \sum_{j=1}^{n} \sum_{l=1}^{n} c_i c_j c_l (k_{\max(i,j,l)} - k_{\min(i,j,l)})^2 \\
\leq&~ \frac{1}{36} \sum_{i=1}^{n} \sum_{j=1}^{n} \sum_{l=1}^{n} c_i c_j c_l \left[ (k_i - k_j)^2 + (k_j - k_l)^2 + (k_l - k_i)^2 \right] \\
=&~ \frac{1}{12} N \sum_{i=1}^{n} \sum_{j=1}^{n} c_i c_j (k_i - k_j)^2 \\
=&~ \frac{1}{6} N^3 V_{\mathrm{K}}.
\end{align}

Therefore,
\begin{equation}
    \sum_{i=1}^{n} c_i S_i^2 \leq N^3 V_{\mathrm{K}} + \frac{4}{6} N^3 V_{\mathrm{K}} = \frac{5}{3} N^3 V_{\mathrm{K}}.
\end{equation}

\noindent \textbf{(iv) $\sum_{i=1}^{n} c_i (k_i - \bar{k'})^2$.}
By definition,
\begin{align}
    \sum_{i=1}^{n} c_i (k_i - \bar{k'})^2 = N V_{\mathrm{K}}.
\end{align}

From (i), (ii), (iii), (iv), and \cref{eq:delta_cov_kr_cov_kr_var_k_expression}, we have
\begin{align}
    \sum_{i=1}^{n} \Delta_i \cdot V_{\mathrm{K}i} c_i \geq&~ \frac{2N+1}{12} \frac{N^2 (N+2) V_{\mathrm{K}}}{(N+1)^2} + \frac{N^2 C^2}{(N+1)^2} - \frac{NC \cdot 4NC}{(N+1)^2} \notag \\ &\qquad - \frac{5 N^3 V_{\mathrm{K}}}{4(N+1)^2 \cdot 3} + \frac{N C^2 \cdot N V_{\mathrm{K}}}{(N+1)^2 V_{\mathrm{K}}} \\
=&~ \frac{N^2 (N^2 + 1)}{6(N+1)^2} V_{\mathrm{K}} - \frac{2 N^2 C^2}{(N+1)^2} \\
\geq&~ \frac{N^2 (N^2 + 1)}{6(N+1)^2} V_{\mathrm{K}} - \frac{2 N^2}{(N+1)^2} \frac{N^2 - 1}{12} V_{\mathrm{K}} \notag \\ &\qquad\qquad \left(\because ~ C^2 \leq V_{\mathrm{R}} V_{\mathrm{K}}  = \frac{N^2 - 1}{12} V_{\mathrm{K}} \right) \\
=&~ \frac{N^2}{3 (N+1)^2} V_{\mathrm{K}} > 0.
\end{align}
Therefore, there exists some $i \in \{1,2,\dots,n\}$ such that $E(\mathcal{K} \uplus \mathcal{Q}_\mathcal{K}(\bm{d} + \bm{e}_i)) > E(\mathcal{K} \uplus \mathcal{Q}_\mathcal{K}(\bm{d}))$.
\end{proof}

Using \cref{lem:d_saturation_lemma}, we can immediately prove \cref{thm:d_saturation}.
\begin{proof}[Proof of \cref{thm:d_saturation}]
\Cref{lem:d_saturation_lemma} implies that, for any $\bm{d}$, there exists some $i$ such that adding $k_i$ as a new poison strictly increases the loss.
Therefore, it immediately follows that the optimal solution $\bm{d}^\ast$ of the Relaxed Poisoning Problem must satisfy $\sum_{i=1}^{n} d^\ast_i = \lambda$.
\end{proof}

\begin{proof}[Proof of \cref{thm:convexity_of_min_b_l_fixed_d}]
After poisoning, the MSE for the key multi-set $\mathcal{K} \uplus \mathcal{Q}_\mathcal{K}(\bm{d})$ can be expressed as follows:
\begin{equation}
    \mathcal{L}(\mathcal{K} \uplus \mathcal{Q}_\mathcal{K}(\bm{d}); w, b) = \frac{1}{n + \lambda} \sum_{i=1}^{n+\lambda} (w k'_i + b - r'_i)^2.
\end{equation}
By solving $\frac{\partial \mathcal{L}(\mathcal{K} \uplus \mathcal{Q}_\mathcal{K}(\bm{d}); w, b)}{\partial b} = 0$, we have
\begin{equation}
    b^\ast = \bar{r'} - w \bar{k'},
\end{equation}
where
\begin{align}
    \bar{r'} &= \frac{1}{n+\lambda} \sum_{i=1}^{n+\lambda} r'_i = \frac{n + \lambda + 1}{2}, \\
    \bar{k'} &= \frac{1}{n+\lambda} \sum_{i=1}^{n+\lambda} k'_i = \frac{1}{n + \lambda} \sum_{i=1}^{n} c_i k_i.
\end{align}
Therefore,
\begin{align}
&~\min_{b} \mathcal{L}(\mathcal{K} \uplus \mathcal{Q}_\mathcal{K}(\bm{d}); w, b)\\
=&~ \frac{1}{n+\lambda} \sum_{i=1}^{n+\lambda} (w k'_i + b^\ast - r'_i)^2 \\
=&~ \frac{1}{n+\lambda} \sum_{i=1}^{n+\lambda} \left( (k'_i - \bar{k'})w - (r'_i - \bar{r'}) \right)^2 \label{eq:mse_expression} \\
=&~ w^2 \mathrm{Var}_\mathrm{K'} - 2 w \mathrm{Cov}_\mathrm{K'R'} + \mathrm{Var}_\mathrm{R'}, \label{eq:mse_expression_2}
\end{align}
where $\mathrm{Var}_\mathrm{K'}$ and $\mathrm{Var}_\mathrm{R'}$ denote the variances of the keys and the ranks after poisoning, respectively, and $\mathrm{Cov}_\mathrm{K'R'}$ denotes the covariance between the keys and the ranks after poisoning.

Therefore, when $\bm{d}$ is fixed, $\min_{b} \mathcal{L}(\mathcal{K} \uplus \mathcal{Q}_\mathcal{K}(\bm{d}); w, b)$ is a quadratic function of $w$.
Since the coefficient of $w^2$ is $\mathrm{Var}_\mathrm{K'} > 0$, $\min_{b} \mathcal{L}(\mathcal{K} \uplus \mathcal{Q}_\mathcal{K}(\bm{d}); w, b)$ is a convex function of $w$.
\end{proof}

\begin{proof}[Proof of \cref{thm:optimal_c_for_fixed_w}]
We first derive the change in loss due to the movement of poison, and then use it to prove \cref{thm:optimal_c_for_fixed_w} by contradiction.

\noindent \textbf{1. Change in loss due to the movement of poison.}
From the conclusion of \cref{thm:d_saturation}, we can assume that $\sum_{i=1}^{n} c_i = n + \lambda$.
Here, we derive the change in loss due to the movement of poison, specifically, the change in Residual Sum of Squares (RSS) when one poison is moved from $k_i$ to $k_j$.
Let $\mathrm{RSS} (\bm{c}) \coloneq (n + \lambda) \min_{b} \mathcal{L}(\mathcal{K} \uplus \mathcal{Q}_\mathcal{K}(\bm{d}); w, b)$ and define $\Delta \mathrm{RSS}_{i \rightarrow j}$ as follows:
\begin{equation}
    \Delta \mathrm{RSS}_{i \rightarrow j} (\bm{c}) = \mathrm{RSS}(\bm{c} - \bm{e}_i + \bm{e}_j) - \mathrm{RSS}(\bm{c}),
\end{equation}
where $i,j$ are natural numbers between 1 and $n$, and $c_i \geq 2$.

Using \cref{eq:mse_expression}, we can calculate and simplify $\Delta \mathrm{RSS}_{i \rightarrow j} (\bm{c})$ as follows:
\begin{align}
\label{eq:delta_rss_ij}
    \Delta \mathrm{RSS}_{i \rightarrow j} (\bm{c}) &= w^2 \left[
        k_j^2 - k_i^2 - \frac{2 Q_{n+1}}{n + \lambda}(k_j - k_i) - \frac{(k_j - k_i)^2}{n + \lambda}
    \right] \notag \\
    &\qquad -2w \Bigg[
        (k_j P_j - Q_j) - (k_i P_i - Q_i) + k_j \notag \\
    &\qquad\qquad - k_{\max(i,j)} - \frac{n + \lambda + 1}{2}(k_j - k_i)
    \Bigg],
\end{align}
where
\begin{equation}
    P_t = \sum_{l=1}^{t-1} c_l, ~~ Q_t = \sum_{l=1}^{t-1} c_l k_l, ~~ T = \sum_{l=1}^{n} c_l k_l.
\end{equation}

\noindent \textbf{2. Proof of \cref{thm:optimal_c_for_fixed_w} by contradiction.}
Using the conclusion of the previous step, we prove \cref{thm:optimal_c_for_fixed_w}.

\noindent \textbf{2.1. Case $w (k_n - k_1) < n + \lambda$.}
We prove that the optimal poison set must consist of only one key by contradiction.
Assume that there exists some $\bm{c}^\ast$ that maximizes $\mathrm{RSS}(\bm{c})$ and satisfies $c^\ast_i \geq 2, c^\ast_j \geq 2$ for $1 \leq i < j \leq n$.
From \cref{eq:delta_rss_ij}, we have
\begin{equation}
\label{eq:delta_rss_ij_ji}
    \Delta \mathrm{RSS}_{i \rightarrow j} (\bm{c}^\ast) + \Delta \mathrm{RSS}_{j \rightarrow i} (\bm{c}^\ast) = 2w (k_j - k_i) \left[1 - \frac{w(k_j - k_i)}{n + \lambda}\right].
\end{equation}
Since $\bm{k}$ is monotonically increasing and $w (k_n - k_1) < n + \lambda$, we have
\begin{equation}
\label{eq:w_ratio_ineq}
    \frac{w(k_j - k_i)}{n + \lambda} \leq \frac{w(k_n - k_1)}{n + \lambda} < 1.
\end{equation}
From \cref{eq:delta_rss_ij_ji,eq:w_ratio_ineq}, we have
\begin{equation}
    \Delta \mathrm{RSS}_{i \rightarrow j} (\bm{c}^\ast) + \Delta \mathrm{RSS}_{j \rightarrow i} (\bm{c}^\ast) > 0.
\end{equation}
Therefore, $\mathrm{RSS}(\bm{c}^\ast - \bm{e}_i + \bm{e}_j) > \mathrm{RSS}(\bm{c}^\ast)$ or $\mathrm{RSS}(\bm{c}^\ast - \bm{e}_j + \bm{e}_i) > \mathrm{RSS}(\bm{c}^\ast)$.
This contradicts the assumption that $\bm{c}^\ast$ maximizes $\mathrm{RSS}(\bm{c})$.
Therefore, the optimal poison set must consist of only one key when $w (k_n - k_1) < n + \lambda$.

\noindent \textbf{2.2. Case $w (k_n - k_1) \geq n + \lambda$.}
We prove that the optimal poison set must consist of $k_1$ and $k_n$ by contradiction.
Assume that there exists some $\bm{c}^\ast$ that maximizes $\mathrm{RSS}(\bm{c})$ and satisfies $c^\ast_i \geq 2$ for $2 \leq i \leq n - 1$.
From \cref{eq:delta_rss_ij}, we have
\begin{align}
\label{eq:delta_rss_i1_in}
    &~ \frac{\Delta \mathrm{RSS}_{i \rightarrow 1} (\bm{c}^\ast)}{w (k_i - k_1)} + \frac{\Delta \mathrm{RSS}_{i \rightarrow n} (\bm{c}^\ast)}{w (k_n - k_i)} \\
    =&~  2 + \left(\frac{R_i}{k_i - k_1} - \frac{R_n - R_i}{k_n - k_i}\right) + w (k_n - k_1) \left(1 - \frac{1}{n + \lambda}\right),
\end{align}
where
\begin{equation}
\label{eq:Rt_define}
    R_t = \sum_{l=1}^{t-1} (k_t - k_l) c_l.
\end{equation}

Now, we have
\begin{equation}
\label{eq:term2}
    \frac{R_i}{k_i - k_1} - \frac{R_n - R_i}{k_n - k_i} \geq - (n + \lambda - 1)
\end{equation}
because
\begin{align}
    &~ \frac{R_i}{k_i - k_1} - \frac{R_n - R_i}{k_n - k_i} \\
=&~ \sum_{l=1}^{i-1} \frac{k_i - k_l}{k_i - k_1} c_l - \left( \sum_{l=1}^{i-1} c_l + \sum_{l=i}^{n-1} \frac{k_n - k_l}{k_n - k_i} c_l \right) \\
=&~ - \sum_{l=1}^{i-1} \frac{k_l - k_1}{k_i - k_1} c_l - \sum_{l=i}^{n-1} \frac{k_n - k_l}{k_n - k_i} c_l \\
\geq&~ - \sum_{l=1}^{i-1} c_l - \sum_{l=i}^{n-1} c_l \\
=&~ - (n + \lambda - c_n) \\
\geq&~ - (n + \lambda - 1).
\end{align}
Here, we used the monotonicity of $\bm{k}$ and $c_n \geq 1$.

Since $w (k_n - k_1) \geq n + \lambda$, we have
\begin{equation}
\label{eq:term3}
    w (k_n - k_1) \left(1 - \frac{1}{n + \lambda}\right) \geq n + \lambda - 1.
\end{equation}

From \cref{eq:delta_rss_i1_in,eq:term2,eq:term3}, we have
\begin{equation}
    \frac{\Delta \mathrm{RSS}_{i \rightarrow 1} (\bm{c}^\ast)}{w (k_i - k_1)} + \frac{\Delta \mathrm{RSS}_{i \rightarrow n} (\bm{c}^\ast)}{w (k_n - k_i)} \geq 2 > 0.
\end{equation}
Therefore, $\mathrm{RSS}(\bm{c}^\ast - \bm{e}_i + \bm{e}_1) > \mathrm{RSS}(\bm{c}^\ast)$ or $\mathrm{RSS}(\bm{c}^\ast - \bm{e}_i + \bm{e}_n) > \mathrm{RSS}(\bm{c}^\ast)$.
This contradicts the assumption that $\bm{c}^\ast$ maximizes $\mathrm{RSS}(\bm{c})$.
Therefore, the optimal poison set must consist of $k_1$ and $k_n$.
\end{proof}

\section{Efficient Calculation Method for the Upper Bound}
\label{sec:efficient_calculation_for_upper_bound}

We efficiently evaluate \cref{eq:upper_bound_relaxed_poisoning_problem} by reducing it to determine the value of $\min_w \max_i f_i(w)$, i.e., the minimum of a function defined as the pointwise maximum of given convex quadratic functions.
In this section, we first describe the efficient algorithm to evaluate the coefficients of the quadratic functions.
Then, we detail three algorithms for solving this problem: (1) a method using golden section search over $w$, (2) a method using binary search over $y$ (function values), and (3) a method using explicit calculation of $\max_{i} f_i(w)$.

\subsection{Efficient Computation of Coefficients}
\label{sec:efficient_computation_of_coefficients}

\Cref{thm:optimal_c_for_fixed_w} implies that the vector $\bm{d}$ we need to consider can be represented as follows:
\begin{equation}
    \bm{d} = a \bm{e}_1 + b \bm{e}_i + (\lambda - a - b) \bm{e}_n,
\end{equation}
where $a,b \in \mathbb{Z}$ and $i \in \{2,3,\dots,n-1\}$.
Specifically, the cases where $a=\lambda$, $b=\lambda$, or $a + b = 0$ correspond to the first type of configuration (i.e., $\exists i \in {1,2,\dots,n}$ such that $d_i = \lambda$), while the case $b=0$ corresponds to the second type (i.e., $d_1 + d_n = \lambda$).

We show that with $\mathcal{O}(n)$ preprocessing, the quantities $\mathrm{Var}_\mathrm{K'}$, $\mathrm{Var}_\mathrm{R'}$, and $\mathrm{Cov}_\mathrm{K'R'}$ for a given $(a,b,i)$ can each be computed in $\mathcal{O}(1)$ time.
From \cref{eq:mse_expression_2}, once these three quantities are obtained, the coefficients of the quadratic function in $w$ can be determined.

\textbf{Precomputation}
We first precompute three arrays $S_t, T_t,$ and $U_t$ for $t = 0, 1, 2, \dots, n$:
\begin{equation}
    S_t \coloneqq \sum_{l=1}^{t} k_l, ~~ 
    T_t \coloneqq \sum_{l=1}^{t} k_l^2, ~~
    U_t \coloneqq \sum_{l=1}^{t} k_l l.
\end{equation}
This preprocessing can be performed in $\mathcal{O}(n)$ time.

\textbf{Computation of Coefficients}
Let $N = n + \lambda$, $c = \lambda - a - b$, and define the vector $\bm{x}(a,b,i)$ as
\begin{align}
    \bm{x}(a,b,i) &\coloneq [\underbrace{k_1, \dots, k_1}_{a + 1},\;
        k_2, \dots, k_{i-1},\;
        \underbrace{k_i, \dots, k_i}_{b + 1},\; \nonumber\\
        &\qquad\qquad\qquad k_{i+1}, \dots, k_{n-1},\;
        \underbrace{k_n, \dots, k_n}_{c + 1}].
\end{align}
Also, define the following quantities:
\begin{align}
    M_{x}(a,b,i) &\coloneq \frac{1}{N} \sum_{j=1}^{N} x(a,b,i)_j, \\
    M_{r} &\coloneq \frac{1}{N} \sum_{j=1}^{N} j = \frac{N+1}{2}, \\
    M_{x^2}(a,b,i) &\coloneq \frac{1}{N} \sum_{j=1}^{N} x(a,b,i)_j^2, \\
    M_{xr}(a,b,i) &\coloneq \frac{1}{N} \sum_{j=1}^{N} x(a,b,i)_j \cdot j, \\
    M_{r^2} &\coloneq \frac{1}{N} \sum_{j=1}^{N} j^2 = \frac{(N+1)(2N+1)}{6}.
\end{align}
Then the desired coefficients $\mathrm{Var}_\mathrm{K'}$, $\mathrm{Var}_\mathrm{R'}$, and $\mathrm{Cov}_\mathrm{K'R'}$ can be expressed as follows:
\begin{align}
    \mathrm{Var}_\mathrm{K'} &= M_{x^2}(a,b,i) - M_{x}(a,b,i)^2, \\
    \mathrm{Var}_\mathrm{R'} &= M_{r^2} - M_{r}^2, \\
    \mathrm{Cov}_\mathrm{K'R'} &= M_{xr}(a,b,i) - M_{x}(a,b,i) M_{r}.
\end{align}
Furthermore, $M_{x}(a,b,i)$, $M_{x^2}(a,b,i)$, and $M_{xr}(a,b,i)$ can be computed as follows:
\begin{align}
    M_{x}(a,b,i) &= \frac{1}{N} \left(a k_1 + b k_i + c k_n + S_n\right), \\
    M_{x^2}(a,b,i) &= \frac{1}{N} \left(a k_1^2 + b k_i^2 + c k_n^2 + T_n\right), \\
    M_{xr}(a,b,i) &= \frac{1}{N}\Biggl[
        U_n + a S_n + b (S_n - S_{i-1})
        + \frac{a(a+1)}{2} k_1 \nonumber\\
        & + \left(b(a+i) + \frac{b(b-1)}{2}\right) k_i
        + \left(cN - \frac{c(c-1)}{2}\right) k_n
    \Biggr].
\end{align}
Since $S_t, T_t,$ and $U_t$ have already been precomputed, all these quantities can be computed in $\mathcal{O}(1)$ time.
Therefore, $\mathrm{Var}_\mathrm{K'}$, $\mathrm{Var}_\mathrm{R'}$, and $\mathrm{Cov}_\mathrm{K'R'}$ can be computed in $\mathcal{O}(1)$ time.

\subsection{Golden Section Search over $w$}
\label{sec:golden_section_search_over_w}

\begin{algorithm}[t]
    \caption{A Method with Golden Section Search over $w$}
    \label{algo:golden_section_search}
    \begin{algorithmic}[1]
    \Require Convex quadratic functions $f_1, f_2, \dots, f_q$
    \Require Search interval $[a, b]$, iterations $T$
    \Function{MinMaxQuadratics\_1}{$f_1, \dots, f_q$, $a$, $b$, $T$}
        \State $\phi \gets \frac{1 + \sqrt{5}}{2}$ \Comment{Golden ratio}
        \State $w_l \gets b - \frac{b - a}{\phi}, \quad w_r \gets a + \frac{b - a}{\phi}$
        \State $y_l \gets \max_i f_i(w_l), \quad y_r \gets \max_i f_i(w_r)$
        \For{$t \gets 1$ to $T$}
            \If{$y_l > y_r$}
                \State $a \gets w_l, \quad w_l \gets w_r, \quad w_r \gets a + \frac{b - a}{\phi}$
                \State $y_l \gets y_r, \quad y_r \gets \max_i f_i(w_r)$
            \Else
                \State $b \gets w_r, \quad w_r \gets w_l, \quad w_l \gets b - \frac{b - a}{\phi}$
                \State $y_r \gets y_l, \quad y_l \gets \max_i f_i(w_l)$
            \EndIf
        \EndFor
        \State $w^\ast \gets \frac{a + b}{2}$, $y^\ast \gets \max_i f_i(w^\ast)$
        \State \Return $y^\ast$
    \EndFunction
    \end{algorithmic}
\end{algorithm}

The function $\max_{i} f_i(w)$ is convex with respect to $w$.
This follows immediately from \cref{thm:convexity_of_min_b_l_fixed_d}, which establishes that each $f_i$ is convex, together with the well-known fact that the pointwise maximum of convex functions is also convex.

The golden section search algorithm is an efficient algorithm for finding the minimum of a unimodal function, and thus can be applied to convex functions.
In each iteration, the golden section search reduces the size of the search interval by a factor of $\frac{1}{\phi} \approx 0.618$, where $\phi$ denotes the golden ratio.
As a result, the interval guaranteed to contain the minimizer $w^\ast$ shrinks exponentially with the number of iterations.
In our setting, evaluating $\max_{i} f_i(w)$ once requires $\mathcal{O}(n + \lambda)$ time.
Therefore, the total computational cost over $T$ iterations is $\mathcal{O}(T (n + \lambda))$.

The initial search interval $[a, b]$ must be chosen to include the optimal value $w^\ast$.
We set $a$ and $b$ as the smallest and largest vertex positions, respectively, among all quadratic functions $f_i$.
This choice ensures that $[a, b]$ always contains the global minimizer $w^\ast$.

\subsection{Binary Search over $y$}
\label{sec:binary_search_over_y}

\begin{algorithm}[t]
    \caption{A Method with Binary Search over $y$}
    \label{algo:binary_search_over_y}
    \begin{algorithmic}[1]
    \Require Convex quadratic functions $f_1, f_2, \dots, f_q$
    \Require Search interval $[y_l, y_r]$, iterations $T$
    \Function{MinMaxQuadratics\_2}{$f_1, \dots, f_q$, $y_l$, $y_r$, $T$}
        \For{$t \gets 1$ to $T$}
            \State $y_m \gets \frac{y_l + y_r}{2}$
            \State $w_l \gets - \infty, ~ w_r \gets \infty, ~ \mathrm{valid} \gets \mathrm{True}$
            \For{$i \gets 1$ to $q$}
                \State $\alpha, \beta \gets \mathrm{SolveQuadEq}(f_i, y_m)$ \Comment{$w$ s.t. $f_i(w) = y_m$.}
                \If{$\alpha \mathrm{~is~None}$}
                    \State $\mathrm{valid} \gets \mathrm{False}$
                    \State \textbf{break}
                \EndIf
                \State $w_l \gets \max(w_l, \alpha), ~ w_r \gets \min(w_r, \beta)$
            \EndFor
            \If{$\mathrm{valid} = \mathrm{True} \land w_l \leq w_r$}
                \State $y_r \gets y_m$
            \Else
                \State $y_l \gets y_m$
            \EndIf
        \EndFor
        \State \Return $y_r$
    \EndFunction
    \end{algorithmic}
\end{algorithm}

We found that for a given $y_m \in \mathbb{R}$, we can determine whether $\min_{w} \max_{i} f_i(w) \leq y_m$ in $\mathcal{O}(n + \lambda)$ time.
This decision procedure relies on two key observations.

First, since each $f_i(w)$ is a convex quadratic function, the interval $\mathcal{W}_i \coloneqq \{w \in \mathbb{R} \mid f_i(w) \leq y_m\}$ can be computed in $\mathcal{O}(1)$ time.
This is achieved by solving the quadratic equation $f_i(w) = y_m$ using the closed-form solution for quadratic equations.

Second, the inequality $\min_{w} \max_{i} f_i(w) \leq y_m$ can be efficiently verified using the intervals $\mathcal{W}_i$ as follows:
\begin{align}
    \min_{w} \max_{i} f_i(w) \leq y_m &\iff
    \exists w \in \mathbb{R}, \forall i, f_i(w) \leq y_m \\
    &\iff \bigcap_{i} \mathcal{W}_i \neq \varnothing.
\end{align}
Hence, the decision problem reduces to checking whether the intersection of all $\mathcal{W}_i$ is nonempty.
Since each $\mathcal{W}_i$ is either an empty set or a continuous interval, the intersection of all $\mathcal{W}_i$ can be computed in $\mathcal{O}(n + \lambda)$ time.
Therefore, we can determine whether $\min_{w} \max_{i} f_i(w) \leq y_m$ in $\mathcal{O}(n + \lambda)$ time.

Using this decision procedure, we can exponentially narrow down the range containing the value $\min_{w} \max_{i} f_i(w)$ via binary search.
Since each decision procedure takes $\mathcal{O}(n + \lambda)$ time, the total computational cost over $T$ iterations is $\mathcal{O}(T(n + \lambda))$.

\subsection{Explicit Calculation of $\max_{i} f_i(w)$}
\label{sec:explicit_calculation_of_max_i_fi_w}

\begin{algorithm}[t]
    \caption{Explicitly Calculates $\max_{i} f_i(w)$}
    \label{algo:explicit_calculation_of_max_i_fi_w}
    \begin{algorithmic}[1]
    \Require Convex quadratic functions $f_1, f_2, \dots, f_q$
    \Function{ExplicitlyCalcMaxFIs}{$f_1, \dots, f_q$}
        \If{$q = 1$}
            \State \Return $[-\infty, \infty]$, $[f_1]$
        \Else
            \State $\bm{t}, \bm{g} \gets \textsc{ExplicitlyCalcMaxFIs}(f_1, \dots, f_{\lfloor q/2 \rfloor})$
            \State $\bm{u}, \bm{h} \gets \textsc{ExplicitlyCalcMaxFIs}(f_{\lfloor q/2 \rfloor + 1}, \dots, f_q)$
            \State \Return $\textsc{MergePiecewiseQuadratics}(\bm{t}, \bm{g}, \bm{u}, \bm{h})$
        \EndIf
    \EndFunction
    \end{algorithmic}
\end{algorithm}

\begin{algorithm}[t]
    \caption{Merge Piecewise Quadratic Functions}
    \label{algo:merge_piecewise_quadratic_functions}
    \begin{algorithmic}[1]
    \Require Thresholds: $\bm{t} \in \mathbb{R}^{n_1+1}$ s.t. $-\infty = t_1 < \dots < t_{n_1+1} = \infty$
    \Require Quadratics: $g_i(w): \mathbb{R} \to \mathbb{R}$ for $i \in \{1, 2, \dots, n\}$
    \Require Thesholds: $\bm{u} \in \mathbb{R}^{n_2+1}$ s.t. $-\infty = u_1 < \dots < u_{n_2+1} = \infty$
    \Require Quadratics: $h_i(w): \mathbb{R} \to \mathbb{R}$ for $i \in \{1, 2, \dots, m\}$
    \Function{MergePiecewiseQuadratics}{$\bm{t}$, $\bm{g}$, $\bm{u}$, $\bm{h}$}
        \State $n_1 \gets |\bm{g}|$, $n_2 \gets |\bm{h}|$
        \State $i \gets 1$, $j \gets 1$
        \State $\bm{s} \gets [-\infty]$
        \State $\bm{p} \gets [~]$
        \While{$i \le n_1 \land j \le n_2$}
            \State $a \gets \max(t_i, u_j), \quad b \gets \min(t_{i+1}, u_{j+1})$
            \If{$a \ge b$}
                \If{$t_{i+1} \le u_{j+1}$} $i \gets i + 1$ \EndIf
                \If{$u_{j+1} \le t_{i+1}$} $j \gets j + 1$ \EndIf
                \State \textbf{continue}
            \EndIf
            \\
            \State $\alpha, \beta \gets \textsc{SolveQuadEq}(g_i,h_j)$ \Comment{$w$ s.t. $g_i(w) = h_j(w)$.}
            \State $\bm{v} \gets [a,b]$
            \If{$\alpha \in (a, b)$} $\bm{v}.\textsc{append}(\alpha)$ \EndIf
            \If{$\beta \in (a, b)$} $\bm{v}.\textsc{append}(\beta)$ \EndIf
            \State $\bm{v} \gets \textsc{Sort}(\bm{v})$

            \\
            \For{$k \gets 1, \dots, |\bm{v}|-1$}
                \State $l \gets v_k, \quad r \gets v_{k+1}$
                \State $x \gets (l + r)/2$
                \If{$g_i(x) \geq h_j(x)$}
                    \State $\bm{s}.\textsc{append}(r), ~~ \bm{p}.\textsc{append}(g_i)$
                \Else
                    \State $\bm{s}.\textsc{append}(r), ~~ \bm{p}.\textsc{append}(h_j)$
                \EndIf
            \EndFor
            \\
            \If{$t_{i+1} = b$} $i \gets i + 1$ \EndIf
            \If{$u_{j+1} = b$} $j \gets j + 1$ \EndIf
        \EndWhile
        \State $i \gets 1$
        \While{$i + 1 \leq |\bm{p}|$}
            \If{$p_i = p_{i+1}$}
                \State $\bm{p}.\textsc{remove}(i+1)$ \Comment{Merge adjacent intervals}
                \State $\bm{s}.\textsc{remove}(i+1)$
            \Else
                \State $i \gets i + 1$
            \EndIf
        \EndWhile
        \State \Return $\bm{s}, \bm{p}$
    \EndFunction
    \end{algorithmic}
\end{algorithm}

We found that the shape of $\max_{i} f_i(w)$, which is a \textit{piecewise quadratic function}, can be computed both efficiently and exactly.
This result is based on the following two key observations:
\begin{enumerate}[leftmargin=1.5em]
    \item Given two piecewise quadratic functions, each of which consists of $n_1$ and $n_2$ segments, respectively, their pointwise maximum can be computed in $\mathcal{O}(n_1 + n_2)$ time.
    \item The number of segments required to represent the pointwise maximum of $m$ quadratic functions as a piecewise quadratic function is at most $\mathcal{O}(m)$.
\end{enumerate}
By leveraging these observations, we can explicitly construct the shape of $\max_{i} f_i(w)$ using a divide-and-conquer strategy.
At each step, we recursively divide the set of functions into halves, merge their piecewise quadratic representations, and continue until a single unified representation is obtained.
This process computes the exact shape of $\max_{i} f_i(w)$ in $\mathcal{O}((n + \lambda) \log (n + \lambda))$ time (\cref{algo:explicit_calculation_of_max_i_fi_w}).

\textbf{(1) Computing the Maximum of Two Piecewise Quadratic Functions in $\mathcal{O}(n_1 + n_2)$ Time}
We first describe how to compute the maximum of two piecewise quadratic functions in linear time.
The pseudocode is shown in \cref{algo:merge_piecewise_quadratic_functions}.

A piecewise quadratic function $g$ consisting of $n_1$ segments is defined by two components: a threshold vector $\bm{t} \in \mathbb{R}^{n_1 + 1}$, such that $-\infty = t_1 < \dots < t_{n_1 + 1} = \infty$, and a sequence of quadratic functions $\{ g_i : \mathbb{R} \to \mathbb{R} \}_{i=1}^{n_1}$.
The function $g$ is defined as follows:
\begin{equation}
    g(w) = g_i(w) \quad \text{if } t_i \le w < t_{i+1} \quad \text{for } i \in \{1, 2, \dots, n_1\}.
\end{equation}
Similarly, let $h$ be another piecewise quadratic function consisting of $n_2$ segments, represented by thresholds $\bm{u} \in \mathbb{R}^{n_2 + 1}$ and quadratic functions $\{h_i\}_{i=1}^{n_2}$.
Our goal is to obtain a piecewise quadratic representation of
\begin{equation}
    p(w) \coloneqq \max(g(w), h(w)).
\end{equation}

The algorithm sequentially scans the intervals of $g$ and $h$ in ascending order.
At each step, it compares the active segments $g_i(w)$ and $h_j(w)$ over their overlapping domain $[t_i, t_{i+1}] \cap [u_j, u_{j+1}]$.
It then determines whether their dominance changes within this domain by solving the quadratic equation $g_i(w) = h_j(w)$.
\begin{itemize}[leftmargin=1.5em]
    \item If it does not change, the larger function over the interval (determined, e.g., by comparing values at the midpoint) is appended to the result.
    \item If it does change, the point where the dominance changes is added as a new threshold, and the dominance relation is determined for each subinterval.
\end{itemize}
This process repeats while incrementally advancing the indices $i$ and $j$ until all intervals are processed.
Finally, adjacent segments corresponding to the same quadratic function are merged to eliminate redundant boundaries.

The computational complexity of this algorithm is $\mathcal{O}(n_1 + n_2)$ because each iteration requires only constant-time operations, and the number of iterations is at most $n_1 + n_2$ because the value of $i + j$ increases by at least one per iteration.
The final merging step also takes $\mathcal{O}(n_1 + n_2)$ time since the number of resulting segments is at most $n_1 + n_2$.

\textbf{(2) The Number of Segments in the Maximum of $m$ Quadratic Functions}
Let $f_1, f_2, \dots, f_m$ be $m$ quadratic functions, and let $\mu$ denote the number of segments in the piecewise quadratic representation of $\max_{i} f_i(w)$.
We now show that $\mu = \mathcal{O}(m)$.

Let $\bm{I} \in \{1, 2, \dots, m\}^{\mu}$ denote the sequence of indices indicating which quadratic function attains the maximum on each segment.
Since the merging algorithm (\textsc{MergePiecewiseQuadratics}) ensures that adjacent segments correspond to different quadratics, consecutive elements of $\bm{I}$ are distinct.
Moreover, for any distinct $i, j \in \{1, \dots, m\}$, the subsequence $i, j, i, j$ does not appear in $\bm{I}$ because any two quadratic functions can intersect at most twice.
Such sequences are known as Davenport--Schinzel sequences of order 2~\citep{sharir1988davenport}.
It is known that a Davenport--Schinzel sequence of order 2 over $m$ symbols has length at most $2m - 1$.
Therefore, we have $\mu = \mathcal{O}(m)$.

\section{Segment + Endpoint}
\label{app:segment_plus_endpoint}

We now describe in detail a class of poisoning attacks that we discovered, which we call \textbf{Segment + Endpoint (Seg+E)}.
First, in \cref{sec:def_of_seg_plut_e} we formally define the class of poisoning attacks called Seg+E.
Next, in \cref{sec:algo_exact_seg_plus_e_original} we present an $\mathcal{O}(n\lambda^3)$-time exact algorithm for Seg+E in the original setting, together with the theorem that justifies the algorithm.
Then, in \cref{sec:algo_exact_seg_plus_e_relaxed} we present an $\mathcal{O}(n\lambda)$-time exact algorithm for Seg+E in the relaxed setting, together with its supporting theorem.

\subsection{Definition of Segment + Endpoint}
\label{sec:def_of_seg_plut_e}

We first define a class of poison sets called \textit{Segment + Endpoint} (abbreviated as \textit{Seg+E}).
\begin{definition_box}
\begin{definition}[\textbf{Segment + Endpoint (Seg+E) in Original Setting}]
\label{def:seg_e}
    Let $\mathcal{K} = \{k_1, k_2, \dots, k_n\} \subset \mathbb{N}$ be a set of $n$ distinct legitimate keys arranged in increasing order.
    Define a function $\Phi_\mathrm{Seg+E}^{\mathrm{(orig)}} : \mathbb{Z}^4 \to 2^{\mathbb{Z}}$ as follows:
    \begin{align}
        &~\Phi_\mathrm{Seg+E}^{\mathrm{(orig)}}(R_1, L_2, R_2, L_3) \nonumber\\
        \coloneq&~ \left([k_1, R_1] \cup [L_2, R_2] \cup [L_3, k_n]\right) \cap \mathbb{Z} \setminus \mathcal{K}.
    \end{align}
    Then, a poison set $\mathcal{P} \subset \mathbb{Z}$ is said to be a Segment + Endpoint (Seg+E) if and only if there exist $R_1, L_2, R_2, L_3 \in \mathbb{Z}$ satisfying $k_1 \le R_1 < L_2 \le R_2 < L_3 \le k_n$ such that
    \begin{equation}
        \mathcal{P} = \Phi_\mathrm{Seg+E}^{\mathrm{(orig)}}(R_1, L_2, R_2, L_3).
    \end{equation}
\end{definition}
\end{definition_box}

Similarly, we define Seg+E in the relaxed setting as follows:
\begin{definition_box}
\begin{definition}[\textbf{Segment + Endpoint (Seg+E) in Relaxed Setting}]
\label{def:seg_e_relaxed}
    Let $\mathcal{K} = \{k_1, k_2, \dots, k_n\} \subset \mathbb{N}$ be a set of $n$ distinct legitimate keys arranged in increasing order.
    Define a function $\Phi_\mathrm{Seg+E}^{\mathrm{(relaxed)}} : \mathbb{Z}_{\geq 0} \times \mathbb{Z}_{\geq 0} \times \mathbb{Z}_{\geq 0} \times \{k_1+1,k_1+2,\dots,k_n-1\} \to \mathcal{M}(\mathbb{Z})$ as follows:
    \begin{equation}
        \Phi_\mathrm{Seg+E}^{\mathrm{(relaxed)}}(a, b, c, p) \coloneq \{
            \underbrace{k_1, \dots, k_1}_{a},
            \underbrace{p, \dots, p}_{b},
            \underbrace{k_n, \dots, k_n}_{c}
        \},
    \end{equation}
    where $\mathcal{M}(\mathbb{Z})$ denotes the set of multisets over $\mathbb{Z}$.
    Then, a poison multiset $\mathcal{P}$ is a Segment + Endpoint (Seg+E) in the relaxed setting if and only if there exists $p \in \{k_1+1,k_1+2,\dots,k_n-1\}$ and $a, b, c \in \mathbb{Z}_{\geq 0}$ such that
    \begin{equation}
        \mathcal{P} = \Phi_\mathrm{Seg+E}^{\mathrm{(relaxed)}}(a, b, c, p).
    \end{equation}
\end{definition}
\end{definition_box}

\subsection{Exact Seg+E Algorithm in the Original Setting}
\label{sec:algo_exact_seg_plus_e_original}

Without our theoretical insights, finding the exact Seg+E configuration poison through brute-force search is computationally prohibitive; it requires at least $\mathcal{O}(R \lambda^3)$ time, where $R \coloneq k_n - k_1$.
There are $\mathcal{O}(\lambda^3)$ combinations of poison counts for the left, middle, and right positions, and $\mathcal{O}(R)$ possible placements for the middle poisons.
For each candidate configuration, computing the MSE requires $\mathcal{O}(n + \lambda)$ time, resulting in a total cost of $\mathcal{O}(R \lambda^3 (n + \lambda))$.
By exploiting differencial computation of MSE, this can be reduced to $\mathcal{O}(R \lambda^3)$, but the dependence on $R$ still makes it prohibitively large in practice.

Our next results drastically reduce the search space for the optimal middle poison location from $\mathcal{O}(R)$ to $\mathcal{O}(n)$ candidates.
\begin{theorem_box}
\begin{theorem}
    \label{thm:seg_e_in_original_setting}
    In the original setting, restricting the Seg+E poisoning pattern to those satisfying $L_2 \in \mathcal{K} \lor R_2 \in \mathcal{K}$ does not make the MSE smaller.
    That is,
    \begin{align}
        &~\max_{R_1, L_2, R_2, L_3} E(\mathcal{K} \cup \Phi_\mathrm{Seg+E}^{\mathrm{(orig)}}(R_1, L_2, R_2, L_3)) \nonumber\\
        =&~ \max_{R_1, L_2, R_2, L_3 ~\text{s.t.}~ L_2 \in \mathcal{K} ~~\lor~~ R_2 \in \mathcal{K}} E(\mathcal{K} \cup \Phi_\mathrm{Seg+E}^{\mathrm{(orig)}}(R_1, L_2, R_2, L_3)).
    \end{align}
\end{theorem}
\end{theorem_box}
\begin{proof}[Proof of \cref{thm:seg_e_in_original_setting}]
\begin{figure}[t]
    \centering
    \includegraphics[width=\columnwidth]{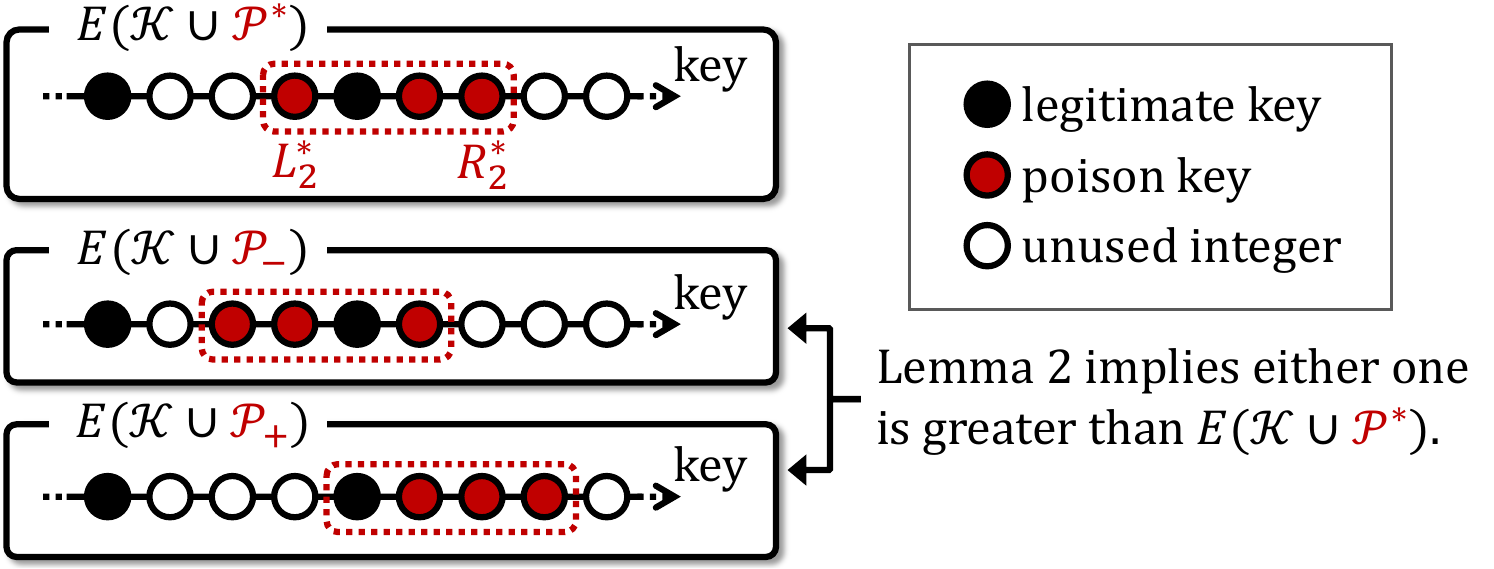}
    \caption{Proof of \cref{thm:seg_e_in_original_setting}.
    \cref{lem:multi_point_poisoning_attack} implies that $E(\mathcal{K} \cup \mathcal{P}_{-})$ or $E(\mathcal{K} \cup \mathcal{P}_{+})$ is greater than $E(\mathcal{K} \cup \mathcal{P}^\ast)$.}
    \label{fig:theorem_sege_vis}
\end{figure}
The proof is similar to the proof of \cref{thm:structure_of_optimal_multi_point_attack}.
For intuition, we show an outline of the proof in \cref{fig:theorem_sege_vis}.

We prove the claim by contradiction.
Assume that a counterexample to \cref{thm:seg_e_in_original_setting} exists.
That is, suppose there exists an optimal Seg+E solution that cannot be represented when restricting to $L_2 \in \mathcal{K} \land R_2 \in \mathcal{K}$; denote this solution by $\mathcal{P}^\ast$.
Write $\mathcal{P}^\ast = \Phi_\mathrm{Seg+E}^{\mathrm{(orig)}}(R^\ast_1, L^\ast_2, R^\ast_2, L^\ast_3)$.
By assumption we have $L^\ast_2 \notin \mathcal{K}$ and $R^\ast_2 \notin \mathcal{K}$, and moreover neither $L^\ast_2-1$ nor $R^\ast_2+1$ belongs to $\mathcal{K}\cup\mathcal{P}^\ast$.
(Otherwise, by appropriately rechoosing $R_1,L_2,R_2,L_3$ one could ensure $L_2 \in \mathcal{K} \lor R_2 \in \mathcal{K}$, which contradicts the assumption.)

Define $\mathcal{P}_{-}$ and $\mathcal{P}_{+}$ as follows:
\begin{align}
    \mathcal{P}_{-} &\coloneq \Phi_\mathrm{Seg+E}^{\mathrm{(orig)}}(R^\ast_1, L^\ast_2-1, R^\ast_2-1, L^\ast_3),\\
    \mathcal{P}_{+} &\coloneq \Phi_\mathrm{Seg+E}^{\mathrm{(orig)}}(R^\ast_1, L^\ast_2+1, R^\ast_2+1, L^\ast_3).
\end{align}
Then, by \cref{lem:multi_point_poisoning_attack}, either
$E(\mathcal{K} \cup \mathcal{P}_{-}) > E(\mathcal{K} \cup \mathcal{P}^\ast)$
or
$E(\mathcal{K} \cup \mathcal{P}^\ast) < E(\mathcal{K} \cup \mathcal{P}_{+})$.
This contradicts the assumption that $\mathcal{P}^\ast$ is an optimal Seg+E solution, completing the proof.
\end{proof}
From \cref{thm:seg_e_in_original_setting}, we can reduce the search space for the optimal middle poison location from $\mathcal{O}(R)$ to $\mathcal{O}(n)$ candidates.

\subsection{Exact Seg+E Algorithm in the Relaxed Setting}
\label{sec:algo_exact_seg_plus_e_relaxed}

\begin{algorithm}[t]
    \caption{Exact Seg+E in Relaxed Setting}
    \label{algo:algorithm_for_finding_the_exact_seg_e_in_relaxed_setting}
    \begin{algorithmic}[1]
    \Require Legitimate Keys: $\mathcal{K}$, Maximum Number of Poisons: $\lambda$
    \Function{ExactSeg+E}{$\mathcal{K},\lambda$}
        \State $a^\ast \gets \mathrm{None}, ~~ b^\ast \gets \mathrm{None}, ~~ i^\ast \gets \mathrm{None}$
        \State $\mathrm{MSE}_\mathrm{MAX} \gets -\infty$
        \For {$a = 0,1,\dots,\lambda$}
            \For {$i = 2,3,\dots,n-1$}
                \State $b \gets \textsc{GetOptimalB}(\mathcal{K},a,i)$
                \State $\mathrm{MSE} \gets \textsc{CalcMSE}(\mathcal{K},a,b,i)$
                \If {$\mathrm{MSE}_\mathrm{MAX} < \mathrm{MSE}$}
                    \State $a^\ast \gets a, ~~ b^\ast \gets b, ~~ i^\ast \gets i$
                    \State $\mathrm{MSE}_\mathrm{MAX} \gets \mathrm{MSE}$
                \EndIf
            \EndFor
        \EndFor
        \State \textbf{return} $(a^\ast,b^\ast,i^\ast)$
    \EndFunction
    \end{algorithmic}
\end{algorithm}

We can show that the exact Seg+E solution in realaxed setting necessarily saturates the poisoning budget: there exists some $i \in \{2,3,\dots,n-1\}$ such that $d_1+d_i+d_n=\lambda$.
This follows from \cref{lem:d_saturation_lemma}, which shows that, in the relaxed setting, we can always increase the MSE by adding more poison.

We define the function $\mathrm{MSE}(a,b,i): \mathbb{Z}_{\geq 0}\times\mathbb{Z}_{\geq 0}\times\{2,3,\dots,n-1\} \to \mathbb{R}$ as follows:
\begin{equation}
    \mathrm{MSE}(a,b,i) \coloneq E(\mathcal{K} \uplus \mathcal{Q}_\mathcal{K}(a \bm{e}_1 + b \bm{e}_i + (\lambda - a - b) \bm{e}_n)),
\end{equation}
where $a+b\leq\lambda$.
The problem of finding the exact Seg+E soltion is equivalent to finding $(a,b,i)$ that maximizes $\mathrm{MSE}(a,b,i)$.

An overview of our $\mathcal{O}(n\lambda)$-time algorithm for finding the exact Seg+E solution is provided in \cref{algo:algorithm_for_finding_the_exact_seg_e_in_relaxed_setting}.
We design two key subroutines:
\begin{itemize}[leftmargin=1.5em]
    \item $\textsc{CalcMSE}(a,b,i)$: Given $(a,b,i)$, this subroutine computes $\mathrm{MSE}(a,b,i)$. After an $\mathcal{O}(n)$ preprocessing step, we can compute each $\mathrm{MSE}(a,b,i)$ value in $\mathcal{O}(1)$ time.
    \item $\textsc{GetOptimalB}(a,i)$: Given $(a,i)$, this subroutine finds the value of $b$ that maximizes $\mathrm{MSE}(a,b,i)$. After an $\mathcal{O}(n)$ preprocessing step, the optimal $b$ for each $(a,i)$ can be obtained in $\mathcal{O}(1)$ time.
\end{itemize}
By applying these two subroutines to find the optimal $b$ for each $(a,i)$, we can find the exact Seg+E solution in $\mathcal{O}(n\lambda)$ time.

We can implement $\textsc{CalcMSE}(a,b,i)$ using the algorithm described in \cref{sec:efficient_computation_of_coefficients}.
In \cref{sec:efficient_computation_of_coefficients}, we showed that with an $\mathcal{O}(n)$ preprocessing step, we can compute $\mathrm{Var}_{\mathrm{K'}}$, $\mathrm{Var}_{\mathrm{R'}}$, and $\mathrm{Cov}_{\mathrm{K'R'}}$ for a given $(a,b,i)$ in $\mathcal{O}(1)$ time.
Therefore, from
\begin{equation}
\mathrm{MSE}(a,b,i) = \mathrm{Var}_{\mathrm{R'}} - \frac{\mathrm{Cov}_{\mathrm{K'R'}}^2}{\mathrm{Var}_{\mathrm{K'}}},
\end{equation}
we can compute $\mathrm{MSE}(a,b,i)$ in constant time, i.e., $\mathcal{O}(1)$.

Next, we describe $\textsc{GetOptimalB}$.
Our \textsc{GetOptimalB} subroutine starts by finding the value of $b$ that satisfies
\begin{equation}
    \frac{\partial \mathrm{MSE}(a,b,i)}{\partial b} = 0.
\end{equation}
Since $\mathrm{Var}_\mathrm{X}(a,b,i) > 0$ and $\mathrm{Cov}_\mathrm{XR}(a,b,i) > 0$, 
\begin{align}
    &~\frac{\partial \mathrm{MSE}(a,b,i)}{\partial b} = 0 \\
    \Leftrightarrow&~
    \label{eq:d_mse_d_b}
    -2 \left(\frac{\partial}{\partial b} \mathrm{Cov}_\mathrm{XR}(a,b,i)\right) \mathrm{Var}_\mathrm{X}(a,b,i) \nonumber\\
    &\qquad + \mathrm{Cov}_\mathrm{XR}(a,b,i) \left(\frac{\partial}{\partial b} \mathrm{Var}_\mathrm{X}(a,b,i)\right) = 0.
\end{align}
Since both $\mathrm{Cov}_\mathrm{XR}(a,b,i)$ and $\mathrm{Var}_\mathrm{X}(a,b,i)$ are quadratic functions of $b$, \cref{eq:d_mse_d_b} is a cubic equation in $b$.
Its coefficients can be computed in $\mathcal{O}(1)$ time using the prefix sums $S_t, T_t, U_t ~ (t=1,2,\dots,n)$ defined in \cref{sec:efficient_computation_of_coefficients}.
Therefore, we can find the value of $b$ that satisfies \cref{eq:d_mse_d_b} in $\mathcal{O}(1)$ time by solving this cubic equation.
The number of real solutions to this cubic equation is at most 3.

Here, we show the following simple lemma.
\begin{theorem_box}
\begin{lemma}
\label{lem:argmax_of_c1_function}
Let $f: [x_\mathrm{min}, x_\mathrm{max}] \to \mathbb{R}$ be a $C^1$ function, where $x_\mathrm{min}, x_\mathrm{max} \in \mathbb{Z}$.
Let
\begin{equation}
    x^\ast = \argmax_{x \in [x_\mathrm{min}, x_\mathrm{max}] \cap \mathbb{Z}} f(x).
\end{equation}
Then $x^\ast$ satisfies one of the following:
\begin{itemize}[leftmargin=1.5em]
    \item $x^\ast = x_\mathrm{min}$.
    \item $x^\ast = x_\mathrm{max}$.
    \item $\exists \hat{x} \in (x^\ast - 1, x^\ast + 1) \cap [x_\mathrm{min}, x_\mathrm{max}] ~\text{s.t.}~ f'(\hat{x}) = 0$.
\end{itemize}
\end{lemma}
\end{theorem_box}
\begin{proof}[Proof of \cref{lem:argmax_of_c1_function}]
When $x^\ast = x_\mathrm{min}$ or $x^\ast = x_\mathrm{max}$, the claim is trivial.
In the following, we consider the case where $x_\mathrm{min} + 1 \leq x^\ast \leq x_\mathrm{max} - 1$.

Since $f(x^\ast)$ is the maximum value of $f(x)$, we have $f(x^\ast) \ge f(x^\ast - 1)$ and $f(x^\ast) \ge f(x^\ast + 1)$.
By the Mean Value Theorem, there exist $\xi_1 \in (x^\ast - 1, x^\ast)$ and $\xi_2 \in (x^\ast, x^\ast + 1)$ such that
\begin{align}
    f'(\xi_1) &= \frac{f(x^\ast) - f(x^\ast - 1)}{x^\ast - (x^\ast - 1)} \ge 0,\\
    f'(\xi_2) &= \frac{f(x^\ast + 1) - f(x^\ast)}{(x^\ast + 1) - x^\ast} \le 0.
\end{align}
If $f'(\xi_1) = 0$ or $f'(\xi_2) = 0$, then the claim follows immediately because $\xi_1, \xi_2 \in (x^\ast - 1, x^\ast + 1)$.
Otherwise (i.e., $f'(\xi_1) > 0$ and $f'(\xi_2) < 0$), by the Intermediate Value Theorem, there exists $\hat{x} \in (\xi_1, \xi_2) \subset (x^\ast - 1, x^\ast + 1)$ such that $f'(\hat{x}) = 0$.
\end{proof}

For a fixed $(a,i)$, $\mathrm{MSE}(a,b,i)$ is a $C^1$ function of $b \in [0, \lambda - a]$.
By \cref{lem:argmax_of_c1_function}, the optimal integer $b^\ast$ must satisfy one of the following:
\begin{itemize}[leftmargin=1.5em]
    \item $b^\ast = 0$.
    \item $b^\ast = \lambda - a \quad$.
    \item $\hat{b} \in (b^\ast - 1, b^\ast + 1) \cap [0, \lambda - a] ~\text{s.t.}~ \frac{\partial \mathrm{MSE}(a,\hat{b},i)}{\partial b} = 0$.
\end{itemize}
Since the cubic equation $\frac{\partial \mathrm{MSE}(a,b,i)}{\partial b} = 0$ has at most 3 real solutions, there are at most 6 integers neighboring to the solution.
By adding the two cases $b^\ast = 0$ and $b^\ast = \lambda - a$, there are at most 8 candidates to consider.
By evaluating the MSE for these 8 candidates, we can find the optimal $b$ in $\mathcal{O}(1)$ time.

\section{Observations and Conjectures}
\label{app:observations_and_conjectures}

\begin{figure*}[t]
    \centering
    \includegraphics[width=\textwidth]{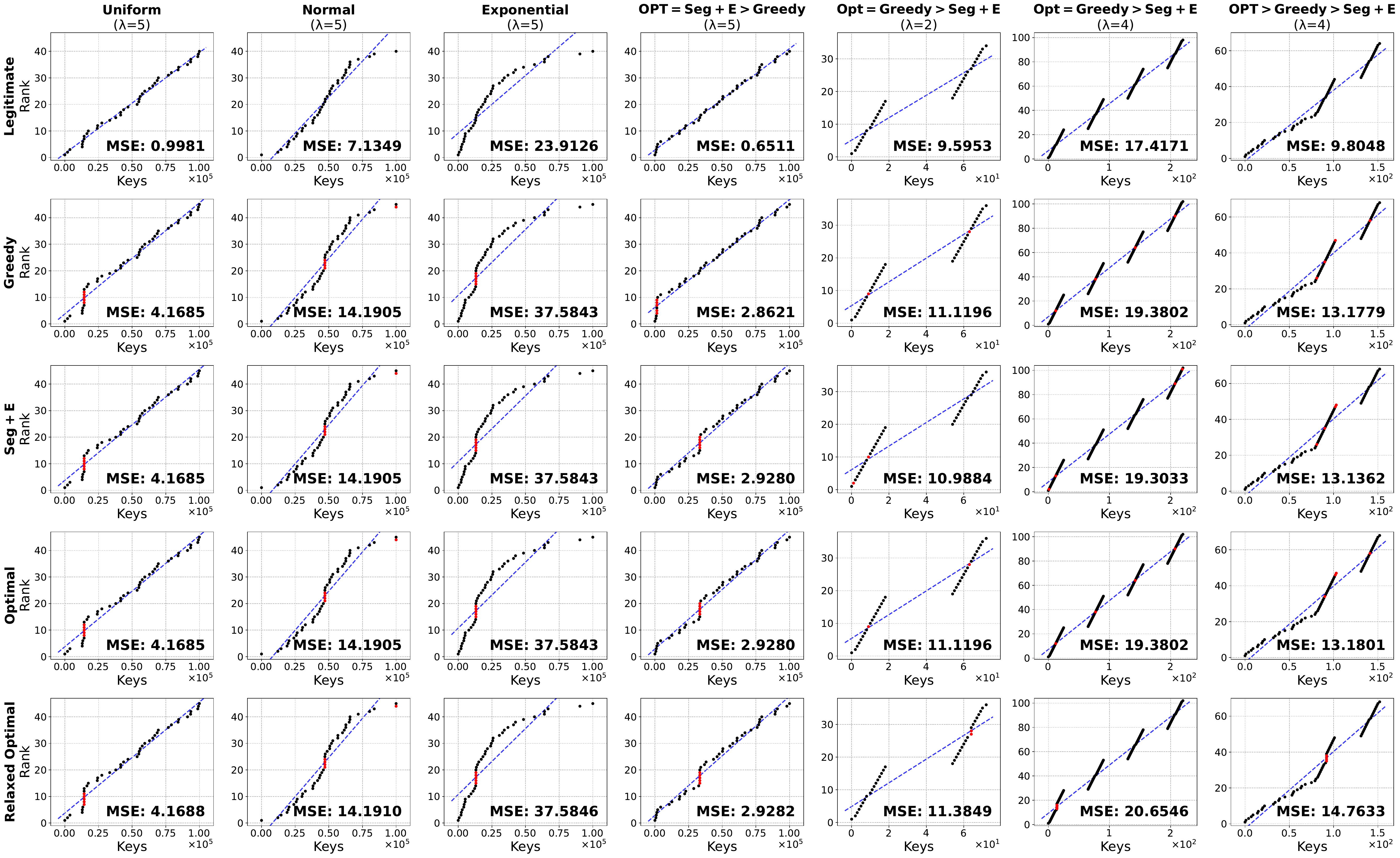}
    \caption{Empirical observations, conjectures, and counterexamples.
All examples (especially those in the second and third columns) illustrate \cref{observ:1}, showing that poisons tend to appear either near the endpoints or around the intersection between the regression line and the sequence of points formed by the legitimate keys.
Columns 1--4 correspond to Seg+E configurations that satisfy \cref{observ:2}, whereas columns 5--7 show artificially constructed corner cases that violate \cref{observ:2}.
In all examples, the optimal poisons in the relaxed setting are Seg+E, supporting \cref{con:2}.}
    \label{fig:rank_plot}
\end{figure*}

In this section, we summarize several empirical observations and conjectures obtained through our experiments.

First, we introduce the following observation about the optimal poison location:
\begin{theorem_box}
\begin{observation}
    \label{observ:1}
    Optimal poisons tend to appear either near the endpoints (i.e., $k_1$ or $k_n$) or around the points where the regression line intersects the sequence of points formed by the legitimate keys.
\end{observation}
\end{theorem_box}
This observation provides an intuitive guideline for predicting where optimal poisons are likely to appear.
All examples in \cref{fig:rank_plot} satisfy \cref{observ:1}, especially the examples in the second and third columns clearly demonstrate this pattern.

We can interpret this observation on multiple intuitive levels.

First, from a heuristic perspective, we can classify effective poisons into two types:
(1) those that themselves incur a large squared error, and
(2) those that significantly increase the squared error of surrounding keys.
Poisons placed near the endpoints serve as the first type.
Especially for datasets drawn from normal or exponential distributions, the prediction error of the regression line tends to be largest at the boundaries.
Thus, inserting poisons near the endpoints directly increases the individual squared error term.
In contrast, poisons placed around the intersection between the regression line and the sequence of points formed by the legitimate keys serve as the second type.
In these regions, adding a poison tends to increase the ranks of larger keys by one, thereby amplifying the overall MSE efficiently.

Furthermore, we can interpret this observation from a mathematical perspective.
Consider the MSE after adding a single poison point $p$ to the set of legitimate keys $\mathcal{K}$, denoted as $E(\mathcal{K} \cup {p})$.
Let $\mathrm{Var}_\mathrm{X}(p)$ and $\mathrm{Cov}_\mathrm{XR}(p)$ denote, respectively, the variance of keys and the covariance between keys and ranks after adding $p$.
Let $\hat{r}: \mathbb{R} \to \mathbb{R}$ represent the regression line function predicting rank from keys, and let $r_p$ denote the actual rank of $p$ in $\mathcal{K} \cup {p}$.
Then,
\begin{equation}
    \frac{\partial E(\mathcal{K} \cup \{p\})}{\partial p} = - \frac{2 \mathrm{Var}_\mathrm{X}(p)}{(n+1) \mathrm{Cov}_\mathrm{XR}(p)} \left(\hat{r}(p) - r_p\right).
\end{equation}
Thus, the derivative $\frac{\partial E(\mathcal{K} \cup {p})}{\partial p}$ becomes zero when $\hat{r}(p) = r_p$, that is, when the predicted rank equals the actual rank.
This analytically explains why optimal poison points tend to emerge near the intersections between the regression line and the sequence of points formed by the legitimate keys, or, at the endpoints.

On the other hand, following observation is very often satisfied in practice, but we also find counterexamples artificially constructed.
\begin{theorem_box}
\begin{observation}
    \label{observ:2}
    In the original setting, the optimal solution is often, but not always, Seg+E (as defined in \cref{def:seg_e}).
\end{observation}
\end{theorem_box}

Columns 1--4 in \cref{fig:rank_plot} exhibit this property clearly.
In each of these samples, poisons appear in intervals connected to the endpoints and one contiguous internal segment.
In contrast, columns 5--7 do not satisfy this property: the optimal poisons in these cases span multiple disjoint interior intervals.
However, these are deliberately constructed adversarial examples (e.g., by intentionally skipping certain integers) and thus are highly unrealistic.
In practical scenarios, \cref{observ:2} holds extremely well.

It is also worth noting that, in the 7th column, neither the Seg+E configuration nor the Greedy algorithm~\cite{kornaropoulos2022price} achieves the optimal solution.
This implies that even a hybrid strategy that selects the better of Seg+E and Greedy does not necessarily guarantee optimality.
Nevertheless, we emphasize again that these examples are intentionally pathological; in all realistic cases we examined, Seg+E achieves the optimal solution.

We further present the following conjecture, for which no counterexamples have been found to date.
\begin{theorem_box}
\begin{conjecture}
    \label{con:2}
    In the relaxed setting, Seg+E (as defined in \cref{def:seg_e_relaxed}) always achieves the optimal solution.
\end{conjecture}
\end{theorem_box}
This conjecture remains neither proved nor disproved, and is left as future work.

We have also observed that the following conjecture holds empirically across all experiments.
If this conjecture is proven, \cref{con:2} would follow immediately; thus, it may serve as a useful stepping stone toward proving \cref{con:2}.
\begin{theorem_box}
\begin{conjecture}
    \label{con:3}
    In the relaxed setting, if there exist $i, j$ such that $1 < i < j < n$ and $d_i, d_j \ge 1$, then the following holds:
    \begin{align}
        &\mathrm{MSE}_{i \to j} > \mathrm{MSE}_\mathrm{BASE} ~~\lor~~ \mathrm{MSE}_{j \to i} > \mathrm{MSE}_\mathrm{BASE} \nonumber\\
    \lor~~  &\mathrm{MSE}_{i \to 1} > \mathrm{MSE}_\mathrm{BASE} ~~\lor~~ \mathrm{MSE}_{i \to n} > \mathrm{MSE}_\mathrm{BASE} \nonumber\\
    \lor~~  &\mathrm{MSE}_{j \to 1} > \mathrm{MSE}_\mathrm{BASE} ~~\lor~~ \mathrm{MSE}_{j \to n} > \mathrm{MSE}_\mathrm{BASE},
    \end{align}
    where
    \begin{align}
        \mathrm{MSE}_{i \to j} &\coloneq E(\mathcal{K} \uplus \mathcal{Q}_\mathcal{K}(\bm{d} - \bm{e}_i + \bm{e}_j)),\\
        \mathrm{MSE}_{j \to i} &\coloneq E(\mathcal{K} \uplus \mathcal{Q}_\mathcal{K}(\bm{d} - \bm{e}_j + \bm{e}_i)),\\
        \mathrm{MSE}_{i \to 1} &\coloneq E(\mathcal{K} \uplus \mathcal{Q}_\mathcal{K}(\bm{d} - \bm{e}_i + \bm{e}_1)),\\
        \mathrm{MSE}_{i \to n} &\coloneq E(\mathcal{K} \uplus \mathcal{Q}_\mathcal{K}(\bm{d} - \bm{e}_i + \bm{e}_n)),\\
        \mathrm{MSE}_{j \to 1} &\coloneq E(\mathcal{K} \uplus \mathcal{Q}_\mathcal{K}(\bm{d} - \bm{e}_j + \bm{e}_1)),\\
        \mathrm{MSE}_{j \to n} &\coloneq E(\mathcal{K} \uplus \mathcal{Q}_\mathcal{K}(\bm{d} - \bm{e}_j + \bm{e}_n)).
    \end{align}
\end{conjecture}
\end{theorem_box}
This conjecture remains neither proved nor disproved, and is left as future work.

\end{document}